\documentclass[twoside,11pt]{article}

\usepackage{jmlr2e}

\usepackage[utf8]{inputenc} 
\usepackage[T1]{fontenc}    
\usepackage{hyperref}       
\usepackage{url}            
\usepackage{booktabs}       
\usepackage{amsfonts}       
\usepackage{nicefrac}       
\usepackage{microtype}      

\usepackage{graphicx}
\usepackage{amssymb}
\usepackage{amsmath}
\usepackage{bm}
\usepackage{dsfont}
\usepackage{float}
\usepackage{mathrsfs}
\usepackage{times}
\usepackage{xspace}
\usepackage{algpseudocode}
\usepackage{algorithm}
\usepackage{subcaption}

\newtheorem{thrm}{Theorem}

\newtheorem{lem}{Lemma}

\newtheorem{defin}{Definition}

\newtheorem{cor}{Corollary}
\newenvironment{mat}[1]{\left[\begin{array}{#1}}{\end{array}\right]}
\newcommand{\bmx}[1]{\begin{mat}{#1}}
\newcommand{\emx}{\end{mat}}

\newcommand{\gss}[3]{\mbox{\boldmath $#1$}_{#2}^{#3}}
\newcommand{\g}[1]{\mbox{\boldmath $#1$}}

\newcommand{\sbm}[1]{\mbox{\scriptsize $\g{#1}$}}
\newcommand{\sbms}[2]{\mbox{\scriptsize $\gss{#1}{#2}{}$}}

\ShortHeadings{}{Katselis, Beck and Srikant}
\firstpageno{3}

\begin{document}

\title{Mixing Times and Structural Inference for Bernoulli Autoregressive Processes}

\author{\name Dimitrios Katselis \email katselis@illinois.edu \\
       \addr Department of Electrical and Computer Engineering\\
       and Coordinated Science Lab\\
       University of Illinois, Urbana-Champaign\\
       Urbana, IL 61801, USA
        \AND
       \name Carolyn L.\ Beck \email beck3@illinois.edu \\
       \addr Department of Industrial and Enterprise Systems Engineering\\
             and Coordinated Science Lab\\
       University of Illinois, Urbana-Champaign\\
       Urbana, IL 61801, USA
        \AND
       \name R. \ Srikant \email rsrikant@illinois.edu \\
       \addr Department of Electrical and Computer Engineering\\
             and Coordinated Science Lab\\
       University of Illinois, Urbana-Champaign\\
       Urbana, IL 61801, USA
}


\editor{}

\maketitle

\begin{abstract}
We introduce a novel multivariate random process producing Bernoulli outputs per dimension, that can possibly formalize binary interactions in various graphical structures and can be used to model opinion dynamics, epidemics, financial and biological time series data, etc.  We call this a \emph{Bernoulli Autoregressive Process} (BAR). A BAR process models a discrete-time vector random sequence of $p$ scalar Bernoulli processes with autoregressive dynamics and corresponds to a particular Markov Chain. The benefit from the autoregressive dynamics is the description of a $2^p\times 2^p$ transition matrix by at most $pd$ effective parameters for some $d\ll p$ or by  two sparse  matrices of dimensions $p\times p^2$ and $p\times p$, respectively, parameterizing the transitions. Additionally, we show that the BAR process mixes rapidly, by proving that the mixing time is $O(\log p)$. The hidden constant in the previous mixing time bound depends explicitly on the values of the chain parameters and implicitly on the maximum allowed in-degree of a node in the corresponding graph.   For a network with $p$ nodes, where each node has in-degree at most $d$ and  corresponds to a scalar Bernoulli process generated by a BAR, we provide a  greedy algorithm that can efficiently learn 
the structure of the underlying directed graph with a sample complexity proportional to the mixing time of the BAR process. The sample complexity of the proposed algorithm is nearly order-optimal as it is only a $\log p$ factor away from an information-theoretic lower bound.  We present simulation results illustrating the performance of our algorithm in various setups, including a model for a biological signaling network.
\end{abstract}


\section{Introduction}

Dynamical systems which evolve over networks are ubiquitous: examples include epidemic and opinion dynamics over social
networks, gene regulatory networks, and stock/option price dynamics in financial markets \cite{bbv08, bim10,bvn11, migs13, qkc15, stb15}.  We model the interactions between the nodes in a network of this type using directed edges, where an edge from node $A$ to node $B$ indicates that the state of node $A$ at one time instant affects the state of node $B$ at the next time instant. Our goal is to infer such a directed graph from time series data. To this end, the mixing properties of the corresponding dynamical system become critical in determining sufficient sample complexities for consistent structure estimators. 

Prior to performing inference on dynamical system data, one has to select a model class for the underlying dynamics.   Modeling the dynamics of
a multivariate process is a core subproblem  of \emph{system identification} \cite{gp77,l99}. The work in this paper can be  thought as laying the groundwork for performing \emph{discrete system identification} in the class of \emph{Bernoulli Autoregressive Processes} (BAR) that are defined later and determining their mixing properties. 
Traditionally, the identification literature deals with the problem of estimating the unknown parameters  of a system model directly; however, in the machine learning literature, the problem is broken down into two steps: first, structure learning and second, parameter estimation. Structure learning refers to the problem of estimating whether each parameter is positive, negative, or zero. Here, we focus only on the structure learning problem;  parameter estimation is typically a simpler problem once the structure is known. When the state variables take continuous values, structure learning can be related to Lasso-type sparse inference ideas (see \cite{bim10}) or more traditional system-theoretic filtering ideas \cite{migs13}. However, to the best of our knowledge, these ideas do not directly apply to models where the state variables take on discrete values, as in our BAR model.

A related approach to discrete system identification is what is known as \emph{causal network inference}. Historically, this problem is linked to the notion of \emph{Granger causality} \cite{g69}. The goal in  \cite{g69} was the description of a hypothesis test to determine whether a time series can be used to improve the predictability of another time series. The problem was initially formalized to tackle the case of linear dynamics. Granger causality has been more recently incorporated and extended in information theory through the notion of \emph{directed information}.   
Following the pioneering work of Marko \cite{m73}, the notion of directed information was first introduced by Massey \cite{m90}, as a measure of the information flow from one random sequence to another in a synchronized fashion. Later on, Kramer introduced the notion of causal conditioning in probability distributions. Connecting the latter with directed information, Kramer was able to analyze the capacity of systems with feedback \cite{k98,k03}. Whenever there is feedback between the input and the output of a system, the \emph{directivity} of information flow becomes significant. 
Very recently, directed information has been used to perform causal network inference of networked dynamical processes \cite{qkc15}. The approach is  very generally applicable, but here we are interested in developing algorithms, which have good sample and computational complexities when applied to a specific model class.

If the underlying model does not have dynamics, i.e., the state of the system at one time instant does not affect the state of the system at the next time instant, then the discrete system identification problem reduces to the problem of learning graphical models from independent samples. Graphical models constitute a powerful statistical tool used to represent joint probability distributions, with the underlying graph dictating the way that these distributions factorize.   The associated graphs are most often undirected giving rise to Markov random fields, while the data are often considered to be i.i.d. \cite{athw12,afw12,b15,w15}. The associated graphs can also be directed and acyclic giving rise to Bayesian networks, which possess a different set of properties from Markov random fields \cite{w15}. The graphical structure affects the computational complexity associated with different statistical inference tasks such as computing marginals, posterior probabilities, maximum a posteriori estimates and sampling from the corresponding distributions. The absence of an edge in such graphs represents conditional independence between two nodes or random variables given the rest of the network. Our problem can be thought as that of learning a directed network in which the output at one time instant becomes the input to the next time instant.

\subsection{Main Results} 

In this paper,  
we introduce a \emph{novel} multivariate dynamical process producing Bernoulli outputs at each node in a network, that can possibly formalize binary interactions in various graphical structures. We call this a \emph{Bernoulli Autoregressive Process}. A BAR models a Bernoulli vector process with linear dynamics imposed on the parameters of the associated Bernoulli random variables and corresponds to a special form of  Markov chain with a \emph{directed} underlying graphical structure. Assuming $p$ nodes in the network, we first prove that any BAR process mixes very rapidly by showing that  $t_{mix}(\theta)=O(\log p)$ for any $\theta\in (0,1)$. Motivated by the operation of local Markov chains including the Glauber dynamics, we also introduce in Section \ref{sec:Comparison} two BAR random walks on the hypercube and under a column-substochasticity assumption on the dynamics,  we prove that $t_{mix}(\theta)=O(p\log p)$ in this case. 
If each node has in-degree at most $d$ and corresponds to a scalar Bernoulli process generated by a BAR,
we provide a greedy algorithm that can learn 
the structure of the underlying directed graph with computational complexity of order  $O(np^2)$ for a sufficient number of samples $n$ proportional to the mixing time of the BAR process.  The aforementioned structure estimator is shown to be nearly order-optimal requiring a sample complexity that is only a multiple of $\log p$ away from a lower information-theoretic bound. 

\subsection{Other Related Work}
 
In their influential paper \cite{cl68}, Chow and Liu showed that learning a tree-structured Markov random field can be achieved with time complexity of order $p^2$, when the graph has $p$ nodes. Very recently, Bresler provided a method to learn the structure of Ising models with $p$ nodes of degree at most $d$ in time $c(d) p^2\log p$, where $c(d)$ is a constant depending \emph{doubly-exponentially} on $d$ and the range of interaction strengths in the Ising model \cite{b15}. \emph{Exponential} dependence on $d$ is unavoidable in binary graphical models as it has been shown in \cite{sw12}. \cite{cl68} and \cite{b15} perform structure learning of graphical models based on i.i.d. samples. In the more relevant work \cite{bgs15}, the problem of learning undirected graphical models based on data generated by the Glauber dynamics is examined. The Glauber dynamics is a special form of a reversible Markov chain and it is often used in the context of Markov Chain  Monte Carlo algorithms to sample from a distribution of interest. The key observation in \cite{bgs15} is that the problem of learning graphs by observing the Glauber dynamics is computationally tractable. Our work differs from the assumptions in \cite{bgs15} in that the corresponding graph for a BAR process is \emph{directed} and at any time instant the individual states of ultimately \emph{all} nodes can be updated. The problem of learning the graphical structure of a general Markov time series is examined in \cite{crvs13}. The proposed approach is relevant to the notion of causal entropy, which is very similar in nature to directed information. The corresponding sample complexity depends logarithmically on the dimension of the process and linearly on the mixing time, while the computational complexity is $O(p^4)$. The problem of learning from epidemic cascades has been considered in \cite{ns12,rbs11} and \cite{ml10}, while a number of papers have studied the problem of learning functions or concepts by observing Markov chain sample paths \cite{av90, bfh94,g03}. Finally, a relevant model to the BAR is the so-called ALARM introduced in \cite{al13}. ALARM differs from the BAR in the use of the logistic function in defining the transition probabilities, which allows linear (positive and negative) coefficients in the autoregressive dynamics. Nevertheless, the interpretability of the model parameters is not as straighforward as in the BAR model. In addition, we can easily devise examples of BAR chains in which the transitions cannot be exactly described by the ALARM model.  Moreover, to the best of our knowledge, no mixing time analysis of the ALARM model exists, while the proposed algorithm for recovering the structure of the network is based on a straightforward application of $\ell_1-$regularization, i.e., Lasso and group-Lasso. It is well-known that convex optimization approaches like Lasso and group-Lasso for model selection have very poor performance, when the observed data have temporal dependencies as in our case \cite{wr09,zy06}. Further analysis details would be of interest, e.g., recovery guarantees, sample complexity and mixing properties of the ALARM model. 

\subsection{Organization}

Section \ref{sec:ourContr} informally presents the contributions of this paper. The BAR model is introduced in Section \ref{sec:BAR}.  Section \ref{sec:BARodLearn} formulates the BAR structure learning problem, while the proposed BAR structure observer is formalized. The main results in this paper and their proofs are presented in Sections \ref{sec:MainResults}, \ref{sec:ProofMainThm1} and \ref{sec:ProofMainThm2}. Section \ref{sec:LowerInf} provides a lower information-theoretic bound on the necessary sample complexity for any BAR structure estimator, verifying that the proposed algorithm is nearly order-optimal. Random walk variants of the BAR model along with their mixing properties are provided in Section \ref{sec:Comparison}, completing the palette of BAR processes.  Finally, numerical examples are provided in Section \ref{sec:sims} and  the paper is concluded in Section \ref{sec:concl}.   

\section{Main Contributions: A Conceptual Overview}
\label{sec:ourContr}

In this section, we informally discuss and provide intuitive interpretations of our results. The graphical structure of a BAR process can be visualized as in Fig. \ref{fig:1a}, where for simplicity we assume that $p=5$. Each node corresponds to a scalar Bernoulli random process.  The arrows represent causal relationships between the nodes. The problems that we consider are to evaluate the mixing time of any BAR process and to identify the parental sets of all nodes in the graph, i.e., the set of nodes whose states at one time instant affect the state of any given node at the next time instant. We assume that the in-degree of each node is at most $d$. Our main results are the following:

\begin{figure}[t!]
    \centering
    \begin{subfigure}[t]{0.5\textwidth}
        \centering
        \includegraphics[height=0.8in]{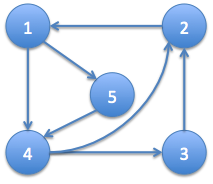}
            \captionsetup{justification=centering}
        \caption{BAR problem setup, $p=5$.}
        \label{fig:1a}
    \end{subfigure}%
    ~ 
    \begin{subfigure}[t]{0.5\textwidth}
        \centering
        \includegraphics[height=0.8in]{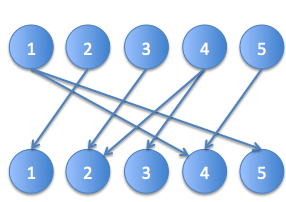}
          \caption{Equivalent BAR problem setup, $p=5$.}
   	 \label{fig:1b}
    \end{subfigure}
    \caption{}
\end{figure}

\begin{thrm}\label{thm:BARMixing}(BAR, Mixing Time-informal)
Consider a BAR process corresponding to an arbitrary graph with $p$ nodes, where each node has in-degree at most $d$. This process corresponds to a special kind of Markov chains with a directed graphical structure. For any $\theta\in (0,1)$, the BAR process mixes rapidly with $t_{mix}(\theta)=O(\log p)$. The hidden terms in $O(\cdot)$ notation depend explicitly on $\theta$ and the parameter values defining the BAR model and implicitly on $d$. Moreover, the BAR random walks corresponding to the same network also mix rapidly with $t_{mix}(\theta)=O(p\log p)$, when the dynamics satisfy a column-substochasticity assumption.  The hidden terms in $O(\cdot)$ notation depend explicitly on $\theta$, the parameter values defining the BAR random walk, the probability of being lazy and implicitly on $d$ 
\end{thrm} 

\begin{thrm}\label{thm:BARInf}(BAR, Algorithm Performance-informal)
Consider the BAR model on an arbitrary graph with $p$ nodes, where each node has in-degree at most $d$.  Given $n=h(d)t_{mix}(\theta)\log p=\tilde{h}(d)\log^2 p$ samples from the model for some $\theta\leq 1/8$,  it is possible to learn the underlying graph using  $O(np^2)$  computations. Here, $h(d)$ is a function of $d$ and the parameter values defining the BAR model. Moreover, since $t_{mix}(\theta)=O(\log p)$, $\tilde{h}(d)$ corresponds to $h(d)$ having absorbed the mixing bound constant. 
\end{thrm} 

The \emph{parental set} or \emph{parental neighborhood} of each node in the graph is separately estimated. To obtain the necessary accuracy in the required statistics for the proposed algorithm, the order of the sample complexity can be proved to be $h(d)t_{mix}(\theta)\log p$. To decide the edge set of the graph, i.e., the existence of arrows and their orientations, the algorithm has to initially examine all possible pairs of nodes separately, the number of which is $p^2$. Given the required computations per sample, it turns out that the complexity of the algorithm is $O(np^2)$ due to a constant upper bound imposed on $d$ by the definition of the BAR model, as we will see shortly. Further, the proposed algorithm for learning the structure of a BAR process can determine if a node affects another node \emph{positively} or \emph{negatively}.

\section{Bernoulli Autoregressive Processes}
\label{sec:BAR}    

We consider a directed graph $\mathcal{G}=(\mathcal{V},\mathcal{E})$ with $p$ nodes (Fig. \ref{fig:1a}). We can associate with this graph a bipartite graph, where the two parts have the same vertex set $\mathcal{V}$ (Fig. \ref{fig:1b}). For each node $u$ in the upper part we have directed links to the lower part only to those nodes that $u$ points to or \emph{causally affects}. Although Fig. \ref{fig:1b} is illuminating, we will focus on Fig. \ref{fig:1a} from now on and we will consider a directed graph $\mathcal{G}=(\mathcal{V},\mathcal{E})$, the topology of which we would like to infer based on a particular form of Bernoulli time series data to be specified shortly.   

The state of each node in $\mathcal{V}$ is assumed to be a Bernoulli random variable (r.v.) $X_i, i=1,2,\ldots, p$. The state vector $\g{X}=[X_1, X_2, \ldots, X_p]^T$ is considered to be observable. We will denote a particular realization of $\g{X}$ by $\g{x}$. Assuming that the state vector of $\mathcal{V}$ at time instant $k$ is $\gss{X}{k}{}$, the BAR updates the state vector as follows: 
\begin{align}\label{eq:BAR}
\gss{X}{k+1}{}=\mathrm{Ber}\left(\g{A}\g{f}(\gss{X}{k}{})+\g{B}\gss{W}{k+1}{}\right),
\end{align}   
with a rowwise interpretation of $\mathrm{Ber}\left(\cdot\right)$. Additionally, $\g{A}$ is the $p\times p^2$ matrix $\mathrm{diag}\left(\gss{a}{1}{T},\gss{a}{2}{T},\ldots,\gss{a}{p}{T}\right)$ with the interpretation that $\gss{a}{1}{T},\gss{a}{2}{T},\ldots,\gss{a}{p}{T}$ are the corresponding $1\times p$ diagonal elements (or blocks) and $\g{f}(\gss{X}{k}{})=\left[f_1(\gss{X}{k}{})^T,f_2(\gss{X}{k}{})^T,\ldots,f_p(\gss{X}{k}{})^T \right]^T$. Each $f_i(\g{X})$ corresponds to a $p\times 1$ vector and has the form 
\begin{align}\label{eq:f}
[f_i\left(\g{X}\right)]_j=f_{i}\left(X_j\right)=\left\{\begin{array}{cc}
    X_j, & j\in \mathcal{S}^{+}(i)\\ 
  1-X_j, & j\in \mathcal{S}^{-}(i) \\
  X_j\: \text{or}\: 1-X_j &  j\in \mathcal{V}\setminus \mathcal{S}(i)\\
\end{array}\right.,
\end{align}
where we denote the parental neighborhood of node $i$ by $\mathcal{S}(i)=\mathrm{supp}(\gss{a}{i}{})$. Here, $\mathrm{supp}(\cdot)$ stands for the unsigned support of a vector or a matrix. For $j\in \mathcal{V}\setminus \mathcal{S}(i)$ the choice of either $X_j$ or $1-X_j$ is irrelevant; hence, we will assume from now on the convention $f_{i}(X_j)=X_j$ in this case. Moreover, we consider the partition  $\mathcal{S}(i)=\mathcal{S}^{+}(i)\cup \mathcal{S}^{-}(i)$, where $\mathcal{S}^{+}(i)\cap \mathcal{S}^{-}(i)=\emptyset$. The set of parents, which \emph{positively influence} node $i$, is denoted by $\mathcal{S}^{+}(i)$ and the set of parents, which \emph{negatively influence} node $i$, is denoted by $\mathcal{S}^{-}(i)$. The simplest but also the most natural case assumes that $f_1(\gss{X}{k}{})=f_2(\gss{X}{k}{})=\cdots=f_p(\gss{X}{k}{})=\gss{X}{k}{}$, i.e., $\mathcal{S}^{-}(i)=\emptyset$ for all $i\in [p]$, resulting to 
\begin{align}\label{eq:BAR2}
\gss{X}{k+1}{}=\mathrm{Ber}\left(\g{\bar{A}}\gss{X}{k}{}+\g{B}\gss{W}{k+1}{}\right).
\end{align} 
Here, $\g{\bar{A}}$ is the $p\times p$ matrix
\begin{equation}\label{eq:Abar}
\g{\bar{A}}=\left(\begin{array}{c}
\gss{a}{1}{T}\\
\gss{a}{2}{T}\\
\vdots\\
\gss{a}{p}{T}\\
\end{array}\right).
\end{equation}
\textbf{Remark}: $\g{\bar{A}}$ can be \emph{always} defined and will  always be a reference matrix in our subsequent analysis, even in the case that it does not directly appear to the BAR model, as for example in (\ref{eq:BAR}).

In addition,  $\gss{X}{0}{}$ is distributed according to $P_{\sbms{X}{0}}$ and we write $\gss{X}{0}{}\sim P_{\sbms{X}{0}}$. $\gss{W}{k+1}{}$ are assumed to be i.i.d. vectors drawn from the product distribution $P_{\sbm{W}}=\mathrm{Ber}(\rho_w)^{\otimes p}$ with $\gss{W}{k+1}{}\perp \gss{X}{t}{}$ for any $t,k\in \mathbb{Z}_{\geq 0}$ with $t<k+1$, where $\perp$ denotes independence. As before, $\gss{w}{k+1}{}$ denotes a particular realization of $\gss{W}{k+1}{}$. Furthermore, we assume that each $\gss{a}{i}{}=[a_{i1},\ldots, a_{ip}]^T$ has support size $d_i$ such that $1\leq d_i\leq d$ and $a_{ij}\in [a_{min},1), \forall i \in [p], \forall j\in \mathcal{S}(i)$ for some $a_{min}\in (0,1)$. Additionally, $\g{B}=[b_{ij}]$ is considered to be diagonal with $b_{ii}\in (0,1)$ and $b_{ii}\geq b_{min}, \forall i\in [p]$. To ensure that the parameters of the Bernoulli random variables in (\ref{eq:BAR}) lie in the interval [0,1], we assume that
\begin{align}\label{eq:BARcoeffs}
\sum_{j=1}^{p}a_{ij}+b_{ii}=\sum_{j\in \mathcal{S}(i)}a_{ij}+b_{ii}=1,\: \text{for}\: \text{all}\: i\in [p].
\end{align}
These constraints show that $\g{\bar{A}}$ is a square substochastic matrix and $\g{B}$ is a doubly substochastic matrix\footnote{$\g{B}$ is diagonal.}. 
Moreover, the imposed constraints on the values of the parameters and (\ref{eq:BARcoeffs}) imply that $d$ is upper bounded by $d_{*}$, where $d_{*}$ is the integer solution of the following program:
\begin{align}
\max_{\bar{d}\in \mathbb{Z}_{>0}} a_{min}\bar{d}+b_{min}\ \ \ \ {\rm s.t.}\ \  a_{min}\bar{d}+b_{min}\leq 1.
\end{align}

The aforementioned rowwise interpretation of the $\mathrm{Ber}\left(\cdot\right)$ in (\ref{eq:BAR}) has the meaning that given the argument $\g{A}\g{f}(\gss{X}{k}{})+\g{B}\gss{W}{k+1}{}$, $[\gss{X}{k+1}{}]_i$ is independently drawn from $\mathrm{Ber}\left(\gss{a}{i}{T}f_i(\gss{X}{k}{})+b_{ii}[\gss{W}{k+1}{}]_i\right)$.
The term $\g{B}\gss{W}{k+1}{}$ ensures persistence of excitation in the model. In this sense, it is sufficient to assume that $\g{B}$ is diagonal. Nevertheless, extensions to more general $\g{B}$'s are possible. Note that a consequence of the assumption that all the $b_{ii}$'s are strictly positive is the prevention of the event that the model generates all-zeros or all-ones state vectors after a visit to a state that is all-zeros on $\cup_{i=1}^{p}\mathcal{S}^{+}(i)$ and all-ones on $\cup_{i=1}^{p}\mathcal{S}^{-}(i)$ or all-ones on $\cup_{i=1}^{p}\mathcal{S}^{+}(i)$ and all-zeros on $\cup_{i=1}^{p}\mathcal{S}^{-}(i)$, respectively, at a particular time instant. The random sequence $\g{\tilde{X}}=\{\gss{X}{n}{}\}_{n\geq 0}$ is referred to as the \emph{Bernoulli Autoregressive Process}.

We now let $\g{a}=\left[\gss{a}{1}{T},\gss{a}{2}{T},\ldots,\gss{a}{p}{T}\right]^{T}$. Then, $\g{a}\in \mathbb{R}^{p^2}$ with $a_{ij}=0$ if $\{j,i\}\notin \mathcal{E}$, where $\{j,i\}$ is the directed link from node $j$ to node $i$. We also let $\g{b}=\mathrm{diag}(\g{B})$ with $b_i=b_{ii}$, where `$b$' is used twice for notational convenience. Thus, for the graph $\mathcal{G}$, we define the set of valid parameter vectors as
{\small
\begin{align*}
\Omega_{a_{min},b_{min}}(\mathcal{G})=&\left\{\g{a}\in \mathbb{R}^{p^2},\g{b}\in \mathbb{R}^{p}: b_{ii}\geq b_{min}, a_{ij}\geq a_{min}\: \mathrm{for}\: \{j,i\}\in \mathcal{E}, a_{ij}=0\ \ \mathrm{for}\ \ \{j,i\}\notin \mathcal{E},\right.\\& \left.\sum_{j\in \mathcal{S}(i)}a_{ij}+b_{ii}=1\right\}.
\end{align*}
}

\section{BAR Learning}
\label{sec:BARodLearn}

Let $\vec{\mathcal{G}}_{p,d}$ be the set of all directed graphs with $p$ nodes, each node having at most $d$ parents. For any graph in $\vec{\mathcal{G}}_{p,d}$, we associate a sign $+$ or  $-$  with each edge from $j$ to $i$ to indicate whether $j$ belongs to $\mathcal{S}^{+}(i)$ or $\mathcal{S}^{-}(i)$ as in (\ref{eq:f}).  For some $\mathcal{G}\in \vec{\mathcal{G}}_{p,d}$ and  a pair of parameter vectors $(\g{a},\g{b})\in \Omega_{a_{min},b_{min}}(\mathcal{G})$, we assume that we initialize the system at $\gss{X}{0}{}\sim \pi$, where  $\pi$ is the stationary measure of the BAR process\footnote{We prove later on that every BAR process  has a unique invariant measure.}. We observe the sequence of correlated state vectors $\gss{X}{0}{},\gss{X}{1}{},\ldots, \gss{X}{n-1}{}\in \{0,1\}^p$ denoted by $\gss{X}{0:n-1}{}$, where each $\gss{X}{k}{}$ is generated by (\ref{eq:BAR}). A \emph{BAR structure observer} is a mapping: 
\begin{align*}
\psi:\left(\{0,1\}^p\right)^{n}\rightarrow \vec{\mathcal{G}}_{p,d}.
\end{align*}
The output 
\[
\left(\hat{\mathcal{S}},\g{\hat{f}}\right)=\left\{\left(\hat{\mathcal{S}}^{+}(i),\hat{\mathcal{S}}^{-}(i)\right)\right\}_{1:p}=\psi\left(\gss{X}{0:n-1}{}\right)
\]
is the observer's best estimate of the unsigned support of $\g{A}$ denoted by $\mathcal{S}=\{\mathcal{S}(1),\ldots,\mathcal{S}(p)\}$  and of $\g{f}(\cdot)$, where $\hat{\mathcal{S}}(i)=\hat{\mathcal{S}}^{+}(i)\cup\hat{\mathcal{S}}^{-}(i)$ for all $i\in [p]$. To evaluate the performance of the  BAR structure observer, we use the zero-one loss 
\[
\ell\left(\psi,\left\{\left(\mathcal{S}^{+}(i),\mathcal{S}^{-}(i)\right)\right\}_{1:p}\right)=\mathbb{I}\{\psi\neq\left\{\left(\mathcal{S}^{+}(i),\mathcal{S}^{-}(i)\right)\right\}_{1:p}\}.
\] 
Here, $\mathbb{I}\{\mathcal{A}\}$ denotes the indicator of the event $\mathcal{A}$. 
 The associated risk for some pair $(\g{a},\g{b})\in \Omega_{a_{min},b_{min}}(\mathcal{G})$ corresponding to $\mathcal{G}\in \vec{\mathcal{G}}_{p,d}$  is given by
$
\mathcal{R}_{\sbm{a},\sbm{b}}\left(\psi\right)=P_{\sbm{a},\sbm{b}}\left(\psi\neq\left\{\left(\mathcal{S}^{+}(i),\mathcal{S}^{-}(i)\right)\right\}_{1:p}\right)
$. 
For robustness, we focus on finding observers such that the \emph{worst case risk}
\[
\mathcal{R}_{*}(\psi)=\sup_{\mathcal{G}\in \vec{\mathcal{G}}_{p,d},  (\sbm{a},\sbm{b})\in \Omega_{a_{min},b_{min}}(\mathcal{G})}P_{\sbm{a},\sbm{b}}\left(\psi\neq\left\{\left(\mathcal{S}^{+}(i),\mathcal{S}^{-}(i)\right)\right\}_{1:p}\right)
\]
tends to zero as $p\rightarrow \infty$ with the least number of samples $n$. 

\subsection{BAR Structure Observer}
\label{subsec:BARalg}

The proposed BAR structure observer operates in two stages: 
\begin{itemize}
\item \emph{Supergraph Selection}: A supergraph of the actual graph is obtained.
\item \emph{Supergraph Trimming}: The obtained supergraph from the previous stage is reduced to the actual graph by excluding nodes from the neighborhood of each node with no causal influence to this node.   
\end{itemize}
The above two stages will be successful with high probability for a sufficient sample complexity.

\subsubsection{Supergraph Selection}

This stage will be based on the following measure of \emph{conditional influence}:
\begin{align}\label{eq:nuij}
\nu_{i|j}=P(X_i^{+1}=1|X_j=1)-P(X_i^{+1}=1|X_j=0),\forall i,j\in \mathcal{V}.
\end{align} 
Here, $X_i^{+1}$ and $X_j$ refer to successive time instants. 
Also, the underlying measure in defining $\nu_{i|j}$ is the stationary, thus we have dropped the temporal indices in the involved random variables.   
The main reason motivating this metric is the desire of a good tradeoff between computational complexity and statistical efficiency in squeezing out structural information from the observed data.   
Similar measures have been used in prior work to define structure estimators, see for example \cite{athw12,afw12, b15,bgs15}. The difference in $\nu_{i|j}$ from prior measures is that the \emph{conditioning does not account for other nodes that possibly belong to the neighborhood of the} $i$\emph{th node}. The key point here is that $\nu_{i|j}$ can be nonzero even when $j\notin \mathcal{S}(i)$, since $X_j$ can be correlated with some or all $X_l, l\in \mathcal{S}(i)$. On the other hand, similar metrics in prior literature aim at defining $\nu_{i|j}$ by conditioning also with respect to $X_l, l\in \mathcal{S}(i)$. Clearly, if $j\notin \mathcal{S}(i)$ and $\nu_{i|j}$ is defined as 
\[
\nu_{i|j}=P(X_i^{+1}=1|X_j=1,\g{X}_{\hat{\mathcal{S}}(i)}=\gss{x}{\hat{\mathcal{S}}(i)}{})-P(X_i^{+1}=1|X_j=0,\g{X}_{\hat{\mathcal{S}}(i)}=\gss{x}{\hat{\mathcal{S}}(i)}{}),
\] 
then 
$
\nu_{i|j}=P(X_i^{+1}=1|\g{X}_{\hat{\mathcal{S}}(i)}=\gss{x}{\hat{\mathcal{S}}(i)}{})-P(X_i^{+1}=1|\g{X}_{\hat{\mathcal{S}}(i)}=\gss{x}{\hat{\mathcal{S}}(i)}{})=0
$
if $\mathcal{S}(i)\subseteq \hat{\mathcal{S}}(i)$
due to the Markov property and thus, for $\mathcal{S}(i)\subseteq \hat{\mathcal{S}}(i)$ and $j\notin \mathcal{S}(i)$, $\nu_{i|j}$ can \emph{only} be zero.
Our subsequent analysis shows that there is no significant loss in the BAR case from considering only conditioning with respect to a single node every time.  
Intuitively, the magnitude of $\nu_{i|j}$ quantifies what is the probability that $X_j$ causally affects $X_i$ either positively or negatively over 
the case of no influence. A more rigorous motivation for defining $\nu_{i|j}$ as the difference of the two conditional probabilities appearing in (\ref{eq:nuij}) is by considering reductions of relevant entropies. More specifically, pick an $i\in \mathcal{V}$ and let
\[
H_{ij,1}(\g{p})=H(X_i^{+1}|X_j=1)=-\sum_{x_i\in \{0,1\}}P(X_i^{+1}=x_i|X_j=1)\log P(X_i^{+1}=x_i|X_j=1)
\]   
and
\[
H_{ij,0}(\g{p}')=H(X_i^{+1}|X_j=0)=-\sum_{x_i\in \{0,1\}}P(X_i^{+1}=x_i|X_j=0)\log P(X_i^{+1}=x_i|X_j=0).
\]   
$H_{ij,1}(\g{p})$ characterizes the residual randomness in $X_i^{+1}$ when $X_j=1$ and $H_{ij,0}(\g{p}')$ quantifies the corresponding 
residual randomness when $X_j=0$. If $X_i^{+1}$ is (causally) independent of $X_j$, i.e., $j\notin \mathcal{S}(i)$ and $X_j$ is independent of \emph{all} $X_l, l\in \mathcal{S}(i)$, then $P(X_i^{+1}=x_i|X_j=1)=P(X_i^{+1}=x_i|X_j=0)=P(X_i^{+1}=x_i)$ and $H_{ij,1}(\g{p})=H_{ij,0}(\g{p}')=H(X_i^{+1})$. Thus, 
$
H_{ij,1}(\g{p})-H_{ij,0}(\g{p}')=0
$
in this case. At the same time, $\nu_{i|j}=0$ and therefore, (causal) independence between $X_i^{+1}$ and $X_j$ is totally characterized by $\nu_{i|j}$. In the general case, consider the \emph{causal} entropy $H(X_i^{+1}|X_j)$ and observe that
\begin{align*}
H(X_i^{+1}|X_j)&=P(X_j=1)H_{ij,1}(\g{p})+P(X_j=0)H_{ij,0}(\g{p}')
\\&=P(X_j=1)(H_{ij,1}(\g{p})-H_{ij,0}(\g{p}'))+H_{ij,0}(\g{p}')
\\&=P(X_j=0)(H_{ij,0}(\g{p}')-H_{ij,1}(\g{p}))+H_{ij,1}(\g{p}).
\end{align*}
Thus, the differences $H_{ij,1}(\g{p})-H_{ij,0}(\g{p}')$ and $H_{ij,0}(\g{p}')-H_{ij,1}(\g{p})$ more or less control the magnitude of  $H(X_i^{+1}|X_j)$. Considering without loss of generality only the first difference, we have the following result:
\begin{lem}\label{lem:nuij}
The difference $H_{ij,1}(\g{p})-H_{ij,0}(\g{p}')$ critically depends on $\nu_{i|j}$.
\end{lem} 
\begin{proof}
Appendix \ref{app:2}.
\end{proof}
 
In practice, we will work the empirical analogues
\[
\hat{\nu}_{i|j}=\hat{P}(X_i^{+1}=1|X_j=1)-\hat{P}(X_i^{+1}=1|X_j=0),\forall i,j\in \mathcal{V}, 
\]
where 
$
\hat{P}(X_i^{+1}=x_i|X_j=x_j)=\frac{\hat{P}(X_i^{+1}=x_i, X_j=x_j)}{\hat{P}(X_j=x_j)},\ \  \hat{P}(X_j=x_j)=n^{-1}\sum_{k=0}^{n-1}\mathbb{I}\{[\gss{X}{k}{}]_j=x_j\}
$
and similarly for $\hat{P}(X_i^{+1}=x_i, X_j=x_j)$. 

\subsubsection{Supergraph Trimming}

For the trimming stage, we assume that 
the upper bound $d$ is \emph{known}, which is a very reasonable assumption\footnote{in the worst scenario, $d=d_{*}$.}, and has been used by  the supergraph selection stage to deliver an overestimate $\hat{\mathcal{S}}=\{\hat{\mathcal{S}}(1),\hat{\mathcal{S}}(2),\ldots, \hat{\mathcal{S}}(p)\}$ with \emph{unknown edge labels} (i.e., $\hat{\mathcal{S}}^{+}(m)$ and $\hat{\mathcal{S}}^{-}(m)$ have not been estimated yet). For a sufficiently large $n$, the true graph is a subgraph of $\hat{\mathcal{S}}$ with high probability. \emph{Hence, our starting point is the assumption that the true graph is contained in $\hat{\mathcal{S}}$}. Let $\hat{\mathcal{S}}(i)=\{l_1<l_2<\cdots<l_d\}$ and assume without loss of generality that $\mathcal{S}(i)=\{l_1,l_2,\ldots, l_{d_i}\}$. Then,
\begin{align}\label{eq:extension1}
P\left(X_i^{+1}=1|\gss{X}{\hat{\mathcal{S}}(i)}{}=\gss{x}{\hat{\mathcal{S}}(i)}{}\right)=\sum_{l_k\in \mathcal{S}(i)}a_{il_k}f_i([\gss{x}{\hat{\mathcal{S}}(i)}{}]_k)+b_i\rho_w.
\end{align}
If $\gss{x}{\hat{\mathcal{S}}(i)}{}$ takes all binary vector values, then there is at least one binary vector $\gss{x}{\hat{\mathcal{S}}(i)}{*}$ such that 
\[
P\left(X_i^{+1}=1|\gss{X}{\hat{\mathcal{S}}(i)}{}=\gss{x}{\hat{\mathcal{S}}(i)}{*}\right)=\sum_{l_k\in \mathcal{S}(i)}a_{il_k}+b_i\rho_w.
\]
Therefore, $\gss{x}{\hat{\mathcal{S}}(i)}{*}$ corresponds to the maximum value that $P\left(X_i^{+1}=1|\gss{X}{\hat{\mathcal{S}}(i)}{}=\gss{x}{\hat{\mathcal{S}}(i)}{}\right)$ can take. If $\gss{x}{\hat{\mathcal{S}}(i)}{*}$ is unique, then we can safely conclude that $\hat{\mathcal{S}}_f(i)=\hat{\mathcal{S}}(i)=\mathcal{S}(i)$ and $d_i=d$. Here, the subscript $\cdot _f$ is reserved for ``final estimates''. Furthermore, we can immediately extract $\hat{\mathcal{S}}_f^{+}(i)$ and $\hat{\mathcal{S}}_f^{-}(i)$ by placing in $\hat{\mathcal{S}}_f^{+}(i)$ all $l\in \hat{\mathcal{S}}_f(i)$ corresponding to $1$ in $\gss{x}{\hat{\mathcal{S}}(i)}{*}$ and by placing in $\hat{\mathcal{S}}_f^{-}(i)$ all $l\in \hat{\mathcal{S}}_f(i)$ corresponding to $0$ in $\gss{x}{\hat{\mathcal{S}}(i)}{*}$. If $\gss{x}{\hat{\mathcal{S}}(i)}{*}$ is not unique, then we conclude that 
$d_i<d$ and there are $l\in \hat{\mathcal{S}}(i)$ such that $l\notin \mathcal{S}(i)$. In this case, we collect all binary vectors corresponding to the maximum value of $P\left(X_i^{+1}=1|\gss{X}{\hat{\mathcal{S}}(i)}{}=\gss{x}{\hat{\mathcal{S}}(i)}{}\right)$, namely $\gss{x}{\hat{\mathcal{S}}(i)}{1,*},\gss{x}{\hat{\mathcal{S}}(i)}{2,*},\ldots,\gss{x}{\hat{\mathcal{S}}(i)}{s,*}$, where $s=2^{d-d_i}$. When $l\in \hat{\mathcal{S}}(i)$ but $l\notin \mathcal{S}(i)$, there will be at least two binary 
vectors among the maximizers in which $l$ appears with the values $1$ and $0$. This observation shows that we can estimate the final neighborhood by placing in $\hat{\mathcal{S}}_f(i)$ all $l\in \hat{\mathcal{S}}(i)$ corresponding to either \emph{only} the value $1$ in all maximizers or to \emph{only}    
the value $0$ in all maximizers. Also, we can immediately extract $\hat{\mathcal{S}}_f^{+}(i)$ and $\hat{\mathcal{S}}_f^{-}(i)$ by placing in $\hat{\mathcal{S}}_f^{+}(i)$ all $l\in \hat{\mathcal{S}}_f(i)$ corresponding to $1$ in all maximizers and by placing in $\hat{\mathcal{S}}_f^{-}(i)$ all $l\in \hat{\mathcal{S}}_f(i)$ corresponding to $0$ in all maximizers. 

In practice, we will empirically estimate $P\left(X_i^{+1}=1|\gss{X}{\hat{\mathcal{S}}(i)}{}=\gss{x}{\hat{\mathcal{S}}(i)}{}\right)$ by 
\[
\hat{P}\left(X_i^{+1}=1|\gss{X}{\hat{\mathcal{S}}(i)}{}=\gss{x}{\hat{\mathcal{S}}(i)}{}\right)=\frac{\hat{P}\left(X_i^{+1}=1,\gss{X}{\hat{\mathcal{S}}(i)}{}=\gss{x}{\hat{\mathcal{S}}(i)}{}\right)}{\hat{P}\left(\gss{X}{\hat{\mathcal{S}}(i)}{}=\gss{x}{\hat{\mathcal{S}}(i)}{}\right)},
\]
where
$
\hat{P}\left(\gss{X}{\hat{\mathcal{S}}(i)}{}=\gss{x}{\hat{\mathcal{S}}(i)}{}\right)=n^{-1}\sum_{k=0}^{n-1}\mathbb{I}\{\gss{X}{\hat{\mathcal{S}}(i)}{}=\gss{x}{\hat{\mathcal{S}}(i)}{}\}
$
and similarly for $\hat{P}\left(X_i^{+1}=1,\gss{X}{\hat{\mathcal{S}}(i)}{}=\gss{x}{\hat{\mathcal{S}}(i)}{}\right)$. Empirical estimates cause the problem that 
almost surely $\hat{P}\left(X_i^{+1}=1|\gss{X}{\hat{\mathcal{S}}(i)}{}=\gss{x}{\hat{\mathcal{S}}(i)}{}\right)$ will be maximized by a \emph{single} binary vector  $\gss{\hat{x}}{\hat{\mathcal{S}}(i)}{*}$, while there might be binary vectors corresponding to less variables than in the true neighborhood, yielding a value of $\hat{P}\left(X_i^{+1}=1|\gss{X}{\hat{\mathcal{S}}(i)}{}=\gss{x}{\hat{\mathcal{S}}(i)}{}\right)$ close to the maximum. Here, $\hat{\cdot}$ is used on $\g{x}$ to denote that we refer to a maximizer of the empirical conditional probability measure. These problems can be resolved by observing the following:
\begin{itemize}
\item As $n\rightarrow \infty$ , $\hat{P}\left(X_i^{+1}=1|\gss{X}{\hat{\mathcal{S}}(i)}{}=\gss{x}{\hat{\mathcal{S}}(i)}{}\right)\in \mathcal{B}\left(P\left(X_i^{+1}=1|\gss{X}{\hat{\mathcal{S}}(i)}{}=\gss{x}{\hat{\mathcal{S}}(i)}{}\right),\tilde{\varepsilon}\right)$ for some $\tilde{\varepsilon}\rightarrow 0$. Here, $\mathcal{B}(c,r)$ stands for a ball with center $c$ and radius $r>0$. Therefore, for sufficiently large $n$, only the actual maximizers $\gss{x}{\hat{\mathcal{S}}(i)}{1,*},\gss{x}{\hat{\mathcal{S}}(i)}{2,*},\ldots,\gss{x}{\hat{\mathcal{S}}(i)}{s,*}$ will yield almost maximum values for $\hat{P}\left(X_i^{+1}=1|\gss{X}{\hat{\mathcal{S}}(i)}{}=\gss{x}{\hat{\mathcal{S}}(i)}{}\right)$.
\item Every time a node $l_k\in \mathcal{S}(i)$ does not participate in (\ref{eq:extension1}) with the appropriate polarity, $P\left(X_i^{+1}=1|\gss{X}{\hat{\mathcal{S}}(i)}{}=\gss{x}{\hat{\mathcal{S}}(i)}{}\right)$ is reduced by at least $a_{min}$ according to our assumptions. In other words, the possible distinct values that $P\left(X_i^{+1}=1|\gss{X}{\hat{\mathcal{S}}(i)}{}=\gss{x}{\hat{\mathcal{S}}(i)}{}\right)$ can take for all possible binary vectors $\gss{x}{\hat{\mathcal{S}}(i)}{}$ differ by at least $a_{min}$. Assuming that the maximum value of $P\left(X_i^{+1}=1|\gss{X}{\hat{\mathcal{S}}(i)}{}=\gss{x}{\hat{\mathcal{S}}(i)}{}\right)$ is $v^{*}$, then for sufficiently large $n$, we will have that $\hat{P}\left(X_i^{+1}=1|\gss{X}{\hat{\mathcal{S}}(i)}{}=\gss{x}{\hat{\mathcal{S}}(i)}{l,*}\right)\in \mathcal{B}\left(v^{*},\tilde{\varepsilon}\right)$ for some $\tilde{\varepsilon}<a_{min}/2$. We can therefore pick 
$\gss{\hat{x}}{\hat{\mathcal{S}}(i)}{1,*}$, which is an estimate of $\gss{x}{\hat{\mathcal{S}}(i)}{1,*}$, as the binary vector delivering the maximum empirical conditional probability and the rest of the maximizers as those binary vectors giving empirical conditional probability values within $2\tilde{\varepsilon}$ from $\hat{P}\left(X_i^{+1}=1|\gss{X}{\hat{\mathcal{S}}(i)}{}=\gss{\hat{x}}{\hat{\mathcal{S}}(i)}{1,*}\right)$.  
\end{itemize}  

Having described the above procedure, the final point to consider is the choice of an appropriate threshold such that the picked maximizers are the correct ones. If $\gss{\hat{x}}{\hat{\mathcal{S}}(i)}{1,*}$ corresponds to $\approx v^{*}+\tilde{\varepsilon}$, then the interval $(v^{*}-\tilde{\varepsilon},v^{*}+\tilde{\varepsilon})$ contains all the maximizers for sufficiently large $n$.  If $\gss{\hat{x}}{\hat{\mathcal{S}}(i)}{1,*}$ corresponds to  $\approx v^{*}-\tilde{\varepsilon}$, then we must make sure that the interval $(v^{*}-3\tilde{\varepsilon},v^{*}-\tilde{\varepsilon})$ contains no points. This leads to the conclusion that $\tilde{\varepsilon}$ should be at most $a_{min}/4$.

\subsubsection{The Algorithm}

 The proposed BAR structure observer, which learns the support of $\g{A}$ and $\g{f}$ when the in-degrees of all nodes are upper bounded by $d$, is given by Alg. \ref{alg:Obs1}.

\begin{algorithm}[H]

\caption{BARObs$\left(\gss{X}{0:n-1}{},d,\tau\leq a_{min}/4\right)$} %

\label{alg:Obs1} %


\begin{algorithmic}[1]
\State \textbf{Supergraph Selection}
\newline
\For{each $\gss{a}{m}{T}$}
\State compute and order $|\hat{\nu}_{m|l}|$ for $l\in [p]$: $|\hat{\nu}_{m|(1)}|\geq \cdots \geq |\hat{\nu}_{m|(p)}|$
\State Choose $\hat{\mathcal{S}}(m)=\{(1),(2),\ldots, (d)\}$. In the worst case $d=d^{*}$
\EndFor
\newline
\State \textbf{Supergraph Trimming}
\newline
\For{each $\gss{a}{m}{T}$}
\State Compute and order the $2^d$ conditional probabilities $\hat{P}\left(X_m^{+1}=1|\gss{X}{\hat{\mathcal{S}}(m)}{}=\gss{x}{\hat{\mathcal{S}}(m)}{}\right)$.
Let the maximum empirical conditional probability value be $\hat{v}^{*}$ 

\State Choose $\gss{\hat{x}}{\hat{\mathcal{S}}(m)}{1,*},\gss{\hat{x}}{\hat{\mathcal{S}}(m)}{2,*},\ldots,\gss{\hat{x}}{\hat{\mathcal{S}}(m)}{\hat{s},*}$ (estimates of $s$ and $\gss{x}{\hat{\mathcal{S}}(m)}{1,*},\gss{x}{\hat{\mathcal{S}}(m)}{2,*},\ldots,\gss{x}{\hat{\mathcal{S}}(m)}{s,*}$) as those vectors corresponding 
to $\hat{P}\left(X_m^{+1}=1|\gss{X}{\hat{\mathcal{S}}(m)}{}=\gss{\hat{x}}{\hat{\mathcal{S}}(m)}{l,*}\right)>\hat{v}^{*}-2\tau$

\State Pick $\hat{\mathcal{S}}_f(m)$ as those $l\in \hat{\mathcal{S}}(m)$ that correspond to only $1$'s or only $0$'s in $\gss{\hat{x}}{\hat{\mathcal{S}(m)}}{1,*},\gss{\hat{x}}{\hat{\mathcal{S}}(m)}{2,*},\ldots,\gss{\hat{x}}{\hat{\mathcal{S}}(m)}{\hat{s},*}$. Let $\hat{d}_m$ be the size of $\hat{\mathcal{S}}_f(m)$

\State Set $\hat{\mathcal{S}}_f^{+}(m)=\emptyset$ and $\hat{\mathcal{S}}_f^{-}(m)=\emptyset$

\For {$l\in \hat{\mathcal{S}}_f(m)$}

\If{$l$ corresponds to $1$ in all maximizers} $\hat{\mathcal{S}}_f^{+}(m)=\hat{\mathcal{S}}_f^{+}(m)\cup \{l\}$

\ElsIf{$l$ corresponds to $0$ in all maximizers} $\hat{\mathcal{S}}_f^{-}(m)=\hat{\mathcal{S}}_f^{-}(m)\cup \{l\}$

\EndIf

\EndFor

\EndFor

\end{algorithmic}

\end{algorithm}

$\newline$

\textbf{Remarks}:
\begin{enumerate}
\item BARObs$\left(\gss{X}{0:n-1}{},d,\tau\leq a_{min}/4\right)$ can be possibly stopped upon the termination of the supergraph selection stage if an overestimate of the actual graph is sufficient for the application at hand. Overestimates of $\mathcal{S}^{+}$ and $\mathcal{S}^{-}$ can be then obtained by 
placing in each $\hat{\mathcal{S}}^{+}(m)$ all $l\in \hat{\mathcal{S}}(m)$ such that $\hat{\nu}_{m|l}>0$ and  in each $\hat{\mathcal{S}}^{-}(m)$ all $l\in \hat{\mathcal{S}}(m)$ such that $\hat{\nu}_{m|l}<0$. The sufficient sample complexity such that \emph{only} the supergraph selection stage is successful with probability at least $1-\gamma$ is \emph{lower} than the sufficient sample complexity such that  the supergraph selection \emph{and} the supergraph trimming stages are both successful with probability at least $(1-\gamma)^2\geq 1-2\gamma$. 
\item In addition to the previous remark,  BARObs$\left(\gss{X}{0:n-1}{},d,\tau\leq a_{min}/4\right)$ can in principle be stopped upon the termination of the supergraph selection stage with an estimate of the \emph{actual} graph, if the individual node degrees $d_1,d_2,\ldots, d_p$ are \emph{a priori} known. In this scenario,  $\hat{\mathcal{S}}(m)=\{(1),(2),\ldots, (d_m)\}$. Moreover, $\mathcal{S}^{+}$ and $\mathcal{S}^{-}$ can be obtained by 
placing in each $\hat{\mathcal{S}}^{+}(m)$ all $l\in \hat{\mathcal{S}}(m)$ such that $\hat{\nu}_{m|l}>0$ and  in each $\hat{\mathcal{S}}^{-}(m)$ all $l\in \hat{\mathcal{S}}(m)$ such that $\hat{\nu}_{m|l}<0$. The comment about the required sample complexity for success with probability at least $1-\gamma$ remains the same as in the previous remark.     
         
\item In the case of (\ref{eq:BAR2}), lines $11$ to $16$ in BARObs$\left(\gss{X}{0:n-1}{},d,\tau\leq a_{min}/4\right)$  are eliminated. In this setup, only the unsigned support of $\g{A}$ is meaningful. 
\end{enumerate}

\section{Main Results}
\label{sec:MainResults}

Our main results concern the mixing properties of BAR processes and the sample complexity of  Alg. \ref{alg:Obs1}. As we analyze the sample complexity of 
 BARObs$\left(\gss{X}{0:n-1}{},d,\tau\leq a_{min}/4\right)$, we also provide sample complexity results for instances in which the algorithm is stopped upon the termination of the supergraph selection stage, as noted in the first two remarks after  Alg. \ref{alg:Obs1}.
 
\subsection{Mixing Time Bound for a General BAR Process}  

A very critical property of the BAR model is that the mixing is rapid. This property is summarized by the following theorem: 
\begin{thrm}\label{thm:BARMxingMain}
Consider the general BAR model (\ref{eq:BAR}). Then, 
{\small  
\begin{align}\label{eq:GeneralBARMixingTime}
t_{mix}(\theta)\leq \left\lceil\frac{\log\left(\frac{\theta (1-\max_{1\leq i\leq p}\sum_{j=1}^pa_{ij})}{p}\right)}{\log\left(\max_{1\leq i\leq p}\sum_{j=1}^pa_{ij}\right)}\right\rceil\leq  \left\lceil\frac{1}{1-\max_{1\leq i\leq p}\sum_{j=1}^pa_{ij}}\left(\log p -\log \left(\theta \left(1-\max_{1\leq i\leq p}\sum_{j=1}^pa_{ij}\right)\right)\right)\right\rceil 
\end{align}
}

for any $\theta \in (0,1)$, i.e., $t_{mix}(\theta)=O(\log p)$.
\end{thrm} 

The previous result holds for the BAR model (\ref{eq:BAR2}) as a special case of (\ref{eq:BAR}).

\noindent \textbf{Remark}: Note that $\max_{1\leq i\leq p}\sum_{j=1}^pa_{ij}$ corresponds to the maximum \emph{row sum} of $\g{\bar{A}}$ (or $\g{A}$), which is less than $1$ by definition.

\subsection{Sample Complexity and Correctness of BARObs$\left(\gss{X}{0:n-1}{},d,\tau\leq a_{min}/4\right)$}

The success of BARObs$\left(\gss{X}{0:n-1}{},d,\tau\leq a_{min}/4\right)$ in determining the actual graph depends on a mild condition that we define 
in subsubsection \ref{subsubsec:IdCodBAR} and we call \emph{BAR Identifiability Condition}. The sample complexity and the correctness of Alg. \ref{alg:Obs1} are summarized by the following theorem: 

\begin{thrm}\label{thm:ExtendedBARObs}
Let $\mathcal{G}\in \vec{\mathcal{G}}_{p,d}$, $(\g{a},\g{b})\in \Omega_{a_{min},b_{min}}(\mathcal{G})$ and $\gss{X}{0}{}\sim \pi$, where $\pi$ is the stationary measure. Suppose that we observe 
the BAR sequence $\gss{X}{0}{},\gss{X}{1}{},\ldots, \gss{X}{n-1}{}$ for $n$ given by 
\begin{align}\label{eq:nRequired_Extended11}
  n\geq 1+\frac{288\log\left(\frac{2^{d+1}C}{\gamma} p{p\choose d}\right)t_{mix}(\theta)}{\tilde{\varepsilon}^2\bar{\beta}^3}.
\end{align}
where $C$ is some constant and $\tilde{\varepsilon}=\tau$.  
Assume that the BAR Identifiability Condition holds and $\tilde{\varepsilon}\leq \varepsilon<\min_{m\in [p]}\frac{\chi_m}{2}$, where $\varepsilon$ is a parameter characterizing with high probability the accuracy of $\hat{\nu}_{i|j}$ used by the supergraph selection stage via $|\hat{\nu}_{i|j}-\nu_{i|j}|<\varepsilon$ and $\{\chi_m\}_{m\in [p]}$ is a set of parameters associated with the BAR Identifiability Condition. Then, BARObs$\left(\gss{X}{0:n-1}{},d,\tau\leq a_{min}/4\right)$ correctly identifies the true graph with probability at least $(1-\gamma)^2\geq 1-2\gamma$.
\end{thrm} 

We note here that the sample complexity given by (\ref{eq:nRequired_Extended11}) has an \emph{exponential} dependence on $d$ via $\bar{\beta}\propto (1/\bar{c})^d$ for some $\bar{c}>1$. Therefore, $n$ scales as $\bar{c}^{3d}$.
Exponential dependence of the sample complexity on $d$ is very usual; in \cite{b15}, the sample complexity has a \emph{double-exponential} dependence on $d$, while exponential dependence on $d$ is unavoidable, e.g., in Ising models as a lower information-theoretic bound derived in \cite{sw12} shows. 

\textbf{Sample Complexity Interpretation}: Consider the static model $\g{Y}=\mathrm{Ber}(\g{A}\g{f}(\g{X})+\g{B}\g{W})$, where $\g{X}\in \{0,1\}^p$ and $\g{X}\sim P_{\sbm{X}}$. Assume that we observe $n$ i.i.d pairs $(\gss{X}{0}{},\gss{Y}{0}{}), (\gss{X}{1}{},\gss{Y}{1}{}),\ldots,$ $(\gss{X}{n-1}{},\gss{Y}{n-1}{})$. 
The required sample complexity is determined by the rate of convergence of the empirical probabilities to the true probabilities.
In this scenario, a sufficient number of samples can be shown to be $n=\Omega(\log p)$. In the dynamic scenario considered here, 
(\ref{eq:nRequired_Extended11}) shows that a sufficient sample complexity  is $\Omega(t_{mix}\log p)$. Again, the sample complexity is determined by the rate that empirical probabilities converge to the true probabilities. Since the BAR model generates almost i.i.d. samples 
every $t_{mix}$ steps, (\ref{eq:nRequired_Extended11}) is consistent with the static scenario result.

\section{Proof of Theorem \ref{thm:BARMxingMain}}
\label{sec:ProofMainThm1}

We first prove a very useful result for the subsequent derivation: 

\begin{lem}\label{lem:CoreLemma}
Let $Z_1,Z_2,\ldots,Z_p$ be random variables such that $0<z_{min}\leq Z_i\leq z_{max}$ almost surely. Assume that $z_{min}\in (0,1)$ and let 
$k\leq p$ random variables among the $Z_i$'s, specifically $Z_1,Z_2,\ldots, Z_k$ without loss of generality, be in $[z_{min},1]$ almost surely and 
$Z_{k+1},Z_{k+2},\ldots, Z_p$ be in $[1,z_{max}]$ almost surely. Then,
\begin{equation}\label{eq:coreLemma1}
E\left[\prod_{i=1}^pZ_i\right]\geq \left(\prod_{i=1}^kE[Z_i]\right)^{\frac{1}{z_{min}}}  \left(\prod_{i=k+1}^pE[Z_i]\right)^{\frac{1}{z_{max}}}.
\end{equation}
In particular: 
\begin{itemize}
\item If $0<z_{min}\leq Z_i\leq 1, \forall i\in [p]$ almost surely, then 
 \begin{equation}\label{eq:coreLemma2}
E\left[\prod_{i=1}^pZ_i\right]\geq \left(\prod_{i=1}^pE[Z_i]\right)^{\frac{1}{z_{min}}}.
\end{equation}
\item If $1\leq Z_i\leq z_{max}, \forall i\in [p]$ almost surely, then 
 \begin{equation}\label{eq:coreLemma3}
E\left[\prod_{i=1}^pZ_i\right]\geq \left(\prod_{i=1}^pE[Z_i]\right)^{\frac{1}{z_{max}}}.
\end{equation}
\end{itemize}
\end{lem}

\begin{proof}
Without loss of generality, let $Z_1,Z_2,\ldots, Z_k$ be in $[z_{min},1]$ almost surely and 
$Z_{k+1},Z_{k+2},\ldots, Z_p$ be in $[1,z_{max}]$ almost surely. Then, 
\begin{align*}
&E\left[\prod_{i=1}^pZ_i\right]=\exp\left(\log\left\{E\left[\prod_{i=1}^pZ_i\right]\right\}\right)\underbrace{\geq}_{(a)}\exp\left(E\left[\sum_{i=1}^p\log Z_i\right]\right)\underbrace{\geq}_{(b)}\exp\left(E\left[\sum_{i=1}^p\frac{Z_i-1}{Z_i}\right]\right)\underbrace{\geq}_{(c)}\\
& \exp\left(E\left[\sum_{i=1}^k\frac{Z_i-1}{z_{min}}+\sum_{i=k+1}^p\frac{Z_i-1}{z_{max}}\right]\right)=\exp\left(\sum_{i=1}^k\frac{E[Z_i]-1}{z_{min}}\right)\exp\left(\sum_{i=k+1}^p\frac{E[Z_i]-1}{z_{max}}\right)\underbrace{\geq}_{(d)}\\& 
\exp\left(\frac{1}{z_{min}}\sum_{i=1}^k\log E[Z_i]\right) \exp\left(\frac{1}{z_{max}}\sum_{i=k+1}^p\log E[Z_i]\right)=\exp\left(\log\left\{\left(\prod_{i=1}^kE[Z_i]\right)^{\frac{1}{z_{min}}}\right\}\right)\cdot\\ &
\exp\left(\log\left\{\left(\prod_{i=k+1}^pE[Z_i]\right)^{\frac{1}{z_{max}}}\right\}\right)=\left(\prod_{i=1}^kE[Z_i]\right)^{\frac{1}{z_{min}}}\left(\prod_{i=k+1}^pE[Z_i]\right)^{\frac{1}{z_{max}}},  
\end{align*}
where:
\begin{itemize}
\item In (a), we have used Jensen's inequality.
\item In (b),  we have employed the inequality $\log(x)\geq \frac{x-1}{x}$ for any $x>0$.
\item In (c), we have noted that $-1<Z_i-1\leq 0$ almost surely for $i=1,2,\ldots,k$ and $0\leq Z_i-1$ almost surely for $i=k+1,k+2,\ldots, p$. 
\item In (d), for $i=1,2,\ldots, p$ we have employed the inequality $x\geq  \log(1+x)$, which holds for any $x>-1$.
\end{itemize}

Based on this result, (\ref{eq:coreLemma2}) and (\ref{eq:coreLemma3}) follow immediately by setting $k=p$ and $k=0$, respectively. 
\end{proof}

\textbf{Remark}: Note that $1$ is included to both $[z_{min},1]$ and $[1,z_{max}]$. If some of the $Z_i$'s have all their mass placed at $1$, then they can participate to the inequality either with the $1/z_{min}$ or the $1/z_{max}$ exponent, since both their realized values and their mean values are $1$. 

$\newline$

We now bound the mixing time of the BAR model (\ref{eq:BAR})  based on an appropriate coupling.

\subsection*{Choice of Coupling}
\label{subsec:Coupling}

Let $\g{\tilde{X}}=\{\gss{X}{k}{}\}$ be a copy of the BAR chain $\gss{X}{k+1}{}=\mathrm{Ber}\left(\g{A}\g{f}(\gss{X}{k}{})+\g{B}\gss{W}{k+1}{}\right)$, which is arbitrarily initialized at some point of the state space $\mathcal{X}=\{0,1\}^{p}$, say $\gss{x}{0}{}$. Let also $\g{\tilde{Y}}=\{\gss{Y}{k}{}\}$ be a different copy of the 
same BAR chain initialized at $\gss{y}{0}{}$, which is chosen according to the stationary measure $\pi$. We will upper bound the mixing time of the BAR chain
using the following coupling that respects the transitions of the BAR model:

\emph{Coupling}:

\begin{enumerate}
\item At every time instant $k$, we sample $\gss{W}{k}{}$ from $\mathrm{Ber}(\rho_w)^{\otimes p}$ and we feed this vector to both $\g{\tilde{X}}$ and $\g{\tilde{Y}}$.
\item At every time instant $k$, we draw i.i.d. random variables $U_1,U_2,\ldots, U_p\sim \mathrm{Unif}[0,1]$. We let 
\[
[\gss{X}{k}{}]_i=\mathbb{I}\left\{U_i\leq \sum_{j=1}^pa_{ij}[f_i(\gss{X}{k-1}{})]_j+b_{ii}[\gss{W}{k}{}]_i\right\}
\]
and 
\[
[\gss{Y}{k}{}]_i=\mathbb{I}\left\{U_i\leq \sum_{j=1}^pa_{ij}[f_i(\gss{Y}{k-1}{})]_j+b_{ii}[\gss{W}{k}{}]_i\right\}.
\]
\end{enumerate}

It is straightforward to see that, individually, the processes $\g{\tilde{X}}$ and $\g{\tilde{Y}}$ preserve the transitions of the BAR model (\ref{eq:BAR}). Also, it is immediate to see
that upon the event $\{\gss{X}{m}{}=\gss{Y}{m}{}\}$ for some $m$, the above coupling, by definition, leads to 
\begin{equation}\label{eq:Coupling1}
\gss{X}{k}{}=\gss{Y}{k}{},\: \forall k\geq m.
\end{equation}

\subsection*{Mixing Time Bound Strategy}

The \emph{maximal distance to stationarity} is defined as \cite{lpw08}:
\[
d(n)=\max_{\sbms{x}{0}}\|\g{P}^n(\cdot|\gss{x}{0}{})-\pi\|_{\rm TV},
\]
where $\g{P}^n(\cdot|\gss{x}{0}{})$ denote the $n-$th step transition probabilities of the BAR chain when initialized at $\gss{x}{0}{}$. $\|\cdot\|_{\rm TV}$ is the total variation distance, which, for any two measures $\mu,\nu$ on $\mathcal{X}$, is defined as $\|\mu-\nu\|_{\rm TV}=\frac{1}{2}\sum_{\sbm{x}\in \mathcal{X}}|\mu(\g{x})-\nu(\g{x})|$ \cite{lpw08}. Moreover, 
\[
\bar{d}(n)=\max_{\sbms{x}{0},\sbms{y}{0}}\|\g{P}^n(\cdot|\gss{x}{0}{})-\g{P}^n(\cdot|\gss{y}{0}{})\|_{\rm TV},
\]
denotes the standardized maximal distance, which satisfies:
\[
d(n)\leq \bar{d}(n)\leq 2d(n).
\]
It is usually easier to work with $\bar{d}(n)$ rather than $d(n)$. We also let 
\[
T_{\sbms{x}{0},\sbms{y}{0}}=\min\{n\in \mathbb{N}: \gss{X}{n}{}=\gss{Y}{n}{}|\gss{X}{0}{}=\gss{x}{0}{},\gss{Y}{0}{}=\gss{y}{0}{}\}
\]
be the stopping time until the two processes meet (also called \emph{coupling time}). 

With these definitions, standard coupling theory gives that:
\begin{equation}\label{eq:d_d_bar_coupling}
d(n)\leq \bar{d}(n)\leq \max_{\sbms{x}{0},\sbms{y}{0}}\tilde{P}_{\sbms{x}{0},\sbms{y}{0}}(\gss{X}{n}{}\neq \gss{Y}{n}{})=\max_{\sbms{x}{0},\sbms{y}{0}}\tilde{P}_{\sbms{x}{0},\sbms{y}{0}}(T_{\sbms{x}{0},\sbms{y}{0}}>n),
\end{equation}
where $\tilde{P}$ denotes the coupling measure and $\tilde{P}_{\sbms{x}{0},\sbms{y}{0}}(\cdot)=\tilde{P}(\cdot|\gss{X}{0}{}=\gss{x}{0}{},\gss{Y}{0}{}=\gss{x}{0}{})$.

\textbf{Strategy}: We will use the aforementioned coupling to bound $\bar{d}(n)$.

\subsection*{The BAR Model: Proof of the Mixing Time Bound}

Consider the individual scalar processes $\{[\gss{X}{k}{}]_i\}$ and $\{[\gss{Y}{k}{}]_i\}$ with the coupling described previously. Let $k=0$. Then:

\[
\tilde{P}_{\sbms{x}{0},\sbms{y}{0}}([\gss{X}{1}{}]_i\neq [\gss{Y}{1}{}]_i)=\sum_{j=1}^pa_{ij}([f_i(\gss{x}{0}{})]_j-[f_i(\gss{y}{0}{})]_j),
\]
given that $\sum_{j=1}^pa_{ij}[f_i(\gss{x}{0}{})]_j>\sum_{j=1}^pa_{ij}[f_i(\gss{y}{0}{})]_j$ or
\[
\tilde{P}_{\sbms{x}{0},\sbms{y}{0}}([\gss{X}{1}{}]_i\neq [\gss{Y}{1}{}]_i)=\sum_{j=1}^pa_{ij}([f_i(\gss{y}{0}{})]_j-[f_i(\gss{x}{0}{})]_j)
\]
in the opposite case. Therefore,
\begin{align}\label{eq:Proof1}
\tilde{P}_{\sbms{x}{0},\sbms{y}{0}}([\gss{X}{1}{}]_i\neq [\gss{Y}{1}{}]_i)&=\left|\sum_{j=1}^pa_{ij}([f_i(\gss{x}{0}{})]_j-[f_i(\gss{y}{0}{})]_j)\right|\leq \sum_{j=1}^pa_{ij}|[f_i(\gss{x}{0}{})]_j-[f_i(\gss{y}{0}{})]_j|\nonumber\\&= \sum_{j=1}^pa_{ij}|[\gss{x}{0}{}]_j-[\gss{y}{0}{}]_j|
\end{align}
and correspondingly
\begin{align}\label{eq:Proof2}
\tilde{P}_{\sbms{x}{0},\sbms{y}{0}}([\gss{X}{1}{}]_i=[\gss{Y}{1}{}]_i)&=1-\left|\sum_{j=1}^pa_{ij}([f_i(\gss{x}{0}{})]_j-[f_i(\gss{y}{0}{})]_j)\right|\geq 1-\sum_{j=1}^pa_{ij}|[f_i(\gss{x}{0}{})]_j-[f_i(\gss{y}{0}{})]_j|\nonumber\\&=1-\sum_{j=1}^pa_{ij}|[\gss{x}{0}{}]_j-[\gss{y}{0}{}]_j|.
\end{align}
In (\ref{eq:Proof1}) and (\ref{eq:Proof2}) we have used the observation that $|[f_i(\gss{x}{0}{})]_j-[f_i(\gss{y}{0}{})]_j|=|[\gss{x}{0}{}]_j-[\gss{y}{0}{}]_j|$ for any $f_i(\cdot)$. Extending this argument to an arbitrary time instant $k$ before the coupling time, we can see that 

\begin{align}\label{eq:Proof111}
\tilde{P}_{\sbms{x}{k-1},\sbms{y}{k-1}}([\gss{X}{k}{}]_i\neq [\gss{Y}{k}{}]_i)&=\left|\sum_{j=1}^pa_{ij}([f_i(\gss{x}{k-1}{})]_j-[f_i(\gss{y}{k-1}{})]_j)\right|\leq \sum_{j=1}^pa_{ij}|[f_i(\gss{x}{k-1}{})]_j-[f_i(\gss{y}{k-1}{})]_j|\nonumber\\&= \sum_{j=1}^pa_{ij}|[\gss{x}{k-1}{}]_j-[\gss{y}{k-1}{}]_j|
\end{align}
and correspondingly
\begin{align}\label{eq:Proof222}
\tilde{P}_{\sbms{x}{k-1},\sbms{y}{k-1}}([\gss{X}{k}{}]_i=[\gss{Y}{k}{}]_i)&=1-\left|\sum_{j=1}^pa_{ij}([f_i(\gss{x}{k-1}{})]_j-[f_i(\gss{y}{k-1}{})]_j)\right|\geq 1-\sum_{j=1}^pa_{ij}|[f_i(\gss{x}{k-1}{})]_j-[f_i(\gss{y}{k-1}{})]_j|\nonumber\\&=1-\sum_{j=1}^pa_{ij}|[\gss{x}{k-1}{}]_j-[\gss{y}{k-1}{}]_j|,
\end{align}
where $\tilde{P}_{\sbms{x}{k-1},\sbms{y}{k-1}}(\cdot)=\tilde{P}(\cdot|\gss{X}{k-1}{}=\gss{x}{k-1}{},\gss{Y}{k-1}{}=\gss{y}{k-1}{})$.

Considering now (\ref{eq:d_d_bar_coupling}) we have:
{\small
\begin{align}\label{eq:Proof5_12}
&\tilde{P}_{\sbms{x}{0},\sbms{y}{0}}(T_{\sbms{x}{0},\sbms{y}{0}}>n)\underbrace{=}_{(a)}\tilde{P}_{\sbms{x}{0},\sbms{y}{0}}(\gss{X}{n}{}\neq \gss{Y}{n}{})\underbrace{=}_{(b)}\sum_{\sbms{x}{n-1},\sbms{y}{n-1}}\tilde{P}_{\sbms{x}{0},\sbms{y}{0}}(\gss{X}{n}{}\neq \gss{Y}{n}{},\gss{X}{n-1}{}=\gss{x}{n-1}{},\gss{Y}{n-1}{}=\gss{y}{n-1}{})\underbrace{=}_{(c)}\nonumber\\&\sum_{\sbms{x}{n-1},\sbms{y}{n-1}}\tilde{P}_{\sbms{x}{n-1},\sbms{y}{n-1}}(\gss{X}{n}{}\neq \gss{Y}{n}{})\tilde{P}_{\sbms{x}{0},\sbms{y}{0}}(\gss{x}{n-1}{},\gss{y}{n-1}{})=\tilde{E}_{\sbms{x}{0},\sbms{y}{0}}\left[\tilde{P}_{\sbms{X}{n-1},\sbms{Y}{n-1}}(\gss{X}{n}{}\neq \gss{Y}{n}{})\right]=\nonumber\\
&\tilde{E}_{\sbms{x}{0},\sbms{y}{0}}\left[\tilde{P}_{\sbms{X}{n-1},\sbms{Y}{n-1}}\left([\gss{X}{n}{}]_i\neq [\gss{Y}{n}{}]_i\: \text{for at least one}\: i\in [p]\right)\right]=\tilde{E}_{\sbms{x}{0},\sbms{y}{0}}\left[1-\tilde{P}_{\sbms{X}{n-1},\sbms{Y}{n-1}}\left([\gss{X}{n}{}]_i=[\gss{Y}{n}{}]_i, \forall i\in [p]\right)\right]=\nonumber\\&\tilde{E}_{\sbms{x}{0},\sbms{y}{0}}\left[1-\prod_{i=1}^p\left(1-\left|\sum_{j=1}^pa_{ij}([f_i(\gss{X}{n-1}{})]_j-[f_i(\gss{Y}{n-1}{})]_j)\right|\right)\right]\underbrace{\leq}_{(d)} \tilde{E}_{\sbms{x}{0},\sbms{y}{0}}\left[1-\prod_{i=1}^p\left(1-\sum_{j=1}^pa_{ij}\left|[\gss{X}{n-1}{}]_j-[\gss{Y}{n-1}{}]_j\right|\right)\right]\nonumber\\&\underbrace{\leq}_{(e)} 
1-\prod_{i=1}^p\left(1-\sum_{j=1}^pa_{ij}\tilde{E}_{\sbms{x}{0},\sbms{y}{0}}\left[|[\gss{X}{n-1}{}]_j-[\gss{Y}{n-1}{}]_j|\right]\right)^{\frac{1}{1-\max_{1\leq i\leq p}\sum_{j=1}^pa_{ij}}}\leq \nonumber\\ &1-\left(1-\max_{1\leq i\leq p}\sum_{j=1}^pa_{ij}\tilde{E}_{\sbms{x}{0},\sbms{y}{0}}\left[|[\gss{X}{n-1}{}]_j-[\gss{Y}{n-1}{}]_j|\right]\right)^{\frac{p}{1-\max_{1\leq i\leq p}\sum_{j=1}^pa_{ij}}}\underbrace{\leq}_{(f)}\nonumber\\ &\frac{p}{1-\max_{1\leq i\leq p}\sum_{j=1}^pa_{ij}} \max_{1\leq i\leq p}\sum_{j=1}^pa_{ij}\tilde{E}_{\sbms{x}{0},\sbms{y}{0}}\left[|[\gss{X}{n-1}{}]_j-[\gss{Y}{n-1}{}]_j|\right].   
\end{align}  
}
\begin{itemize}
\item In (a), we explain the equality in (\ref{eq:d_d_bar_coupling}) by noting the sequence of implications $\{\gss{X}{n}{}\neq \gss{Y}{n}{}\}\Rightarrow \{\gss{X}{n-1}{}\neq \gss{Y}{n-1}{}\}\Rightarrow\cdots \Rightarrow\{\gss{X}{0}{}\neq \gss{Y}{0}{}\}$, which is a consequence of the coupling.
\item In (b), we marginalize over $\gss{X}{n-1}{}$ and $\gss{Y}{n-1}{}$. 
\item In (c), we employ the fact that $\tilde{P}(\gss{X}{n}{}\neq \gss{Y}{n}{}|\gss{X}{n-1}{}=\gss{x}{n-1}{},\gss{Y}{n-1}{}=\gss{y}{n-1}{},\gss{X}{0}{}=\gss{x}{0}{},\gss{Y}{0}{}=\gss{y}{0}{})=\tilde{P}(\gss{X}{n}{}\neq \gss{Y}{n}{}|\gss{X}{n-1}{}=\gss{x}{n-1}{},\gss{Y}{n-1}{}=\gss{y}{n-1}{})=\tilde{P}_{\sbms{x}{n-1},\sbms{y}{n-1}}(\gss{X}{n}{}\neq \gss{Y}{n}{})$. 
\item In (d), (\ref{eq:Proof222}) is plugged in.
\item In (e), we use (\ref{eq:coreLemma2}) in Lemma \ref{lem:CoreLemma} by setting $Z_i=1-\sum_{j=1}^pa_{ij}|[\gss{X}{n-1}{}]_j-[\gss{Y}{n-1}{}]_j|, \forall i\in [p]$ and by noting that $0<1-\max_{1\leq i\leq p}\sum_{j=1}^pa_{ij}\leq Z_i\leq 1, \forall i\in [p]$ almost surely.
\item In (f), Bernoulli's inequality for $r\geq 1$ is employed:
\[
\left(1-\max_{1\leq i\leq p}\sum_{j=1}^pa_{ij}\tilde{E}_{\sbms{x}{0},\sbms{y}{0}}\left[|[\gss{X}{n-1}{}]_j-[\gss{Y}{n-1}{}]_j|\right]\right)^{r}\geq 1-r\max_{1\leq i\leq p}\sum_{j=1}^pa_{ij}\tilde{E}_{\sbms{x}{0},\sbms{y}{0}}\left[|[\gss{X}{n-1}{}]_j-[\gss{Y}{n-1}{}]_j|\right],
\] 
with $r=\frac{p}{1-\max_{1\leq i\leq p}\sum_{j=1}^pa_{ij}}$.
\end{itemize}

We now note that the following recursive relation holds:
\begin{equation}\label{eq:recursion}
\tilde{E}_{\sbms{x}{0},\sbms{y}{0}}[|[\gss{X}{k-1}{}]_j-[\gss{Y}{k-1}{}]_j|]\leq \sum_{l=1}^pa_{jl}\tilde{E}_{\sbms{x}{0},\sbms{y}{0}}\left[|[\gss{X}{k-2}{}]_l-[\gss{Y}{k-2}{}]_l|\right].
\end{equation}
To prove this, we have:

\begin{align*}
&\tilde{E}_{\sbms{x}{0},\sbms{y}{0}}[|[\gss{X}{k-1}{}]_j-[\gss{Y}{k-1}{}]_j|]\underbrace{=}_{(g)}\sum_{\sbms{x}{k-2},\sbms{y}{k-2}}\tilde{E}_{\sbms{x}{k-2},\sbms{y}{k-2}}[|[\gss{X}{k-1}{}]_j-[\gss{Y}{k-1}{}]_j|] \tilde{P}_{\sbms{x}{0},\sbms{y}{0}}(\gss{x}{k-2}{},\gss{y}{k-2}{})\underbrace{=}_{(h)}\\&
\sum_{\sbms{x}{k-2},\sbms{y}{k-2}}\tilde{P}_{\sbms{x}{k-2},\sbms{y}{k-2}}([\gss{X}{k-1}{}]_j\neq [\gss{Y}{k-1}{}]_j)\tilde{P}_{\sbms{x}{0},\sbms{y}{0}}(\gss{x}{k-2}{},\gss{y}{k-2}{})
\\& \underbrace{\leq}_{(i)}\sum_{l=1}^pa_{jl}\tilde{E}_{\sbms{x}{0},\sbms{y}{0}}\left[|[f_j(\gss{X}{k-2}{})]_l-[f_j(\gss{Y}{k-2}{})]_l|\right]=\sum_{l=1}^pa_{jl}\tilde{E}_{\sbms{x}{0},\sbms{y}{0}}\left[|[\gss{X}{k-2}{}]_l-[\gss{Y}{k-2}{}]_l|\right].
\end{align*}

Note that:
\begin{itemize}
\item In $(g)$, we condition with respect to $\gss{X}{k-2}{},\gss{Y}{k-2}{}$. 
\item In $(h)$, we use the fact that $|[\gss{X}{k-1}{}]_j-[\gss{Y}{k-1}{}]_j|$ is a Bernoulli random variable. 
\item In $(i)$,  we employ (\ref{eq:Proof111}).
\end{itemize}
$\\$


We now upper bound $\max_{1\leq i\leq p}\sum_{j=1}^pa_{ij}\tilde{E}_{\sbms{x}{0},\sbms{y}{0}}\left[|[\gss{X}{n-1}{}]_j-[\gss{Y}{n-1}{}]_j|\right]$ in (\ref{eq:Proof5_12}) using the recursion (\ref{eq:recursion}). To this end, we have:

\begin{align*}
&\max_{1\leq i\leq p}\sum_{j=1}^pa_{ij}\tilde{E}_{\sbms{x}{0},\sbms{y}{0}}\left[|[\gss{X}{n-1}{}]_j-[\gss{Y}{n-1}{}]_j|\right]\leq\\& 
\max_{1\leq i\leq p}\sum_{j=1}^pa_{ij}\sum_{i_1=1}^pa_{ji_1}\sum_{i_2=1}^{p}a_{i_1i_2}\cdots \sum_{i_{n-1}=1}^pa_{i_{n-2}i_{n-1}}|[\gss{x}{0}{}]_{i_{n-1}}-[\gss{y}{0}{}]_{i_{n-1}}|\leq\\& \max_{1\leq i\leq p}\sum_{j=1}^pa_{ij}\sum_{i_1=1}^pa_{ji_1}\sum_{i_2=1}^{p}a_{i_1i_2}\cdots \sum_{i_{n-1}=1}^pa_{i_{n-2}i_{n-1}}
\end{align*}

Therefore, 
\begin{align}\label{eq:crucious3}
&\max_{1\leq i\leq p}\sum_{j=1}^pa_{ij}\tilde{E}_{\sbms{x}{0},\sbms{y}{0}}\left[|[\gss{X}{n-1}{}]_j-[\gss{Y}{n-1}{}]_j|\right]\leq \left(\max_{1\leq i\leq p}\sum_{j=1}^pa_{ij}\right)^{n}.
\end{align}
By (\ref{eq:d_d_bar_coupling}), (\ref{eq:Proof5_12}) and (\ref{eq:crucious3}), we obtain
\[
\bar{d}(n)=\max_{\sbms{x}{0},\sbms{y}{0}}\|\g{P}^n(\cdot|\gss{x}{0}{})-\g{P}^n(\cdot|\gss{y}{0}{})\|_{\rm TV}\leq \frac{p}{1-\max_{1\leq i\leq p}\sum_{j=1}^pa_{ij}} \left(\max_{1\leq i\leq p}\sum_{j=1}^pa_{ij}\right)^{n}.
\]
Hence, the number of required steps to mixing can be upper bounded by requiring:
\[
\frac{p}{1-\max_{1\leq i\leq p}\sum_{j=1}^pa_{ij}} \left(\max_{1\leq i\leq p}\sum_{j=1}^pa_{ij}\right)^{n} \leq \theta.
\]

This immediately shows that 
\[
t_{mix}(\theta)\leq \left\lceil\frac{\log\left(\frac{\theta \left(1-\max_{1\leq i\leq p}\sum_{j=1}^pa_{ij}\right)}{p}\right)}{\log\left(\max_{1\leq i\leq p}\sum_{j=1}^pa_{ij}\right)}\right\rceil.
\]

Furthermore, we can write:
\begin{align*}
\max_{1\leq i\leq p}\sum_{j=1}^{p}a_{ij}=1-\left(1-\max_{1\leq i\leq p}\sum_{j=1}^{p}a_{ij}\right),
\end{align*}
thus
\begin{align*}
\log\left(\max_{1\leq i\leq p}\sum_{j=1}^{p}a_{ij}\right)&=\log\left(1-\left(1-\max_{1\leq i\leq p}\sum_{j=1}^{p}a_{ij}\right)\right)\leq \frac{-2\left(1-\max_{1\leq i\leq p}\sum_{j=1}^{p}a_{ij}\right)}{2-\left(1-\max_{1\leq i\leq p}\sum_{j=1}^{p}a_{ij}\right)}\\&= \frac{2\left(\max_{1\leq i\leq p}\sum_{j=1}^{p}a_{ij} -1\right)}{1+\max_{1\leq i\leq p}\sum_{j=1}^{p}a_{ij}}.
\end{align*}

Here, we have used the inequality $\log(1+x)\leq 2x/(2+x)$ for $x\in (-1,0]$. We therefore obtain:
\begin{align*}
t_{mix}(\theta)&\leq \left\lceil\frac{1}{2}\frac{1+\max_{1\leq i\leq p}\sum_{j=1}^pa_{ij}}{1-\max_{1\leq i\leq p}\sum_{j=1}^pa_{ij}}\left(\log p -\log \left(\theta \left(1-\max_{1\leq i\leq p}\sum_{j=1}^pa_{ij}\right)\right)\right)\right\rceil\\
&\leq \left\lceil\frac{1}{1-\max_{1\leq i\leq p}\sum_{j=1}^pa_{ij}}\left(\log p -\log \left(\theta \left(1-\max_{1\leq i\leq p}\sum_{j=1}^pa_{ij}\right)\right)\right)\right\rceil. 
\end{align*}

\section{Proof of Theorem \ref{thm:ExtendedBARObs}}
\label{sec:ProofMainThm2}

In this section, we proceed in a modular fashion, proving first some necessary intermediate results. We then combine these results to prove    
Theorem \ref{thm:ExtendedBARObs}.

\subsection{Transition Matrix of the BAR Model}
\label{subsec:BAR-MC}

Suppose that $\gss{X}{0}{}\sim P_{\sbms{X}{0}}$, where $P_{\sbms{X}{0}}$ can be an arbitrary measure or $\pi$. Let $\tilde{\g{X}}=\{\gss{X}{n}{}\}_{ n\geq 0}$ be the sequence of samples generated by the BAR model (\ref{eq:BAR}) when initialized at $\gss{X}{0}{}$. Then, $\tilde{\g{X}}$ is a homogeneous, first order Markov chain on the finite state space $\mathcal{X}=\{0,1\}^p$ with transition matrix given by:
\begin{align}
P(\gss{x}{k+1}{}|\gss{x}{k}{})&=E_{\sbms{W}{k+1}}\left[\g{P}(\gss{x}{k+1}{}|\gss{x}{k}{},\gss{W}{k+1}{})\right]\nonumber\\&= E_{\sbms{W}{k+1}}\left[\prod_{i=1}^{p}[\g{A}\g{f}(\gss{x}{k}{})+\g{B}\gss{W}{k+1}{}]_i^{[\sbms{x}{k+1}]_i}(1-[\g{A}\g{f}(\gss{x}{k}{})+\g{B}\gss{W}{k+1}{}]_i)^{1-[\sbms{x}{k+1}]_i}\right]\nonumber\\&=\sum_{\sbms{w}{k+1}\in \mathcal{X}}\prod_{j=1}^p\rho_w^{[\sbms{w}{k+1}]_j}(1-\rho_w)^{1-[\sbms{w}{k+1}]_j}\cdot\nonumber\\ &\ \  \ \ \ \prod_{i=1}^{p}\left[\g{A}\g{f}(\gss{x}{k}{})+\g{B}\gss{w}{k+1}{}\right]_i^{[\sbms{x}{k+1}]_i}\left(1-\left[\g{A}\g{f}(\gss{x}{k}{})+\g{B}\gss{w}{k+1}{}\right]_i\right)^{1-[\sbms{x}{k+1}]_i}\nonumber\\&=\prod_{i=1}^{p}[\g{A}\g{f}(\gss{x}{k}{})+\rho_w\g{B}\gss{1}{}{}]_i^{[\sbms{x}{k+1}]_i}(1-[\g{A}\g{f}(\gss{x}{k}{})+\rho_w\g{B}\gss{1}{}{}]_i)^{1-[\sbms{x}{k+1}]_i},
\end{align}
where we have used the independence of the involved random variables as functions of $\gss{W}{k+1}{}$ for a given $\gss{x}{k}{}$ and the fact that 
for each $i$ only one of the two terms $\cdot^{[\sbms{x}{k+1}]_i}$ or $\cdot^{1-[\sbms{x}{k+1}]_i}$ appears in the product. Here, $\g{1}$ represents the 
$p\times 1$ all-ones vector.

It is easy to see that the BAR chain is irreducible and aperiodic, and further since it is finite state, the chain is geometrically ergodic \cite{r95}.  

\subsection{Uniform Bounds for Required Stationary Probabilities}
\label{subsec:BARSpecif}

In this section, we derive some uniform bounds on specific stationary probabilities that emerge in the proposed structure estimator. 

\subsubsection{Bounding  Marginal Stationary Probabilities for Supergraph Selection}
 We first observe that since the BAR chain is irreducible and aperiodic, all states are positive recurrent. Thus, $0< \pi(\g{x})< 1, \forall \g{x}\in \mathcal{X}$. If the transition matrix 
is doubly stochastic, the stationary distribution is uniform. Thus, $\pi(\g{x})=1/2^p, \forall \g{x}\in \mathcal{X}$. This special but important case reveals that $\pi(\g{x})$ can be positive for all $\g{x}$, but $\pi(\g{x})\downarrow 0$ as $p\rightarrow \infty$. In other words, as $p$ increases, the minimum fraction of time that an irreducible chain spends at any given state decreases. In this important special case, 
\begin{align*}
P(X_l=x_l)=\sum_{\sbm{x}: [\sbm{x}]_l=x_l}\frac{1}{2^p}=\frac{1}{2}
\end{align*}
for $x_l\in \{0,1\}$, where $X_l=[\gss{X}{k}{}]_l$ and the underlying measure of $\gss{X}{k}{}$ is $\pi$. 

To continue with the derivation of uniform lower bounds on the desired stationary probabilities, consider for the moment any sample path converging to stationarity.  Let $\gss{p}{k+1}{}$ be the vector
\[
\gss{p}{k+1}{}=\left[P([\gss{X}{k+1}{}]_1=1),P([\gss{X}{k+1}{}]_2=1),\ldots,P([\gss{X}{k+1}{}]_p=1)\right]^T.
\]
Using (\ref{eq:BAR}) we obtain:
\begin{align}\label{eq:ProbEvolution}
\gss{p}{k+1}{}&=E\left[\gss{X}{k+1}{}\right]=E\left[E\left[\mathrm{Ber}\left(\g{A}\g{f}(\gss{X}{k}{})+\g{B}\gss{W}{k+1}{}\right)|\gss{X}{k}{},\gss{W}{k+1}{}\right]\right]\nonumber\\&=E\left[\g{A}\g{f}(\gss{X}{k}{})+\g{B}\gss{W}{k+1}{}\right]=\g{A}\g{f}(\gss{p}{k}{})+\rho_w\g{B1}.
\end{align}
Let $\g{\bar{f}}$ be the vector with $1$'s at the locations where $\g{f}(\cdot)$ inverts polarities and zeros elsewhere. Also, let $\g{\tilde{A}}$ be $\g{A}$ with negated the entries where a polarity inversion happens. At stationarity,
\begin{align*}
\g{p}=\g{A}\g{f}(\g{p})+\rho_w\g{B1}=\g{\tilde{A}}(\g{1}\otimes\g{p})+\g{A}\g{\bar{f}}+\rho_w\g{B1},
\end{align*} 
where $\otimes$ denotes the Kronecker product. Letting $\g{\hat{A}}$ denote the $p\times p$ matrix whose $i$th row $\gss{\hat{a}}{i}{T}$ corresponds to 
\[
\gss{\hat{a}}{i}{T}=\gss{\tilde{A}}{i,(i-1)p+1:ip}{},
\]
the previous equation becomes $\g{p}=\g{\hat{A}}\g{p}+\g{A}\g{\bar{f}}+\rho_w\g{B1}$, leading to
\begin{equation}\label{eq:marg1}
\g{p}=\left(\g{I}-\g{\hat{A}}\right)^{-1}\left(\g{A}\g{\bar{f}}+\rho_w\g{B1}\right).
\end{equation}
The invertibility of $\g{I}-\g{\hat{A}}$ is ensured by the following lemma:
\begin{lem}\label{lem:Marg1}
Let $\rho(\cdot)$ denote the spectral radius of a matrix. Then,
\[
\rho(\g{\hat{A}})\leq \rho(\g{\bar{A}})<1,
\]
where $\g{\bar{A}}$ is given by (\ref{eq:Abar}).
\end{lem}  
\begin{proof}
Appendix \ref{app:1}.
\end{proof}

 Using (\ref{eq:marg1}), we can bound $P(X_l=x_l)$ as follows:
\begin{align}\label{eq:beta_Init}
P(X_l=x_l)\geq \min_{1\leq i\leq p}\left[\left(\g{I}-\g{\hat{A}}\right)^{-1}\left(\g{A}\g{\bar{f}}+\rho_w\g{B1}\right)\right]_i\wedge 1-\left\|\left(\g{I}-\g{\hat{A}}\right)^{-1}\left(\g{A}\g{\bar{f}}+\rho_w\g{B1}\right)\right\|_{\infty}.
\end{align} 
Nevertheless, it is not clear if this lower bound is independent of $p$. To ensure that $P(X_l=x_l)$ is lower bounded by a quantity independent of $p$,
we note that $\g{p}=\g{A}\g{f}(\g{p})+\rho_w\g{B1}$ implies that $[\g{A1}+\rho_w\g{B1}]_l\geq P(X_l=1)\geq \rho_w[\g{B1}]_l$, which leads to $P(X_l=x_l)\geq \beta$ with 
\begin{align}\label{eq:beta1}
\beta=\rho_w\left(1-\max_{1\leq i\leq p}\sum_{j=1}^pa_{ij}\right)\wedge 1-\|\g{A1}+\rho_w\g{B1}\|_{\infty}\geq \rho_wb_{min}\wedge 1-\|\g{A1}+\rho_w\g{B1}\|_{\infty}.
\end{align} 
This lower bound is clearly independent of $p$. For any given $p$, using either (\ref{eq:beta_Init}) or (\ref{eq:beta1}) in the subsequent sample complexities is valid, with a better (lower) sample complexity when the tighter (maximum) bound between (\ref{eq:beta_Init}) and (\ref{eq:beta1}) is employed.
A special case of interest is summarized by the following Lemma:
\begin{lem}\label{lem:Marg2}
Let $f_1(\gss{X}{k}{})=f_2(\gss{X}{k}{})=\cdots =f_p(\gss{X}{k}{})=\gss{X}{k}{}$, i.e., consider (\ref{eq:BAR2}). Then, $\g{p}=\rho_w\g{1}$.
\end{lem}
\begin{proof}
Appendix \ref{app:1}.
\end{proof}
Clearly, in this case 
\begin{equation}\label{eq:beta2}
\beta=\rho_w\wedge 1-\rho_w.
\end{equation}

\subsubsection{Bounding Pairwise Marginal Probabilities for Supergraph Selection}

Define now the process $\tilde{\g{Z}}=\{\gss{Z}{n}{}=(\gss{X}{n}{},\gss{X}{n-1}{})\}_{n\geq 1}$. Then, 
\begin{align*}
P(\gss{Z}{n}{}=(\gss{z}{n}{},\gss{z}{n-1}{})|\gss{Z}{1}{}=(\gss{x}{1}{},\gss{x}{0}{}),\ldots,\gss{Z}{n-1}{}=(\gss{x}{n-1}{},\gss{x}{n-2}{}))=P(\gss{X}{n}{}=\gss{z}{n}{}|\gss{X}{n-1}{}=\gss{x}{n-1}{}) 
\end{align*}
if $\gss{z}{n-1}{}=\gss{x}{n-1}{}$ and $0$ if $\gss{z}{n-1}{}\neq \gss{x}{n-1}{}$.
Thus, $\tilde{\g{Z}}$ is a Markovian process with properties dictated by  $\tilde{\g{X}}$. A different way to see this is to note that
\begin{align*}
\gss{Z}{n+1}{}=\left(
  \begin{array}{c}
    \gss{X}{n+1}{} \\
    \gss{X}{n}{} \\
  \end{array}
\right)=\mathrm{Ber}\left((\gss{I}{2}{}\otimes\g{A})\left(
  \begin{array}{c}
    \g{f}(\gss{X}{n}{}) \\
    \g{f}(\gss{X}{n-1}{}) \\
  \end{array}
\right)+(\gss{I}{2}{}\otimes\g{B})\left(
  \begin{array}{c}
    \gss{W}{n+1}{} \\
    \gss{W}{n}{} \\
  \end{array}
\right)\right),
\end{align*}
where $\gss{I}{2}{}$ is the $2\times 2$ identity matrix or 
\begin{align*}
\gss{Z}{n+1}{}=(\gss{X}{n+1}{},\gss{X}{n}{})=\mathrm{Ber}\left(\g{A}(\g{f}(\gss{X}{n}{}),\g{f}(\gss{X}{n-1}{}))+\g{B}(\gss{W}{n+1}{},\gss{W}{n}{})\right)
\end{align*}
with an appropriate interpretation of how $\mathrm{Ber}(\cdot)$ is applied.
These expressions show that $\tilde{\g{Z}}$ is Markovian with transitions parameterized by the same matrices $\g{A},\g{B}$ as $\tilde{\g{X}}$.
Moreover, $\tilde{\g{Z}}$ is geometrically ergodic and has stationary distribution denoted by $\pi'$. 

We are mainly interested in lower bounding $P([\gss{X}{k+1}{}]_m=1,[\gss{X}{k}{}]_l=x_l)=P(X_m^{+1}=1,X_l=x_l)$ for the stationary measure. By conditioning 
on $X_1,\ldots, X_{l-1},X_{l+1},\ldots, X_p$ and on $W_m^{+1}$, we have: 
\begin{align*}
P(X_{m}^{+1}=1|X_l=x_l)&=\sum_{x_1,\ldots, x_{l-1},x_{l+1},\ldots, x_p, w_m^{+1}}P(X_{m}^{+1}=1,x_1,\ldots, x_{l-1},x_{l+1},\ldots, x_p, w_m^{+1}|X_l=x_l)\\&=\sum_{w_m^{+1}}P(w_m^{+1})\sum_{x_1,\ldots, x_{l-1},x_{l+1},\ldots, x_p}P(x_1,\ldots, x_{l-1},x_{l+1},\ldots, x_p|X_l=x_l)\cdot\\&\ \ \ \ \ \
P(X_{m}^{+1}=1|x_1,\ldots, x_{l-1},X_l=x_l,x_{l+1},\ldots, x_p, w_m^{+1})\\&\geq b_m\rho_w\sum_{x_1,\ldots, x_{l-1},x_{l+1},\ldots, x_p}P(x_1,\ldots, x_{l-1},x_{l+1},\ldots, x_p|X_l=x_l)=b_m\rho_w.
\end{align*}
Here, the independence of $w_m^{+1}$ from all past state vectors has been used.
Thus,
\begin{align*}
P(X_{m}^{+1}=1|X_l=x_l)&\geq \min_{1\leq m\leq p}b_{m}\rho_w=\left(1-\max_{1\leq i\leq p}\sum_{j=1}^pa_{ij}\right)\rho_w\geq b_{min}\rho_w.
\end{align*}
Using now the fact that $P(X_m^{+1}=1,X_l=x_l)=P(X_l=x_l)P(X_{m}^{+1}=1|X_l=x_l)$, we obtain:
\begin{align}
P(X_m^{+1}=1,X_l=x_l)\geq \beta \left(1-\max_{1\leq i\leq p}\sum_{j=1}^pa_{ij}\right)\rho_w,
\end{align}
where $\beta$ is given by the tightest bound between (\ref{eq:beta_Init}) and (\ref{eq:beta1})\footnote{In practice, a larger $\beta$ corresponds to (\ref{eq:beta_Init}).}.

\subsubsection{Combining the Bounds}

Combining the bounds for marginal and pairwise marginal stationary probabilities, we obtain:
\begin{align}\label{eq:measureSpec1}
\min_{x_l\in \{0,1\}}\left\{P(X_l=x_l),P(X_m^{+1}=1,X_l=x_l)\right\}\geq \tilde{\beta},
\end{align}
with 
\begin{equation}\label{eq:measureSpec2}
\tilde{\beta}=\beta \left(1-\max_{1\leq i\leq p}\sum_{j=1}^pa_{ij}\right)\rho_w,
\end{equation}
where we have used the observation that the term inside the parentheses and $\rho_w$ are less than $1$. 

(\ref{eq:measureSpec1}) is critical for the subsequent derivations. Nevertheless, it is straightforward  to see by upper bounding $P(X_m^{+1}=1|X_l=x_l)$ that  (\ref{eq:measureSpec1}) can be extended to
\begin{align}\label{eq:measureSpec3}
\min_{x_l,x_m\in \{0,1\}}\left\{P(X_l=x_l),P(X_m^{+1}=x_m,X_l=x_l)\right\}\geq \check{\beta},
\end{align}
with 
\begin{equation}\label{eq:measureSpec4}
\check{\beta}=\beta \left(\left(1-\max_{1\leq i\leq p}\sum_{j=1}^pa_{ij}\right)\rho_w\wedge 1-\max_{1\leq i\leq p}\left(b_i\rho_w+\sum_{j=1}^pa_{ij}\right)\right).
\end{equation}

\subsubsection{Stationary Probability Bounds for Supergraph Trimming}

As before, we require some lower bound on $P(\gss{X}{\tilde{\mathcal{S}}(i)}{}=\gss{x}{\tilde{\mathcal{S}}(i)}{})$ and $P\left(X_i^{+1}=1,\gss{X}{\tilde{\mathcal{S}}(i)}{}=\gss{x}{\tilde{\mathcal{S}}(i)}{}\right)$ for any $\tilde{\mathcal{S}}(i)\subset [p]$ such that $|\tilde{\mathcal{S}}(i)|=d$ and any $i\in [p]$, when the underlying measure is the stationary. Setting $\tilde{\mathcal{S}}(i)=\{l_1,l_2,\ldots, l_d\}$ with $l_m\in [p]$ for all $m\in [d]$ and $\gss{X}{\tilde{\mathcal{S}}(i)}{}=[X_{l_1},X_{l_2},\ldots, X_{l_d}]^T$ ($\gss{X}{\tilde{\mathcal{S}}(i)}{}$ corresponds to time instant $k$)  we have: 
\begin{align*}
P(\gss{X}{\tilde{\mathcal{S}}(i)}{}&=\gss{x}{\tilde{\mathcal{S}}(i)}{})=\sum_{\sbms{x}{k-1}}P(\gss{X}{\tilde{\mathcal{S}}(i)}{}=\gss{x}{\tilde{\mathcal{S}}(i)}{}|\gss{x}{k-1}{})\pi(\gss{x}{k-1}{})\\&=\sum_{\sbms{x}{k-1}}\prod_{m=1}^d[\g{A}\g{f}(\gss{x}{k-1}{})+\rho_w\g{B1}]_{l_m}^{[\sbms{x}{\tilde{\mathcal{S}}(i)}]_{m}}(1-[\g{A}\g{f}(\gss{x}{k-1}{})+\rho_w\g{B1}]_{l_m})^{1-[\sbms{x}{\tilde{\mathcal{S}}(i)}]_{m}}\pi(\gss{x}{k-1}{})\\ &\geq \sum_{\sbms{x}{k-1}}\prod_{m=1}^d[\rho_w\g{B1}]_{l_m}^{[\sbms{x}{\tilde{\mathcal{S}}(i)}]_{m}}(1-[\g{A1}+\rho_w\g{B1}]_{l_m})^{1-[\sbms{x}{\tilde{\mathcal{S}}(i)}]_{m}}\pi(\gss{x}{k-1}{})\\ & = \prod_{m=1}^d[\rho_w\g{B1}]_{l_m}^{[\sbms{x}{\tilde{\mathcal{S}}(i)}]_{m}}(1-[\g{A1}+\rho_w\g{B1}]_{l_m})^{1-[\sbms{x}{\tilde{\mathcal{S}}(i)}]_{m}}\\ &\geq 
 \left(\left(1-\max_{1\leq i\leq p}\sum_{j=1}^pa_{ij}\right)\rho_w\wedge 1-\max_{1\leq i\leq p}\left(b_i\rho_w+\sum_{j=1}^pa_{ij}\right)\right)^d=\left(\frac{1}{\bar{c}}\right)^d,
\end{align*}
for $\bar{c}=1/\left(\left(1-\max_{1\leq i\leq p}\sum_{j=1}^pa_{ij}\right)\rho_w\wedge 1-\max_{1\leq i\leq p}\left(b_i\rho_w+\sum_{j=1}^pa_{ij}\right)\right)>1$.

Moreover, 
\begin{align}
P\left(X_i^{+1}=1|\gss{X}{\tilde{\mathcal{S}}(i)}{}=\gss{x}{\tilde{\mathcal{S}}(i)}{}\right)&=\sum_{l\in \mathcal{S}(i)\cap \tilde{\mathcal{S}}(i)}a_{il}f_i(x_l)+\sum_{l\in \mathcal{S}(i)\cap \tilde{\mathcal{S}}^c(i)}a_{il}E\left[f_i(X_l)|\gss{X}{\tilde{\mathcal{S}}(i)}{}=\gss{x}{\tilde{\mathcal{S}}(i)}{}\right]+b_i\rho_w\nonumber\\&\geq b_i\rho_w.
\end{align}

Combining the bounds, we obtain:
\begin{equation}\label{eq:stmeasBound2}
\min_{i\in [p],\tilde{\mathcal{S}}(i)\subset [p],|\tilde{\mathcal{S}}(i)|=d, \sbms{x}{\tilde{\mathcal{S}}(i)}\in \{0,1\}^d}\left\{P(\gss{X}{\tilde{\mathcal{S}}(i)}{}=\gss{x}{\tilde{\mathcal{S}}(i)}{}),P\left(X_i^{+1}=1,\gss{X}{\tilde{\mathcal{S}}(i)}{}=\gss{x}{\tilde{\mathcal{S}}(i)}{}\right)\right\}\geq \bar{\beta},
\end{equation}
where 
\begin{align*}
\bar{\beta}&= \left(\left(1-\max_{1\leq i\leq p}\sum_{j=1}^pa_{ij}\right)\rho_w\wedge 1-\max_{1\leq i\leq p}\left(b_i\rho_w+\sum_{j=1}^pa_{ij}\right)\right)^d\left(1-\max_{1\leq i\leq p}\sum_{j=1}^pa_{ij}\right)\rho_w\\&= \left(\frac{1}{\bar{c}}\right)^d\left(1-\max_{1\leq i\leq p}\sum_{j=1}^pa_{ij}\right)\rho_w.
\end{align*}

\textbf{Remark}: We note here that the index $i$ in the event $\{\gss{X}{\tilde{\mathcal{S}}(i)}{}=\gss{x}{\tilde{\mathcal{S}}(i)}{}\}$ is irrelevant, since the derived bounds account for any $d-$sized subset of $[p]$. We have chosen to retain $i$ in (\ref{eq:stmeasBound2}) for reasons of continuation of the previous discussion. A more clear way of writing (\ref{eq:stmeasBound2}) is:
\begin{align*}
\min_{i\in [p],\bar{\mathcal{S}},\tilde{\mathcal{S}}\subset [p],|\bar{\mathcal{S}}|=d, |\tilde{\mathcal{S}}|=d, \sbms{x}{\bar{\mathcal{S}}},\sbms{x}{\tilde{\mathcal{S}}}\in \{0,1\}^d}\left\{P(\gss{X}{\bar{\mathcal{S}}}{}=\gss{x}{\bar{\mathcal{S}}}{}),P\left(X_i^{+1}=1,\gss{X}{\tilde{\mathcal{S}}}{}=\gss{x}{\tilde{\mathcal{S}}}{}\right)\right\}\geq \bar{\beta}.
\end{align*}

\subsection{Key Results and Technical Conditions for the Supergraph Selection Stage}
\label{subsec:proofsForTrimming}

While the supergraph selection stage of Alg. \ref{alg:Obs1}  performs well on  simulated and pseudo-real datasets, in theory, we have to impose certain 
conditions to guarantee the correctness of this stage. These results are presented in the sequel.

\subsubsection{Empirical Estimation of the Conditional Influence}

Guaranteeing the performance of the supergraph selection stage in Alg. \ref{alg:Obs1}, requires sufficiently accurate $\hat{\nu}_{i|j}$'s for all $i,j\in \mathcal{V}$. To this end, we define the event 
\[
\mathcal{A}(\varepsilon)=\left\{\left|\hat{\nu}_{i|j}-\nu_{i|j}\right|\leq \varepsilon\: \text{for}\: \text{all}\: i,j\in \mathcal{V}\right\}.
\]
The required sample complexity such that $\mathcal{A}(\varepsilon)$ holds with high probability is summarized by the following Lemma:
\begin{lem}\label{lem:empEst1}
Let $0<\gamma<1$ and $\theta\leq 1/8$. Assume that $t_{mix}(\theta)$ is the $\theta-$mixing time of $\tilde{\g{X}}=\{\gss{X}{n}{}\}$ and  $\tilde{\g{Z}}=\{\gss{Z}{n}{}\}$. If 
 \begin{align}\label{eq:nRequired}
  n\geq 1+\frac{1152\log\left(\frac{4p^2}{\gamma}C\right)t_{mix}(\theta)}{\varepsilon^2\tilde{\beta}^3},
\end{align}
where $C$ is some constant,
then $P(\mathcal{A}(\varepsilon))\geq 1-\gamma$.
\end{lem}
\begin{proof}
We need to bound the quantity $|\hat{\nu}_{i|j}-\nu_{i|j}|$ uniformly over $\mathcal{V}\times \mathcal{V}$.   We denote by $t_{mix}(\theta)$ the $\theta-$mixing time of the BAR chain defined as follows:
\begin{align*}
t_{mix}(\theta)=\min\left\{k: \max_{\sbm{\nu}}\|\g{\nu}\gss{P}{}{k}-\g{\pi}\|_{\rm TV}\leq \theta\right\},
\end{align*}
where $\g{P}$ is the transition matrix and $\g{\nu}$ is an arbitrary initial measure (in vector form). Consider the BAR chain $\tilde{\g{X}}=\{\gss{X}{n}{}\}_{n\geq 0}$ with $\gss{X}{0}{}\sim P_{\sbms{X}{0}}$ and
let $g_k:\{0,1\}^p\rightarrow [0,1]$ be a function defined at the $k-$th time step such that $E_{\pi}[g_{k}(\g{X})]=\kappa$ for all $k$. Then by Theorem 3 in \cite{cllm12}, there exists a constant $c$ independent of $\kappa,\tilde{\delta}$ and $\theta$ such that for $\theta\leq 1/8$:
\begin{align}\label{eq:appG1}
P\left(\left|\frac{1}{n}\sum_{k=0}^{n-1}g_k(\gss{X}{k}{})-\kappa\right|\geq \tilde{\delta}\kappa\right)\leq 2c\left\|P_{\sbms{X}{0}}\right\|_{\pi}\exp\left(-\frac{\tilde{\delta}^2\kappa n}{72 t_{mix}(\theta)}\right),
\end{align}
when $0\leq \tilde{\delta}\leq 1$. In our context, $\|\g{v}\|_{\pi}$ is the $\pi-$norm of a vector $\g{v}\in \mathbb{R}^{2^p}$ defined by $\|\g{v}\|_{\pi}=\sqrt{\langle\g{v},\g{v}\rangle}_{\pi}=\sqrt{\sum_{j=1}^{2^p}\frac{v_j^2}{\pi_j}}$.

First, fix an $l$ and consider the plug-in estimator $\hat{P}(X_l=x_l)=n^{-1}\sum_{k=0}^{n-1}\mathbb{I}\left\{[\gss{X}{k}{}]_l=x_l\right\}$, where $x_l\in \{0,1\}$. 
Then $E_{\pi}[\hat{P}(X_l=x_l)]=P(X_l=x_l)$ and (\ref{eq:appG1}) gives:
\begin{align*}
P\left(\left|\hat{P}(X_l=x_l)-P(X_l=x_l)\right|\geq \tilde{\delta}P(X_l=x_l)\right)\leq 2c
\left\|P_{\sbms{X}{0}}\right\|_{\pi}\exp\left(-\frac{\tilde{\delta}^2P(X_l=x_l) n}{72 t_{mix}(\theta)}\right), 
\end{align*}
where $0 \leq \tilde{\delta}\leq 1$. Using (\ref{eq:measureSpec1}), we obtain:
\begin{align*}
P\left(\left|\hat{P}(X_l=x_l)-P(X_l=x_l)\right|\geq \tilde{\delta}\right)\leq 2c
\left\|P_{\sbms{X}{0}}\right\|_{\pi}\exp\left(-\frac{\tilde{\delta}^2\tilde{\beta} n}{72 t_{mix}(\theta)}\right).
\end{align*} 
By union bounding over $l\in [p]$ and $x_l\in \{0,1\}$, we obtain: 
\begin{align}\label{eq:appG3}
P\left(\max_{ l\in [p],x_l\in \{0,1\}}\left|\hat{P}(X_l=x_l)-P(X_l=x_l)\right|\geq \tilde{\delta}\right)\leq  4pc
\left\|P_{\sbms{X}{0}}\right\|_{\pi}\exp\left(-\frac{\tilde{\delta}^2\tilde{\beta} n}{72 t_{mix}(\theta)}\right).
\end{align}

We now turn to the process $\tilde{\g{Z}}=\{\gss{Z}{n}{}=(\gss{X}{n}{},\gss{X}{n-1}{})\}_{n\geq 1}$ with $\gss{X}{0}{}\sim P_{\sbms{X}{0}}$. It is straightforward to show that the $\theta-$mixing time of $\tilde{\g{Z}}$ coincides with that of $\tilde{\g{X}}$ (Thinking of this point in terms of coupling times, the coupling time of $\tilde{\g{Z}}$ is the coupling time of $\tilde{\g{X}}$ increased by one step, hence the mixing times of $\tilde{\g{X}}$ and $\tilde{\g{Z}}$ essentially coincide). Thus, as before we obtain:
\begin{align*}
P\left(\left|\hat{P}(Z_{ml}=z_{ml})-P(Z_{ml}=z_{ml})\right|\geq \tilde{\delta}P(Z_{ml}=z_{ml})\right)\leq 2c
\left\|P_{\sbms{Z}{1}}\right\|_{\pi'}\exp\left(-\frac{\tilde{\delta}^2P(Z_{ml}=z_{ml}) (n-1)}{72 t_{mix}(\theta)}\right), 
\end{align*}
where $0 \leq \tilde{\delta}\leq 1$ and $Z_{ml}=(X_{m}^{+1},X_l)$, $z_{ml}\in \{(1,0),(1,1)\}$.  We can now use (\ref{eq:measureSpec1}) and union bounding to obtain: 
\begin{align}\label{eq:appG4}
&P\left(\max_{(m,l)\in [p]\times [p],z_{ml}\in \{(1,0),(1,1)\}}\left|\hat{P}(Z_{ml}=z_{ml})-P(Z_{ml}=z_{ml})\right|\geq \tilde{\delta}\right) \leq 4c p^2
\left\|P_{\sbms{Z}{1}}\right\|_{\pi'}\cdot\nonumber\\&\exp\left(\frac{-\tilde{\delta}^2\tilde{\beta} (n-1)}{72 t_{mix}(\theta)}\right).
\end{align}

Observing now (\ref{eq:appG3}) and (\ref{eq:appG4}), we conclude that all the desired events, i.e., 
\begin{align*}
&\left|\hat{P}(X_l=x_l)-P(X_l=x_l)\right|\leq \tilde{\delta},\\
&\left|\hat{P}(Z_{ml}=z_{ml})-P(Z_{ml}=z_{ml})\right|\leq \tilde{\delta},\ \ z_{ml}\in \{(1,0),(1,1)\}
\end{align*}
for all $l,m\in [p]$  hold with probability at least $1-\gamma$ if 

\begin{align*}
n\geq 1+\frac{\log\left(\frac{4p^2}{\gamma} C\right)72 t_{mix}(\theta)}{\tilde{\delta}^2\tilde{\beta}},
\end{align*}
where
\begin{align}\label{eq:appGm1}
C=c(\left\|P_{\sbms{X}{0}}\right\|_{\pi}\vee \left\|P_{\sbms{Z}{1}}\right\|_{\pi'}).
\end{align}
Note that since we sample $\gss{X}{0}{},\gss{X}{1}{},\ldots,\gss{X}{n-1}{}$ with $\gss{X}{0}{}\sim \pi$, we have that $\left\|P_{\sbms{X}{0}}\right\|_{\pi}=\left\|P_{\sbms{Z}{1}}\right\|_{\pi'}=1$ and thus, $C=c$.

We now focus on bounding $\left|\hat{\nu}_{i|j}-\nu_{i|j}\right|$ for $i,j\in [p]$. We have:
{\small
\begin{align*}
&\left|\hat{\nu}_{i|j}-\nu_{i|j}\right|=\left|\hat{P}(X_i^{+1}=1|X_j=1)-\hat{P}(X_i^{+1}=1|X_j=0)-P(X_i^{+1}=1|X_j=1)+P(X_i^{+1}=1|X_j=0)\right|\leq \\& \underbrace{\left|\hat{P}(X_i^{+1}=1|X_j=1)-P(X_i^{+1}=1|X_j=1)\right|}_{C_1}+ \underbrace{\left|\hat{P}(X_i^{+1}=1|X_j=0)-P(X_i^{+1}=1|X_j=0)\right|}_{C_2}=C_1+C_2.
\end{align*}
}
Dealing with $C_1$, we obtain:
{\small
\begin{align*}
&C_1\leq \left|\frac{\hat{P}(X_i^{+1}=1,X_j=1)}{\hat{P}(X_j=1)}-\frac{\hat{P}(X_i^{+1}=1,X_j=1)}{P(X_j=1)}\right|+ \left|\frac{\hat{P}(X_i^{+1}=1,X_j=1)}{P(X_j=1)}-\frac{P(X_i^{+1}=1,X_j=1)}{P(X_j=1)}\right|\leq\\& \frac{\hat{P}(X_i^{+1}=1,X_j=1)}{\hat{P}(X_j=1)}\frac{\tilde{\delta}}{P(X_j=1)}+\frac{\tilde{\delta}}{P(X_j=1)}= 
\hat{P}(X_i^{+1}=1|X_j=1)\frac{\tilde{\delta}}{P(X_j=1)}+\frac{\tilde{\delta}}{P(X_j=1)}\leq\\& \frac{2\tilde{\delta}}{P(X_j=1)}\leq \frac{2\tilde{\delta}}{\tilde{\beta}},
\end{align*}
}
where in the last inequality (\ref{eq:measureSpec1}) has been used.

By symmetry, the same bound holds for $C_2$. Thus, 
\begin{align*}
\left|\hat{\nu}_{i|j}-\nu_{i|j}\right|\leq \frac{4\tilde{\delta}}{\tilde{\beta}}.
\end{align*}
Choosing $\tilde{\delta}=\varepsilon \tilde{\beta}/4$, we have that $P(\mathcal{A}(\varepsilon))\geq 1-\gamma$ when 
\begin{align*}
n\geq 1+\frac{1152\log\left(\frac{4p^2}{\gamma}C\right)t_{mix}(\theta)}{\varepsilon^2\tilde{\beta}^3}.
\end{align*}

\end{proof}

\subsubsection{BAR Identifiability Condition}
\label{subsubsec:IdCodBAR}

We are now ready to define a condition such that the supergraph selection stage of Alg. \ref{alg:Obs1} succeeds with high probability. The condition
is derived through a more detailed study of $\nu_{m|l}$ (or $\nu_{i|j}$ in the previous sections for a different pair of subscript letters as the corresponding indices). As mentioned earlier, $\nu_{m|l}$ is defined with respect to the stationary measure of the BAR model and it is given by the expression:
\begin{equation}\label{eq:nu_ml_BARIdCond}
\nu_{m|l}=P(X^{+1}_m=1|X_l=1)-P(X^{+1}_m=1|X_l=0).
\end{equation}
 We also denote by $\mathcal{S}(m\setminus l)$ the set $\mathcal{S}(m)\setminus \{l\}$. Moreover, each $f_i(\cdot), i\in [p]$ in (\ref{eq:BAR}) is specified by the corresponding two sets $\mathcal{S}^{+}(i)$ and $\mathcal{S}^{-}(i)$.  For $l\in \mathcal{S}(m)$, it is easy to show that
 
\begin{align}\label{eq:numl}
\nu_{m|l}=a_{ml}\left(f_{m}(X_l=1)-f_{m}(X_l=0)\right)+
\sum_{j\in \mathcal{S}(m\setminus l)}a_{mj}\left(E[f_{m}(X_j)|X_{l}=1]-E[f_{m}(X_j)|X_{l}=0]\right),
\end{align}   

which follows by conditioning and summing over the rest of the nodes in $\mathcal{S}(m)$ and on $W_m^{+1}$.

Similarly, for $l'\notin \mathcal{S}(m)$, 
\[
\nu_{m|l'}=\sum_{j\in \mathcal{S}(m)}a_{mj}\left(E[f_{m}(X_j)|X_{l'}=1]-E[f_{m}(X_j)|X_{l'}=0]\right).
\]

For $l\in \mathcal{S}(m)$, if $l\in \mathcal{S}^{+}(m)$ then 
\[
\nu_{m|l}=a_{ml}+\cdots
\]
and we say that $X_l$ \emph{positively causes} $X^{+1}_m$. On the other hand, if $l\in \mathcal{S}^{-}(m)$ then 
\[
\nu_{m|l}=-a_{ml}+\cdots
\]
and we say that $X_l$ \emph{negatively causes} $X^{+1}_m$. Observe also that 
\[
-1\leq E[f_m(X_j)|X_{i}=1]-E[f_m(X_j)|X_{i}=0]\leq 1, \forall j\in \mathcal{S}(m),\forall m, i\in [p],
\]
 since $E[f_m(X_j)|X_{i}=*]=P(X_j=1|X_{i}=*), \forall j\in \mathcal{S}^{+}(m)$ and $E[f_m(X_j)|X_{i}=*]=P(X_j=0|X_{i}=*), \forall j\in \mathcal{S}^{-}(m)$. Here, $*$ denotes either $0$ or $1$.
\begin{defin}\label{def:2}
Consider any $i,j\in \mathcal{V}$. We will say that node $i$ is \emph{positively} (\emph{negatively}) correlated with 
$j$ if $P(X_i=1|X_j=1)>P(X_i=1|X_j=0)$ (<) or alternatively if $P(X_i=0|X_j=0)>P(X_i=0|X_j=1)$ (<). By symmetry, the definitions can be also expressed 
in terms of $P(X_j|X_i)$. Here, $i,j$ refer to the \emph{same} temporal index $k$.
\end{defin}

Fix a temporal index $k$ and let $l\in \mathcal{V}$. Consider the partition of $\mathcal{V}_l=\mathcal{V}\setminus \{l\}$ given by $\mathcal{D}_l^{+}\cup \mathcal{D}_l^{-}$, where 
$\mathcal{D}_l^{+}\cap \mathcal{D}_l^{-}=\emptyset$. We assume that $\mathcal{D}_l^{+}$ contains all $j\in \mathcal{V}_l$ that are positively correlated with $l$ and  $\mathcal{D}_l^{-}$ contains all $j\in \mathcal{V}_l$ that are negatively correlated with $l$. Furthermore, we note that there is a symmetry here: if $j\in \mathcal{D}_l^{+}$ or  $j\in \mathcal{D}_l^{-}$, then 
 $l\in \mathcal{D}_j^{+}$ or  $l\in \mathcal{D}_j^{-}$, respectively. This is a direct consequence of Definition \ref{def:2}. 
 
 To investigate if the above definitions of ``correlation'' are systematic, we have the following result:  
\begin{cor}\label{cor:Corr1}
Consider Definition \ref{def:2}. Let $j\in \mathcal{S}^{+}(m)$. Then:
\begin{align*}
&E[f_m(X_j)|X_{i}=1]-E[f_m(X_j)|X_{i}=0]=P(X_j=1|X_i=1)-\\
&P(X_j=1|X_i=0)=P(X_j=0|X_i=0)-P(X_j=0|X_i=1).
\end{align*}
Also, the difference $E[X_j|X_{i}=1]-E[X_j|X_{i}=0]$ is \emph{not symmetric} in $i,j$. 

Moreover, similar results hold for $j\in \mathcal{S}^{-}(m)$.
\end{cor}
\begin{proof}
Appendix \ref{app:3}.
\end{proof}

We are now ready to investigate the identifiability condition required such that the supergraph selection stage of Alg. \ref{alg:Obs1} correctly estimates a superset of the network. The main task will be to control the \emph{degree of pairwise correlations} between nodes in $\mathcal{V}$ for the same temporal index $k$ such that the first stage succeeds. 

Consider (\ref{eq:numl}) and let $l\in \mathcal{S}^{+}(m)$. Then $f_m(X_l=1)-f_m(X_l=0)=1$.  Nodes in $\mathcal{S}^{+}(m)\cap \mathcal{D}_l^{+}$ and in $\mathcal{S}^{-}(m)\cap \mathcal{D}_l^{-}$ push $\nu_{m|l}$ to become more positive, i.e., $>a_{ml}$. Thus, we allow \emph{from independence up to the highest possible correlation dictated by} (\ref{eq:measureSpec3}) of these nodes with $l$:
\begin{align}\label{eq:BARIdCond_111}
E[f_m(X_j)]X_l=1]-E[f_m(X_j)|X_l=0]\in [0,1),\ \ \forall j\in (\mathcal{S}^{+}(m)\cap \mathcal{D}_l^{+}) \cup(\mathcal{S}^{-}(m)\cap \mathcal{D}_l^{-}).
\end{align} 
Similarly, for $l\in \mathcal{S}^{-}(m)$, $f_m(X_l=1)-f_m(X_l=0)=-1$ and  we have
\begin{align}\label{eq:BARIdCond_222}
E[f_m(X_j)]X_l=1]-E[f_m(X_j)|X_l=0]\in (-1,0],\ \ \forall j\in (\mathcal{S}^{-}(m)\cap \mathcal{D}_l^{+}) \cup(\mathcal{S}^{+}(m)\cap \mathcal{D}_l^{-}),
\end{align} 
i.e.,  we allow \emph{from independence up to the highest possible correlation} of these nodes with $l$ also here,
since $\nu_{m|l}=-a_{ml}+\cdots$ and these nodes push $\nu_{m|l}$ to become more negative, i.e., $<-a_{ml}$.

\textbf{Interpretation of (\ref{eq:BARIdCond_111}) and (\ref{eq:BARIdCond_222})}: The nodes $j$ in (\ref{eq:BARIdCond_111}) and (\ref{eq:BARIdCond_222}) lead to an increase in magnitude of the corresponding $\nu_{m|l}$, co-signed with $a_{ml}$ in each case. Therefore, these nodes facilitate the supergraph selection stage of Alg. \ref{alg:Obs1}(see line $3$ in Alg. \ref{alg:Obs1}). Therefore, we impose \emph{no} constraints on these nodes. 

Let $l\in \mathcal{S}(m)$ and associate with it the set $\mathcal{L}_{ml}=\mathcal{S}(m)\setminus (\mathcal{S}^{+}(m)\cap \mathcal{D}_l^{+}) \cup(\mathcal{S}^{-}(m)\cap \mathcal{D}_l^{-})$ if $l\in \mathcal{S}^{+}(m)$ or  $\mathcal{L}_{ml}=\mathcal{S}(m)\setminus (\mathcal{S}^{-}(m)\cap \mathcal{D}_l^{+}) \cup(\mathcal{S}^{+}(m)\cap \mathcal{D}_l^{-})$ if 
 $l\in \mathcal{S}^{-}(m)$. Consider numbers $\gamma_{j,ml}\in [0,1)$ and restrict  
\begin{align}
&E[f_m(X_j)]X_l=1]-E[f_m(X_j)|X_l=0]\in [-\gamma_{j,ml},0],\ \ \forall j\in \mathcal{L}_{ml}, \forall l\in \mathcal{S}^{+}(m)\label{eq:BARIdCond_333}\\
&E[f_m(X_j)]X_l=1]-E[f_m(X_j)|X_l=0]\in [0,\gamma_{j,ml}],\ \ \forall j\in \mathcal{L}_{ml}, \forall l \in \mathcal{S}^{-}(m).\label{eq:BARIdCond_444}
\end{align}

\textbf{Interpretation of (\ref{eq:BARIdCond_333}) and (\ref{eq:BARIdCond_444})}: The last two equations correspond to the required control on the set of target pairwise correlations. More specifically, we constrain the values
of the correlations that lower $\nu_{m|l}$, when $l\in \mathcal{S}^{+}(m)$ and those correlations that increase $\nu_{m|l}$ when $l\in \mathcal{S}^{-}(m)$.   

With these introductions in mind, we can now give the 
following sufficient condition for identifiability:

\textbf{BAR Identifiability Condition}:
There exist real numbers  $\chi_m,\{\gamma_{j,ml}\}_{j\in \mathcal{L}_{ml},l\in \mathcal{S}(m)}$, $\eta_{\mathcal{S}(m)}, \eta_{\mathcal{S}^c(m)}\in (0,1)$ for all $m\in [p]$, such that  $\chi_m<\eta_{\mathcal{S}(m)}\leq a_{\min}$, $\eta_{\mathcal{S}^c(m)}\leq (\eta_{\mathcal{S}(m)}-\chi_m)/\sum_{j=1}^pa_{mj}$ for which 
\begin{align}
&0\leq \gamma_{j,ml}< 1\wedge \frac{a_{ml}}{a_{mj}},\ \   \forall j\in \mathcal{L}_{ml}, \forall l\in \mathcal{S}(m),\nonumber\\
&|E[f_m(X_j)]X_l=1]-E[f_m(X_j)|X_l=0]|\leq \gamma_{j,ml}, \forall j\in \mathcal{L}_{ml}, \forall l \in \mathcal{S}(m), \nonumber\\
&a_{ml}-\sum_{j\in \mathcal{L}_{ml}}\gamma_{j,ml}a_{mj}\geq \eta_{\mathcal{S}(m)}, \forall l\in \mathcal{S}(m), \nonumber\\
&\text{and}\: \left|E[f_m(X_j)|X_{l'}=1]-E[f_m(X_j)|X_{l'}=0]\right|\leq \eta_{\mathcal{S}^c(m)},\forall j\in \mathcal{S}(m),  \forall l'\notin \mathcal{S}(m).
\end{align}

\textbf{Interpretation of the BAR Identifiability Condition:} The BAR Identifiability Condition imposes an upper bound on the absolute value of the correlations between any $l\in \mathcal{S}(m)$ and any node in $\mathcal{L}_{ml}$. A similar condition is imposed on the allowed correlation between any $l'\notin \mathcal{S}(m)$ and any node in $\mathcal{S}(m)$. The numbers $\eta_{\mathcal{S}(m)}$ and $\eta_{\mathcal{S}^c(m)}$ make sure that $\nu_{m|l}$ and $\nu_{m|l'}$ for $l\in \mathcal{S}(m)$ and $l'\notin \mathcal{S}(m)$ are sufficiently separated under all allowed values of the pairwise correlations, such that the supergraph selection stage in Alg. \ref{alg:Obs1} can isolate the true graph for any $d$ such that $d_i\leq d, \forall i\in [m]$.  

\textbf{A scenario of special interest}: Let without loss of generality $l\in \mathcal{S}^{+}(m)$ (a similar scenario can be constructed for $l\in \mathcal{S}^{-}(m)$). Then, $\nu_{m|l}=a_{ml}+\cdots$. Any $j\in (\mathcal{S}^{+}(m)\cap \mathcal{D}_l^{+}) \cup(\mathcal{S}^{-}(m)\cap \mathcal{D}_l^{-})$ makes $\nu_{m|l}$ more positive, i.e., causes an increase to $|\nu_{m|l}|$ (see line $3$ in Alg. \ref{alg:Obs1}), helping in the selection of $l$ by the supergraph selection stage. Any $j\in \mathcal{L}_{ml}$ leads to a reduction of $\nu_{m|l}$. The constraints in (\ref{eq:BARIdCond_333}) and $a_{ml}-\sum_{j\in \mathcal{L}_{ml}}\gamma_{j,ml}a_{mj}\geq \eta_{\mathcal{S}(m)}, \forall l\in \mathcal{S}(m)$ in the BAR Identifiability Condition imply that we allow 
as little correlation between $j\in \mathcal{L}_{ml}$ and $l$ as it would preserve the positive sign of $\nu_{m|l}$. According to Alg. \ref{alg:Obs1}, this is \emph{not} necessary. Indeed, suppose that for some $j\in \mathcal{L}_{ml}$ the corresponding $a_{mj}$ is much larger than any other such coefficient in the $m$th row of $\g{A}$ (or $\g{\bar{A}}$). Allowing then very small \emph{and} very large values of $\gamma_{j,ml}$, we could either have $\nu_{m|l}>0$ and $|\nu_{m|l}|$ sufficiently large or $\nu_{m|l}<0$ and $|\nu_{m|l}|$ sufficiently large, respectively. These subcases would lead to the (correct) selection of $l$ by the supergraph selection stage. Nevertheless, we have chosen to eliminate such scenarios via the above definition of the BAR Identifiability Condition in order to validate the algorithmic variants described in the first two remarks after Alg. \ref{alg:Obs1}.  Clearly, these observations lead to 
a first relaxation of the BAR Identifiability Condition in the case that  the algorithmic variants described in the first two remarks after Alg. \ref{alg:Obs1} are of no interest.

\textbf{Relaxation of the BAR Identifiability Condition}: The provided form of the BAR Identifiability Condition is the most stringent one in the sense that 
it accounts for any possible value of $d$, even for the scenario where $d_i=d$ for some $i$'s in $[p]$ or the extreme case $d_1=d_2=\cdots=d_p=d$. One may note that if $d$ is sufficiently large, then the condition can be relaxed, allowing to $l'\notin \mathcal{S}(m)$ for some or even all $m$ to yield $|\nu_{m|l'}|\geq |\nu_{m|l}|$ for some $l$ or $l$'s in $\mathcal{S}(m)$, as long as all $l\in \mathcal{S}(m)$ are among the $d$ largest $\nu_{m|l}$'s picked by the supergraph selection stage. Nevertheless, in the following analysis we use the stringent form of the BAR Identifiability Condition to account for the worst case scenario in this respect, i.e., to have probabilistic guarantees about the graph recovery problem even in scenarios where  $d_i=d$ for some $i$'s in $[p]$ or $d_1=d_2=\cdots=d_p=d$.

\textbf{Qualitative Comparison with Existing Methods}: Interestingly enough, in the BAR model it appears to be good to have up to strong correlation between $l$ and nodes in $\mathcal{S}(m)\setminus \mathcal{L}_{ml}, \forall m \in [p], \forall l\in \mathcal{S}(m)$. In this scenario, any other method of picking $\mathrm{supp}(\g{A})$ such as Lasso and other relevant convex optimization approaches for model selection, would fail to the best of our knowledge. Moreover, because of this interesting characteristic, although the BAR Identifiability Condition imposes restrictions on the ``harmful'' pairwise correlations, one may observe that in practice the extent of harmful correlations that can be accommodated, can be much larger given that we have strong ``helpful" correlations.

\subsubsection{Probabilistic Guarantees for the Supergraph Selection Stage}
\label{subsubsec:mainthmSupergraphSelection}

The following Theorem verifies the correctness of the supergraph selection stage based on the \emph{BAR Identifiability Condition}:
\begin{thrm}\label{thm:BARObs}
Let $\mathcal{G}\in \vec{\mathcal{G}}_{p,d}$, $(\g{a},\g{b})\in \Omega_{a_{min},b_{min}}(\mathcal{G})$ and $\gss{X}{0}{}\sim \pi$, where $\pi$ is the stationary measure. Suppose that we observe 
the BAR sequence $\gss{X}{0}{},\gss{X}{1}{},\ldots, \gss{X}{n-1}{}$ for $n$ given by (\ref{eq:nRequired}). 
Assume that the BAR Identifiability Condition holds and that $\varepsilon<\min_{m\in [p]}\chi_m/2$. Then, given any valid $d$ such that $d_i\leq d, \forall i\in [p]$, the supergraph selection stage of Alg. \ref{alg:Obs1} correctly identifies on overestimate of the true graph with probability at least $1-\gamma$.
\end{thrm}
  
\begin{proof}  The proof is straightforward. By the considered sample complexity and by Lemma \ref{lem:empEst1}, the event $\mathcal{A}(\varepsilon)$ holds with probability at least $1-\gamma$. The validity of the BAR Identifiability Condition implies that for all $m\in [p]$
\begin{align}\label{eq:thmMain_21}
|\nu_{m|l}|\geq \chi_m+|\nu_{m|l'}|,\: \text{for}\: \text{all}\: l\in \mathcal{S}(m), l'\notin \mathcal{S}(m),
\end{align}
\begin{align*}
\nu_{m|l}>0, \forall l\in \mathcal{S}^{+}(m), \ \ \nu_{m|l}<0, \forall l\in \mathcal{S}^{-}(m).
\end{align*}
Consider any $l\in \mathcal{S}^{+}(m)$. Then, in the worst case $\hat{\nu}_{m|l}=\nu_{m|l}-\varepsilon$ and $\hat{\nu}_{m|l'}=\nu_{m|l'}+\varepsilon$ for some $l'\notin \mathcal{S}(m)$ such that $\nu_{m|l'}=\nu_{m|l}-\chi_m$.
Therefore, we immediately see that 
\[
\hat{\nu}_{m|l}>\hat{\nu}_{m|l'},
\]
due to $2\varepsilon<\min_{m\in [p]}\chi_m$.
An analogous argument holds for $l\in \mathcal{S}^{-}(m)$. Thus, the supergraph selection stage of Alg. \ref{alg:Obs1} correctly identifies on overestimate of the true graph with probability at least $1-\gamma$.
\end{proof}

\textbf{Remark}: Theorem \ref{thm:BARObs} provides the sample complexity and probabilistic guarantees for the first two remarks after Alg. \ref{alg:Obs1}.  

\subsection{Key Technical Results for the Supergraph Trimming Stage}

We now turn to the supergraph trimming stage. For any $i\in \mathcal{V}$, we require a sample complexity that will produce sufficiently accurate estimates of $P\left(X_i^{+1}=1|\gss{X}{\tilde{\mathcal{S}}(i)}{}=\gss{x}{\tilde{\mathcal{S}}(i)}{}\right)$ for any $d-$sized set $\tilde{\mathcal{S}}(i)\subset [p]$, such 
that the performance of Alg. \ref{alg:Obs1} is guaranteed for any possible $\hat{\mathcal{S}}(i)$ (or $\hat{\mathcal{S}}$) returned by the supergraph selection stage. To this end, we define the event 
\begin{align*}
\tilde{\mathcal{A}}(\tilde{\varepsilon})=&\left\{\left|\hat{P}\left(X_i^{+1}=1|\gss{X}{\tilde{\mathcal{S}}(i)}{}=\gss{x}{\tilde{\mathcal{S}}(i)}{}\right)-P\left(X_i^{+1}=1|\gss{X}{\tilde{\mathcal{S}}(i)}{}=\gss{x}{\tilde{\mathcal{S}}(i)}{}\right)\right|\leq \tilde{\varepsilon}, \forall i\in \mathcal{V}, \forall \tilde{\mathcal{S}}(i)\subset [p]\right.\\& \left. \text{with}\: |\tilde{\mathcal{S}}(i)|=d,\forall \gss{x}{\tilde{\mathcal{S}}(i)}{}\in \{0,1\}^d\right\}.
\end{align*}
The required sample complexity such that $\tilde{\mathcal{A}}(\tilde{\varepsilon})$ holds with high probability is summarized by the following Lemma:
\begin{lem}\label{lem:empEst2}
Let $0<\gamma<1$ and $\theta\leq 1/8$. Assume that $t_{mix}(\theta)$ is the $\theta-$mixing time of $\g{\tilde{X}}=\{\gss{X}{n}{}\}_{n\geq 0}$ and  $\g{\tilde{Z}}=\{\gss{Z}{n}{}\}_{n\geq 1}$. If 
 \begin{align}\label{eq:nRequired_Extended}
  n\geq 1+\frac{288\log\left(\frac{2^{d+1}C}{\gamma} p{p\choose d}\right)t_{mix}(\theta)}{\tilde{\varepsilon}^2\bar{\beta}^3}.
\end{align}
where $C$ is the same constant as in Lemma \ref{lem:empEst1},
then $P(\tilde{\mathcal{A}}(\tilde{\varepsilon}))\geq 1-\gamma$.
\end{lem}
\begin{proof}
The proof is in the same lines as the proof of Lemma \ref{lem:empEst1}. By union bounding, we immediately have:
\begin{align}\label{eq:app93}
&P\left(\max_{\forall \tilde{\mathcal{S}}(m)\subset [p]\: \text{s.t.}\: |\tilde{\mathcal{S}}(m)|=d, \forall \sbms{x}{\hat{\mathcal{S}}(m)}\in \{0,1\}^d}\left|\hat{P}(\gss{X}{\tilde{\mathcal{S}}(m)}{}=\gss{x}{\tilde{\mathcal{S}}(m)}{})-P(\gss{X}{\tilde{\mathcal{S}}(m)}{}=\gss{x}{\tilde{\mathcal{S}}(m)}{})\right|\geq \tilde{\delta}\right)\leq  \nonumber\\ & 2^{d+1}{p\choose d}c
\left\|P_{\sbms{X}{0}}\right\|_{\pi} \exp\left(-\frac{\tilde{\delta}^2\bar{\beta} n}{72 t_{mix}(\theta)}\right).
\end{align} 
Note that $m$ in $\mathcal{S}(m)$ here is dummy in the sense that for any given $m$ the above concentration inequality accounts for all $d-$sized subsets of $[p]$.

We now turn to the process $\{\gss{Z}{n}{}=(\gss{X}{n}{},\gss{X}{n-1}{})\}_{n\geq 1}$ with $\gss{X}{0}{}\sim P_{\sbms{X}{0}}$. As before we obtain:
{\small
\begin{align}\label{eq:app94}
&P\left(\max_{\forall m\in [p],\forall \tilde{\mathcal{S}}(m)\subset [p]\: \text{s.t.}\: |\tilde{\mathcal{S}}(m)|=d, \forall \sbms{x}{\hat{\mathcal{S}}(m)}\in \{0,1\}^d}\left|\hat{P}(X_m^{+1}=1,\gss{X}{\tilde{\mathcal{S}}(m)}{}=\gss{x}{\tilde{\mathcal{S}}(m)}{})-P(X_m^{+1}=1,\gss{X}{\tilde{\mathcal{S}}(m)}{}=\gss{x}{\tilde{\mathcal{S}}(m)}{})\right|\geq \tilde{\delta}\right)\nonumber\\& \leq  2^{d+1}c p{p\choose d}
\left\|P_{\sbms{Z}{1}}\right\|_{\pi'}\exp\left(-\frac{\tilde{\delta}^2\bar{\beta} (n-1)}{72 t_{mix}(\theta)}\right).
\end{align} 
}

Observing now (\ref{eq:app93}) and (\ref{eq:app94}), we conclude that all the desired events, i.e., 
\begin{align*}
&\left|\hat{P}(\gss{X}{\tilde{\mathcal{S}}(m)}{}=\gss{x}{\tilde{\mathcal{S}}(m)}{})-P(\gss{X}{\tilde{\mathcal{S}}(m)}{}=\gss{x}{\tilde{\mathcal{S}}(m)}{})\right|\leq \tilde{\delta},\\
&\left|\hat{P}(X_m^{+1}=1,\gss{X}{\tilde{\mathcal{S}}(m)}{}=\gss{x}{\tilde{\mathcal{S}}(m)}{})-P(X_m^{+1}=1,\gss{X}{\tilde{\mathcal{S}}(m)}{}=\gss{x}{\tilde{\mathcal{S}}(m)}{})\right|\leq \tilde{\delta},\forall \gss{x}{\tilde{\mathcal{S}}(m)}{}\in \{0,1\}^d
\end{align*}
for all $m\in [p]$ and for all $\tilde{\mathcal{S}}(m)\subset [p]$ such that $|\tilde{\mathcal{S}}(m)|=d$,  hold with probability at least $1-\gamma$ if 

\begin{align*}
n\geq 1+\frac{\log\left(\frac{2^{d+1}C}{\gamma} p{p\choose d}\right)72 t_{mix}(\theta)}{\tilde{\delta}^2\bar{\beta}},
\end{align*}
where
\begin{align}\label{eq:appGm1}
C=c(\left\|P_{\sbms{X}{0}}\right\|_{\pi}\vee \left\|P_{\sbms{Z}{1}}\right\|_{\pi'}).
\end{align}
Note again that since we sample $\gss{X}{0}{},\gss{X}{1}{},\ldots,\gss{X}{n-1}{}$ with $P_{\sbms{X}{0}}=\pi$, we have that $\left\|P_{\sbms{X}{0}}\right\|_{\pi}=\left\|P_{\sbms{Z}{1}}\right\|_{\pi'}=1$ and thus, $C=c$.

We now focus on bounding $\left|\hat{P}\left(X_i^{+1}=1|\gss{X}{\tilde{\mathcal{S}}(i)}{}=\gss{x}{\tilde{\mathcal{S}(i)}}{}\right)-P\left(X_i^{+1}=1|\gss{X}{\tilde{\mathcal{S}}(i)}{}=\gss{x}{\tilde{\mathcal{S}}(i)}{}\right)\right|$ for $i\in [p], \tilde{\mathcal{S}}(i)\subset [p],  \gss{x}{\tilde{\mathcal{S}}(i)}{}\in \{0,1\}^d$. Using similar steps as before, we obtain:

\begin{align*}
\left|\hat{P}\left(X_i^{+1}=1|\gss{X}{\tilde{\mathcal{S}}(i)}{}=\gss{x}{\tilde{\mathcal{S}}(i)}{}\right)-P\left(X_i^{+1}=1|\gss{X}{\tilde{\mathcal{S}}(i)}{}=\gss{x}{\tilde{\mathcal{S}}(i)}{}\right)\right|\leq \frac{2\tilde{\delta}}{\bar{\beta}},
\end{align*}

where in the last inequality (\ref{eq:stmeasBound2}) has been used.

Choosing $\tilde{\delta}=\tilde{\varepsilon} \bar{\beta}/2$, we have that $P(\tilde{\mathcal{A}}(\tilde{\varepsilon}))\geq 1-\gamma$ when 
\begin{align*}
n\geq 1+\frac{288\log\left(\frac{2^{d+1}C}{\gamma} p{p\choose d}\right)t_{mix}(\theta)}{\tilde{\varepsilon}^2\bar{\beta}^3}.
\end{align*}

\end{proof}

\subsection{Proof of Theorem \ref{thm:ExtendedBARObs}}

We first note that the sample complexity in (\ref{eq:nRequired}) is generally smaller than the sample complexity in (\ref{eq:nRequired_Extended11}) or (\ref{eq:nRequired_Extended}). 
Therefore, the proof is immediate by combining the previous results in this section. The probability $(1-\gamma)^2$ occurs as the product of the probabilities that 
the supergraph selection stage returns an overestimate of the true graph multiplied by the probability that the supergraph trimming stage
correctly determines all in-degrees and neighborhoods.

Since $t_{mix}(\theta)=O(\log p)$, we conclude that the sample complexity of BARObs$\left(\gss{X}{0:n-1}{},d,\tau\leq a_{min}/4\right)$ is 
$n=\Omega(\log^2 p)$.

\section{A Lower Bound on the Sample Complexity}
\label{sec:LowerInf}

To assess the quality of the proposed algorithm, we derive an information-theoretic lower bound on the sample complexity of any such algorithm based on Fano's inequality:
\begin{lem}\label{lem:FanoBAR}
For any given $0<\epsilon<1$, requiring $\mathcal{R}_{*}(\psi)\leq \epsilon$ implies that $n\geq \frac{(1-\epsilon)}{p}\sum_{i=1}^p\log {p\choose d_i}$.
\end{lem}

\begin{proof} 
Appendix \ref{app:8}.
\end{proof}

This Lemma shows that a necessary sample complexity for any method is $\propto d\log p$ for $d\ll p$. Since, $1\leq d\leq d_{*}$, we conclude that necessarily $n=\Omega(\log p)$.
Comparing this sample complexity with the sample complexity of BARObs$\left(\gss{X}{0:n-1}{},d,\tau\leq a_{min}/4\right)$ we can see that  BARObs$\left(\gss{X}{0:n-1}{},d,\tau\leq a_{min}/4\right)$ is nearly order-optimal as it is only a multiple of $\log p$ away from this information-theoretic lower bound.

\section{Comparison with Other Models}
\label{sec:Comparison}

In this section, we define and study some useful BAR model variations. Through this study, we make comparisons with the BAR model (\ref{eq:BAR}).   

\subsection{BAR Random Walks on the Hypercube $\{0,1\}^p$}

Motivated by the single-site operation of the Glauber dynamics, we define the following two random walks:

\emph{BAR random walk on the hypercube $\{0,1\}^p$}: Let  $\gss{X}{k}{}=[[\gss{X}{k}{}]_1, [\gss{X}{k}{}]_2, \ldots, [\gss{X}{k}{}]_p]^T$ be the state vector at time $k$. At time $k+1$, node $i$ is selected uniformly at random with probability $1/p$ and is updated according to the following rule:
\begin{align}\label{eq:BARhypercube}
[\gss{X}{k+1}{}]_i=\mathbb{I}\{U\leq \gss{a}{i}{T}f_i(\gss{X}{k}{})+b_{ii}[\gss{W}{k+1}{}]_i\},
\end{align}
resulting to $\gss{X}{k+1}{}=[[\gss{X}{k}{}]_1, [\gss{X}{k}{}]_2, \ldots,[\gss{X}{k}{}]_{i-1},[\gss{X}{k+1}{}]_i,[\gss{X}{k}{}]_{i+1},\ldots, [\gss{X}{k}{}]_p]^T$. Here, $U\sim \mathrm{Unif}[0,1]$ is independently drawn at every time instant.  Note that the event $\{\gss{X}{k+1}{}=\gss{X}{k}{}\}$ can occur with positive probability.

A lazy (or delayed) version of the above random walk is also possible:

\emph{Lazy BAR random walk on the hypercube $\{0,1\}^p$}: Let  $\gss{X}{k}{}=[[\gss{X}{k}{}]_1, [\gss{X}{k}{}]_2, \ldots, [\gss{X}{k}{}]_p]^T$ be the state vector at time $k$. At time $k+1$, 
\begin{align*}
\gss{X}{k+1}{}=\left\{
  \begin{array}{cc}
    \gss{X}{k}{} & \text{with probability}\: \varphi \\
    \left[[\gss{X}{k}{}]_1, [\gss{X}{k}{}]_2, \ldots,[\gss{X}{k}{}]_{i-1},[\gss{X}{k+1}{}]_i,[\gss{X}{k}{}]_{i+1},\ldots, [\gss{X}{k}{}]_p\right]^T & \text{with probability}\: \frac{1-\varphi}{p} \\
  \end{array}
\right.,
\end{align*}
where $[\gss{X}{k+1}{}]_i$ in the second line is updated according to (\ref{eq:BARhypercube}).

In the next subsection, we provide some theoretical insights into the mixing properties of these two random walks.

\subsection{Mixing Time Bounds for BAR Random Walks on the Hypercube} 

The analysis of the mixing time for the BAR random walk on the hypercube $\{0,1\}^p$ and of its lazy version is almost the same. Using path coupling, we derive the following result: 
\begin{thrm}\label{thm:BARrwMxing1}
Consider the BAR random walk on the hypercube $\{0,1\}^p$. Under the condition that $\g{\bar{A}}$ is column-substochastic,
\begin{equation}\label{eq:BARrwMixing}
t_{mix,rw}(\theta)\leq \left\lceil \frac{p}{1-\max_{1\leq j\leq p}\sum_{i=1}^{p}a_{ij}}\left(\log p+\log \left(\frac{1}{\theta}\right)\right)\right\rceil
\end{equation}
for any $\theta \in (0,1)$, i.e., $t_{mix,rw}(\theta)=O(p\log p)$. For the lazy version of this walk, 
\begin{equation}\label{eq:BARrwMixingLazy}
t_{mix,lazy}(\theta)\leq \left\lceil \frac{p}{(1-\varphi)\left(1-\max_{1\leq j\leq p}\sum_{i=1}^{p}a_{ij}\right)}\left(\log p+\log \left(\frac{1}{\theta}\right)\right)\right\rceil,
\end{equation}
i.e., $t_{mix,lazy}(\theta)=O(p\log p)$ as well.
\end{thrm}
\begin{proof}
We will use the path coupling method to bound the mixing time of the BAR random walk on the hypercube \cite{bd97}. To this end, we assume that $\{\gss{X}{n}{}\}$ and $\{\gss{Y}{n}{}\}$ are two copies of the BAR random walk, with the later initialized with the corresponding stationary measure. We consider the following coupling: 

\begin{enumerate}
\item At time step $k+1$, pick node $i$ with probability $1/p$.
\item Sample $[\gss{W}{k+1}{}]_i$ from $\mathrm{Ber}(\rho_w)$.
\item Toss $U\sim \mathrm{Unif}[0,1]$ and let 
\begin{align*}
[\gss{X}{k+1}{}]_i=\mathbb{I}\{U\leq \gss{a}{i}{T}f_i(\gss{X}{k}{})+b_{ii}[\gss{W}{k+1}{}]_i\}
\end{align*}
and
\begin{align*}
[\gss{Y}{k+1}{}]_i=\mathbb{I}\{U\leq \gss{a}{i}{T}f_i(\gss{Y}{k}{})+b_{ii}[\gss{W}{k+1}{}]_i\}.
\end{align*}
\end{enumerate}

Clearly, this is a valid coupling since the chains respect their individual transitions. The paths over which we will bound the expected distance are allowed to have only vertices differing in a single bit.  Thus, for any given states $\g{x},\g{y}$, the connecting paths that we consider are of the form: 
\[
\xi=\g{x},\gss{v}{1}{},\gss{v}{2}{},\ldots, \gss{v}{m}{},\g{y},
\]    
where all pairs of successive vertices have Hamming distance $1$. We denote by $d_{H}(\g{x},\g{y})=\sum_{i=1}^p\mathbb{I}\{[\g{x}]_i\neq [\g{y}]_i\}$ the Hamming distance of $\g{x},\g{y}$ and  we focus on the set of 
vertices $\left\{(\g{s},\g{t}):d_{H}(\g{s},\g{t})=1\right\}$.  We assume that $(\g{s},\g{t})$ is such that $s_j=1,t_j=0$ and $s_l=t_l, \forall l\neq j$. The case where $s_j=0$ and $t_j=1$ is exactly symmetric. Let $(\g{s}',\g{t}')$ be the next time-instant vertices:

\underline{For $i\neq j$:} Here, $s_i'\neq t_i'$ with probability $a_{ij}$ and $s_i'=t_i'$ with probability $1-a_{ij}$.

\underline{For $i=j$:} Here, $s_i'\neq t_i'$ with probability $a_{jj}$ and $s_i'=t_i'$ with probability $1-a_{jj}$.

Thus, 
\begin{align*}
E[d_{H}(\g{s}',\g{t}')|\g{s},\g{t}]&=\frac{1}{p}\left(\sum_{i=1,i\neq j}^p2a_{ij}+1-a_{ij}\right)+\frac{1}{p} a_{jj}+\frac{1}{p}0\cdot (1-a_{jj})\\&=\frac{1}{p}\left(\sum_{i=1}^{p}a_{ij}+p-1\right).
\end{align*}
Assuming that $\g{\bar{A}}$ is column-substochastic,
\[
\frac{1}{p}\left(\max_{1\leq j\leq p}\sum_{i=1}^{p}a_{ij}+p-1\right)=1-\frac{1-\max_{1\leq j\leq p}\sum_{i=1}^{p}a_{ij}}{p}<1.
\]
Therefore, using the elementary inequality $1-x\leq e^{-x}$, which holds for any $x\in \mathbb{R}$, we obtain
\begin{equation}\label{eq:columnSubStoch}
E[d_{H}(\g{s}',\g{t}')|\g{s},\g{t}]\leq \exp\left(-\frac{1-\max_{1\leq j\leq p}\sum_{i=1}^{p}a_{ij}}{p}\right)
\end{equation}
Setting $a=\left(1-\max_{1\leq j\leq p}\sum_{i=1}^{p}a_{ij}\right)/p$ and using the formula \cite{bd97}  
\[
t_{mix}(\theta)\leq\left\lceil\frac{-\log(\theta)+\log(\mathrm{diam}(\mathcal{X}))}{a}\right\rceil,
\]
where $\mathrm{diam}(\mathcal{X})$ is the diameter of $\mathcal{X}$ with respect to the Hamming distance, i.e.,  $\mathrm{diam}(\mathcal{X})=p$,
we end up with
\begin{align}\label{eq:lemmaProofMixing2}
t_{mix,rw}(\theta)\leq \left\lceil \frac{p}{1-\max_{1\leq j\leq p}\sum_{i=1}^{p}a_{ij}}\left(\log p+\log \left(\frac{1}{\theta}\right)\right)\right\rceil.
\end{align}

For the lazy version of the random walk, the corresponding coupling is the following:

\begin{enumerate}
\item At time step $k+1$, toss $U_{lazy}\sim \mathrm{Unif}[0,1]$. 
\item If $U_{lazy}\leq \varphi$, then both chains remain idle. If $U_{lazy}>\varphi$, then do the following:
\begin{enumerate} 
\item Pick node $i$ with probability $1/p$.
\item Sample $[\gss{W}{k+1}{}]_i$ from $\mathrm{Ber}(\rho_w)$.
\item Toss $U\sim \mathrm{Unif}[0,1]$ and let 
\begin{align*}
[\gss{X}{k+1}{}]_i=\mathbb{I}\{U\leq \gss{a}{i}{T}f_i(\gss{X}{k}{})+b_{ii}[\gss{W}{k+1}{}]_i\}
\end{align*}
and
\begin{align*}
[\gss{Y}{k+1}{}]_i=\mathbb{I}\{U\leq \gss{a}{i}{T}f_i(\gss{Y}{k}{})+b_{ii}[\gss{W}{k+1}{}]_i\}.
\end{align*}
\end{enumerate}
\end{enumerate}

Based on this coupling, we can easily see that 
\begin{align*}
E[d_{H}(\g{s}',\g{t}')|\g{s},\g{t}]&=\varphi+\frac{1-\varphi}{p}\left(\sum_{i=1,i\neq j}^p2a_{ij}+1-a_{ij}\right)+\frac{1-\varphi}{p} a_{jj}+\frac{1-\varphi}{p}\cdot 0\cdot (1-a_{jj})\\&=\varphi+\frac{1-\varphi}{p}\left(\sum_{i=1}^{p}a_{ij}+p-1\right).
\end{align*}
Assuming again column-substochasticity for $\g{\bar{A}}$ and following similar steps as before, we end up with
\begin{align}
t_{mix,lazy}(\theta)\leq \left\lceil \frac{p}{(1-\varphi)\left(1-\max_{1\leq j\leq p}\sum_{i=1}^{p}a_{ij}\right)}\left(\log p+\log \left(\frac{1}{\theta}\right)\right)\right\rceil.
\end{align}

\end{proof}

The condition of column-substochasticity can be imposed on the definition of BAR random walks on the hypercube. Alternatively, it is of interest to examine conditions that such a property holds with high probability. To this end, we have the following two additional results:
\begin{lem}\label{lem:BARrwMxing2}
Selecting the support of each row of $\g{\bar{A}}$ uniformly at random from the set of $1\times p$ binary vectors with $d_i$ ones for any $i\in [p]$, yields that with probability at least $1-1/p^{2c-1}$ for some $c>1/2$ each column has at most $\sum_{i=1}^pd_i/p+\sqrt{cp\log p}$ nonzero entries.
\end{lem}   
\begin{proof}
Appendix \ref{app:5}.
\end{proof}
\begin{lem}\label{lem:BARrwMxing3}
Choosing the support of $\g{\bar{A}}$ as described by Lemma \ref{lem:BARrwMxing2} and constraining the magnitude of each entry in $\g{\bar{A}}$ to be less than $1/\left(\sum_{i=1}^p\frac{d_i}{p}+\sqrt{cp\log p}\right)$ leads to column-substochasticity of  $\g{\bar{A}}$ with probability at least $1-1/p^{2c-1}$ for some $c>1/2$. In this scenario,  
\[
t_{mix,rw/lazy}(\theta)=O(p\log p)
\]
with at least the same probability.
\end{lem}  
\begin{proof}
Appendix \ref{app:5}.
\end{proof}

\textbf{Remarks}: 
\begin{enumerate}
\item The interpretation of Lemma \ref{lem:BARrwMxing3} is two-fold:  either for a given $p$, each nonzero entry in $\g{\bar{A}}$ is constrained to a maximum value such that the associated BAR model is valid and column-substochasticity of $\g{\bar{A}}$ holds or for a given regime of allowed entry values, an admissible regime for $p$ is determined such that column-substochasticity of $\g{\bar{A}}$ holds with a desired probability. 
\item The BAR random walks touch upon a single node every time they make a transition. By this fact and the coupon's collector problem, we can immediately see that, with high probability as $p$ increases, we need at least $p\log p$ steps to visit and fix all possible discrepancies between $\g{x}$ and $\g{y}$. This observation yields the conclusion that, in fact, $t_{mix,rw}(\theta)=\Theta(p\log p)$ and $t_{mix,lazy}(\theta)=\Theta(p\log p)$ with high probability as $p\rightarrow \infty$.
\end{enumerate}
The above results indicate that BAR random walks on the hypercube mix rapidly under column-substochasticity of $\g{\bar{A}}$. The difference in the general BAR models (\ref{eq:BAR}) and (\ref{eq:BAR2}) is that updates of ultimately all nodes can occur at every time instant. More updates could lead to instability. However, our previous analysis in the context of Theorem \ref{thm:BARMxingMain} showed that  a general BAR chain mixes rapidly \emph{without requiring column-substochasticity of} $\g{\bar{A}}$.

\section{Numerical Examples}
\label{sec:sims} 

In this section, we present numerical examples to demonstrate the performance of the proposed structure observer. We \emph{separately} examine the performance of the supergraph selection stage \emph{only} in the ways described by the first two remarks after Alg. \ref{alg:Obs1} and the performance of the complete BARObs$(\gss{X}{0:n-1}{},d,\tau\leq a_{min}/4)$.   The horizontal axis in Figs. \ref{fig:2a},\ref{fig:2b},\ref{fig:3a},\ref{fig:3b},\ref{fig:4a},\ref{fig:4b},\ref{fig:5a},\ref{fig:5b},\ref{fig:6a}, \ref{fig:6b},\ref{fig:7a}, \ref{fig:7b}, \ref{fig:8a} and \ref{fig:8b} corresponds to the number of samples. The vertical axis in Figs.  \ref{fig:2a},\ref{fig:2b},\ref{fig:3a},\ref{fig:3b},\ref{fig:4a},\ref{fig:4b},\ref{fig:5a},\ref{fig:5b}, \ref{fig:8b} corresponds to $\mathrm{P}(\hat{\mathcal{S}}=\mathcal{S})$ and  $\mathrm{P}((\hat{\mathcal{S}}^{+},\hat{\mathcal{S}}^{-})=(\mathcal{S}^{+},\mathcal{S}^{-}))$ or $\mathrm{P}(\hat{\mathcal{S}}_f=\mathcal{S})$ and $\mathrm{P}((\hat{\mathcal{S}}_f^{+},\hat{\mathcal{S}}_f^{-})=(\mathcal{S}^{+},\mathcal{S}^{-}))$. The vertical axis in Fig. \ref{fig:6b}  corresponds to the fraction of correctly identified edges and non-edges and in \ref{fig:6a}, \ref{fig:7a}, \ref{fig:7b} and \ref{fig:8a}  corresponds to the fraction of correctly identified edges in the BAR model selection.   Figs.  \ref{fig:2a},\ref{fig:2b},\ref{fig:3a},\ref{fig:3b},\ref{fig:4a},\ref{fig:4b},\ref{fig:5a},\ref{fig:5b}, and \ref{fig:6a} correspond to simulated data, while Figs. \ref{fig:6b}-\ref{fig:8b} correspond to pseudoreal data from a model of a biological signaling network. In all figures, we upper bound the $b_i$'s by $b_{max}=0.2$.

\subsection{Simulated Data}

We split this subsection into (i) the performance of the supergraph selection stage only via the schemes described in the first two remarks after Alg. \ref{alg:Obs1} and (ii)
the performance of the complete  BARObs$(\gss{X}{0:n-1}{},d,\tau\leq a_{min}/4)$. 

\subsubsection{Supergraph Selection Only and the First Two Remarks After Alg. \ref{alg:Obs1}}

Fig. \ref{fig:2a} assumes that $d_i=d$ for all $i\in [p]$ and corresponds to comparing
the described scheme in the second remark after Alg. \ref{alg:Obs1} with an exhaustive learning algorithm, which selects as support for $\gss{a}{i}{T}$ the $d-$sized neighborhood among all the ${p\choose d}$ neighborhoods  with the maximum directed information flow to the $i$th node. The proposed scheme has comparable performance with the aforementioned exhaustive observer, but with a significantly smaller computational complexity and comparable (and small) number of samples. In this plot, we compare only with respect  to the selection of $\mathcal{S}$. The described scheme in the second remark after Alg. \ref{alg:Obs1} can further discriminate between $\hat{\mathcal{S}}^{+}(i)$ and $\hat{\mathcal{S}}^{-}(i)$. This is verified in Fig. \ref{fig:2b} for a different setup. In Fig. \ref{fig:2b}, a single $\g{A}$ satisfying the BAR Identifiability Condition has been selected for all $n$ and $d_i=d$ for all $i\in [p]$ again. Moreover, $p=30,d=9, a_{min}=0.1, b_{min}=0.1$ and $\rho_w=0.5$.  We note here that in practice there might be instances where the selection of $\mathcal{S}$ is errorless, but the discrimination of $\mathcal{S}^{+}(i)$ and $\mathcal{S}^{-}(i)$ is erroneous for some of the rows.

Remaining in the described scheme in the second remark after Alg. \ref{alg:Obs1}, Fig. \ref{fig:3a} presents the scenario where all nodes have \emph{different} but a priori known in-degrees $d_i$ and all  the in-degrees are upper bounded by $d$. To implicitly quantify what fraction of BAR chains with $p=30, d=3, a_{min}=0.1$, $b_{min}=0.1$ and $\rho_w=0.5$ ($b_{max}=0.2$)  satisfy the BAR Identifiability Condition or tends to create a larger amount of helpful than harmful correlations, we vary $\g{A}$ per run. We observe that approximately $80\% $ of the generated chains fall into this category leading to complete recovery with approximately $2000$ samples. Fig. \ref{fig:3b} shows that for a chain in this scenario, recovery can be achieved even with approximately $1000$ samples.  Additionally, since $\mathrm{P}(\hat{\mathcal{S}}=\mathcal{S})$ or $\mathrm{P}(\hat{\mathcal{S}}\neq\mathcal{S})$ are hard metrics in the sense that they correspond to the correct selection of $pd$ parameters in $\mathcal{S}$, it is of interest to examine how many edges and orientations are correctly selected 
with an increasing number of samples, even when $\hat{\mathcal{S}}\neq \mathcal{S}$. To this end, Fig. \ref{fig:6a} corresponds to the scenario $p=1000, d=10, b_{min}=0.05, \rho_{w}=0.5$ and all nodes have the same in-degree $d_i=d$. Clearly,  a very large part of the true network is correctly selected fairly quickly. Note that the values of $p$ and $d$ in this scenario could be prohibitive for many previously studied model order selection methods. In this plot, ``Fraction of Correct Edges'' refers to what fraction of the true edges \emph{only} the scheme identifies. Knowledge of the in-degrees and 
the fact that $d\ll p$ imply that the non-edges are also identified with very high probability. This is demonstrated in Fig. \ref{fig:6b}, which we describe later on.     

Moving now to  the described scheme in the first remark after Alg. \ref{alg:Obs1}, Fig. \ref{fig:4a} presents the scenario where all nodes have different and unknown in-degrees $d_i$ and all the in-degrees are upper bounded by $d$. In this case, the described scheme in the first remark after Alg. \ref{alg:Obs1} will return a supergraph of the actual graph. Therefore, the $y-$axis corresponds to 
probability of correct recovery of a supergraph of the actual graph.   To implicitly quantify what fraction of chains with $p=30, a_{min}=0.1, b_{min}=0.1, \rho_w=0.5$ and in-degree overestimate $d=3$ satisfies the BAR identifiability condition or tends to create a larger amount of helpful than harmful correlations, we vary $\g{A}$ per run. We observe that this plot does not differ significantly from Fig. \ref{fig:3a} as expected, since the upper bound $d$ allows in many cases rows with exactly $d$ nonzero entries. Furthermore, Fig. \ref{fig:4b} shows that for a matrix in the same scenario, recovery can be achieved with approximately $1000$ samples. In practice, the minimum number of samples turns out to actually be $600$. Comparing these results with Fig. \ref{fig:3b}, we can see that the outcome is in agreement with the fact that the described scheme in the first remark after Alg. \ref{alg:Obs1} generates a supergraph of the actual graph, hence requiring less samples for complete recovery.

\subsubsection{Complete BARObs$(\gss{X}{0:n-1}{},d,\tau\leq a_{min}/4)$}
 
Fig. \ref{fig:5a} presents the performance of BARObs in a scenario where all nodes have different and unknown in-degrees $d_i$ and all of the in-degrees are upper bounded by $d$. In this case, the supergraph selection stage will return a supergraph of the actual graph and then the supergraph trimming stage will refine this estimate by estimating the in-degrees and the actual neighborhoods of each node. To implicitly quantify what fraction of chains with $p=30, d=3, a_{min}=0.1, b_{min}=0.1$ and $\rho_w=0.5$ satisfy the BAR Identifiability Condition or tends to create a larger amount of helpful than harmful correlations and what is the correspondence with respect to the BARObs in samples for perfect detection, we vary $\g{A}$ per run.  Furthermore, Fig. \ref{fig:5b} shows that for a matrix in the same scenario, recovery can be achieved with approximately $14000$ samples. This is the price paid for the unknown $d_i$'s that have to be estimated.

\subsection{ Data from a Biological Signaling Network}

We use a publicly available model for an abscisic acid signaling network consisting of $43$ nodes provided in \cite{js08}. Each node is linked to a boolean rule  defined by a subset of nodes in the network. We approximate the network in the spirit of BAR models. The approximation is not exactly the described BAR model but an appropriate version of the model for the boolean network at hand.  We generate stochastic pseudodynamical data by adding binary/boolean noise to the produced time series. In Fig. \ref{fig:6b}, we plot the output of the algorithm in the second remark after Alg. \ref{alg:Obs1}. We observe that the supergraph selection only stage can identify almost $98\%$ of the true underlying network using only $800$ samples. Of course, the curve here is biased since we assume that we know exactly the in-degrees in this plot, hence the non-edges are correctly identified with high probability since the in-degrees are small. Moreover, the performance of this algorithm when we are only interested in the fraction of the actual edges correctly identified is depicted in Fig. \ref{fig:7a}. We also examine the performance of the algorithm described in the first remark after Alg. \ref{alg:Obs1}. Here, the goal is to evaluate what fraction of the actual edges we pick by a known overestimate $d$. In Fig.\ref{fig:7b}, $d=5$ and in Fig. \ref{fig:8a}, $d=10$. We clearly see the improvement by using a larger $d$. Finally, based on the same publicly available model for an abscisic acid signaling network, we create a noisy version of an AND/OR only network with $p=43$ and $d_i=2$ for all $i\in [p]$. We apply 
BAROBs and as Fig. \ref{fig:8b} shows, we can identify the true network with probability $0.8$ when we have approximately $11000$ time-series samples and with probability almost $1$ when we have $19000$ samples.     

\begin{figure}[t!]
    \centering
    \begin{subfigure}[t]{0.5\textwidth}
        \centering
        \includegraphics[height=2.2in]{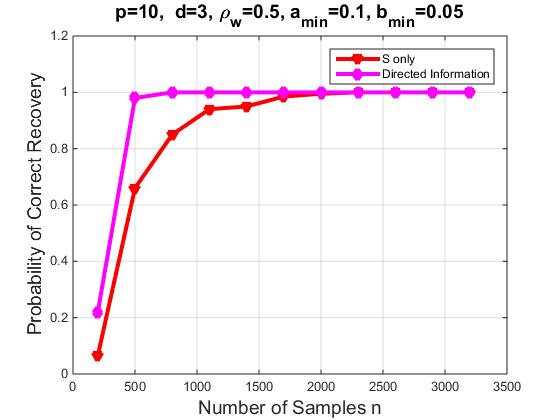}
            \captionsetup{justification=centering}
        \caption{$p=10, d=3, a_{min}=0.1, b_{min}=0.05, \rho_{w}=0.5$.}
        \label{fig:2a}
    \end{subfigure}%
    ~ 
    \begin{subfigure}[t]{0.5\textwidth}
        \centering
        \includegraphics[height=2.2in]{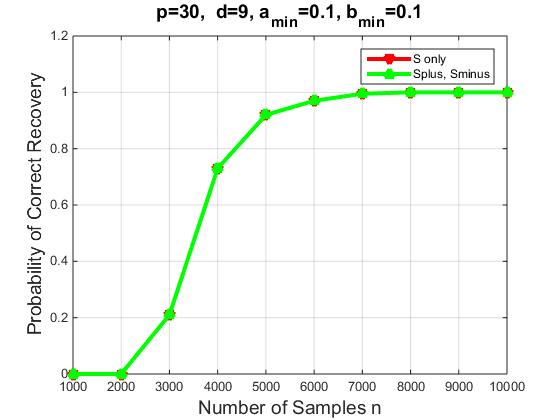}
            \captionsetup{justification=centering}
        \caption{$p=30, d=9, a_{min}=0.1, b_{min}=0.1, \rho_{w}=0.5$.}
        \label{fig:2b}
    \end{subfigure}%
     \caption{(a) Comparison of the scheme in the second remark after Alg. \ref{alg:Obs1} with an exhaustive observer based on directed information. (b)  The scheme in the second remark after Alg. \ref{alg:Obs1} for simulated data with a \emph{single} $\g{A}$.
     \emph{Same} and \emph{a priori known} $d_i=d$ for all nodes in both figures.}
\end{figure}

\begin{figure}[t!]
    \centering
    \begin{subfigure}[t]{0.5\textwidth}
        \centering
        \includegraphics[height=2.2in]{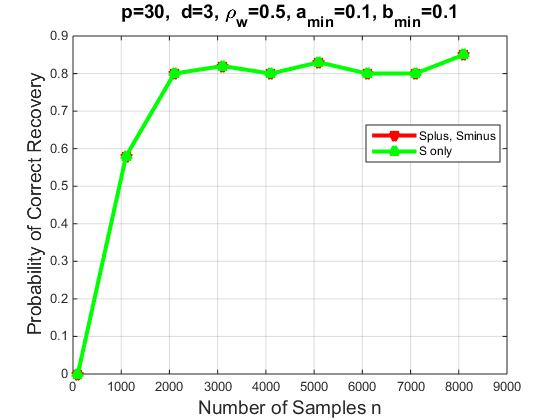}
            \captionsetup{justification=centering}
        \caption{$p=30, d=3, a_{min}=0.1, b_{min}=0.1, \rho_{w}=0.5$.}
        \label{fig:3a}
    \end{subfigure}%
    ~ 
    \begin{subfigure}[t]{0.5\textwidth}
        \centering
        \includegraphics[height=2.2in]{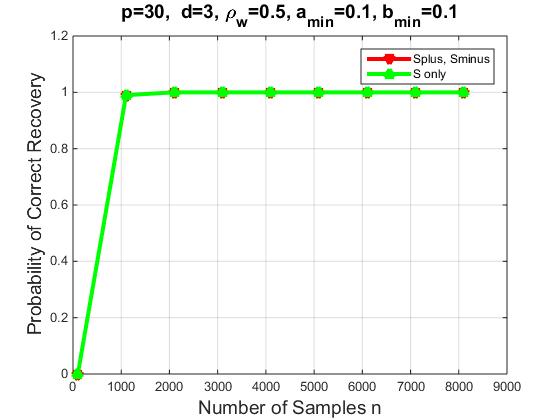}
            \captionsetup{justification=centering}
        \caption{$p=30, d=3, a_{min}=0.1, b_{min}=0.1, \rho_{w}=0.5$.}
        \label{fig:3b}
    \end{subfigure}%
     \caption{(a) The scheme in the second remark after Alg. \ref{alg:Obs1} for simulated data with \emph{multiple} $\g{A}$'s. (b) The scheme in the second remark after Alg. \ref{alg:Obs1} for simulated data with a \emph{single} $\g{A}$.
     \emph{Different} and \emph{a priori known} $d_i$'s for all nodes in both figures.}
\end{figure}

\begin{figure}[t!]
    \centering
    \begin{subfigure}[t]{0.5\textwidth}
        \centering
        \includegraphics[height=2.2in]{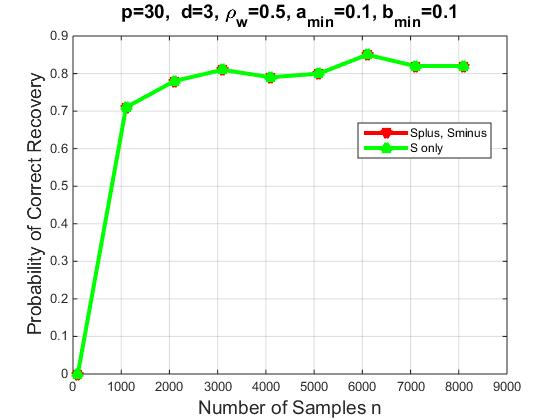}
            \captionsetup{justification=centering}
        \caption{$p=30, d=3, a_{min}=0.1, b_{min}=0.1, \rho_{w}=0.5$.}
        \label{fig:4a}
    \end{subfigure}%
    ~ 
    \begin{subfigure}[t]{0.5\textwidth}
        \centering
        \includegraphics[height=2.2in]{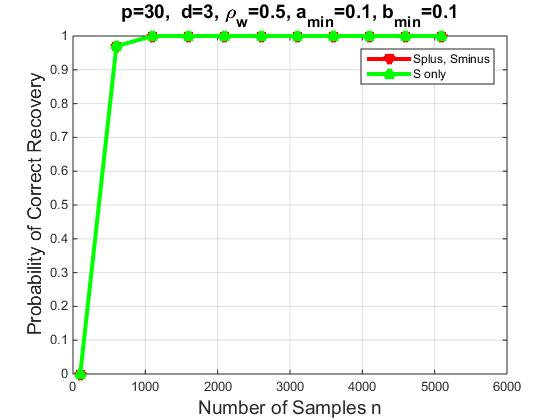}
            \captionsetup{justification=centering}
        \caption{$p=30, d=3, a_{min}=0.1, b_{min}=0.1, \rho_{w}=0.5$.}
        \label{fig:4b}
    \end{subfigure}%
     \caption{(a) The scheme in the first remark after Alg. \ref{alg:Obs1} for simulated data with \emph{multiple} $\g{A}$'s. (b) The scheme in the first remark after Alg. \ref{alg:Obs1} for simulated data with a \emph{single} $\g{A}$.
     \emph{Different} $d_i$'s for all nodes in both figures. Known overestimate $d$.}
\end{figure}

\begin{figure}[t!]
    \centering
    \begin{subfigure}[t]{0.5\textwidth}
        \centering
        \includegraphics[height=2.2in]{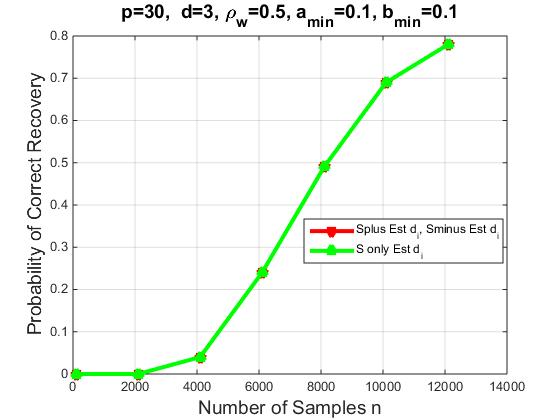}
            \captionsetup{justification=centering}
        \caption{$p=30, d=3, a_{min}=0.1, b_{min}=0.1, \rho_{w}=0.5$.}
        \label{fig:5a}
    \end{subfigure}%
    ~ 
    \begin{subfigure}[t]{0.5\textwidth}
        \centering
        \includegraphics[height=2.2in]{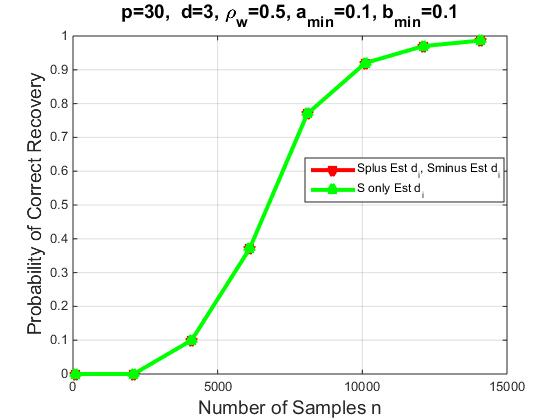}
            \captionsetup{justification=centering}
        \caption{$p=30, d=3, a_{min}=0.1, b_{min}=0.1, \rho_{w}=0.5$.}
        \label{fig:5b}
    \end{subfigure}%
     \caption{(a) BARObs for simulated data with \emph{multiple} $\g{A}$'s. (b) BARObs for simulated data with a \emph{single} $\g{A}$.
     \emph{Different} and \emph{unknown} $d_i$'s for all nodes in both figures.}
\end{figure}

\begin{figure}[t!]
    \centering
    \begin{subfigure}[t]{0.5\textwidth}
        \centering
        \includegraphics[height=2.2in]{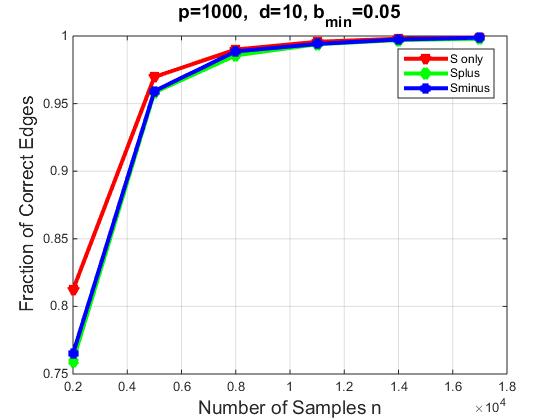}
            \captionsetup{justification=centering}
        \caption{$p=1000, d=10, b_{min}=0.05, \rho_{w}=0.5$.}
        \label{fig:6a}
    \end{subfigure}%
    ~ 
    \begin{subfigure}[t]{0.5\textwidth}
        \centering
        \includegraphics[height=2.2in]{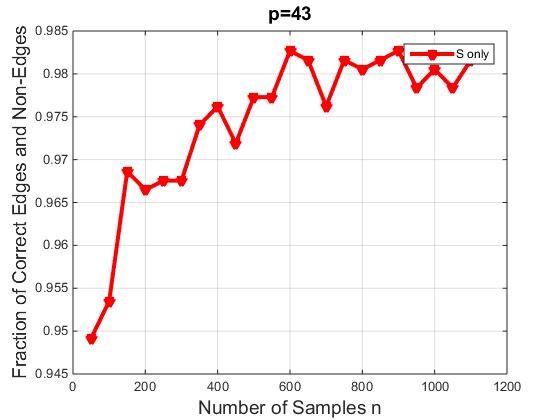}
            \captionsetup{justification=centering}
        \caption{$p=43$.}
        \label{fig:6b}
    \end{subfigure}%
    \caption{(a) The scheme in the second remark after Alg. \ref{alg:Obs1}:  large network.  \emph{Same} and \emph{a priori known} $d_i=d$ for all nodes. (b) The scheme in the second remark after Alg. \ref{alg:Obs1}: Biological signaling network. \emph{A priori known} $d_i$'s.}
   
\end{figure}

\begin{figure}[t!]
    \centering
    \begin{subfigure}[t]{0.5\textwidth}
        \centering
        \includegraphics[height=2.2in]{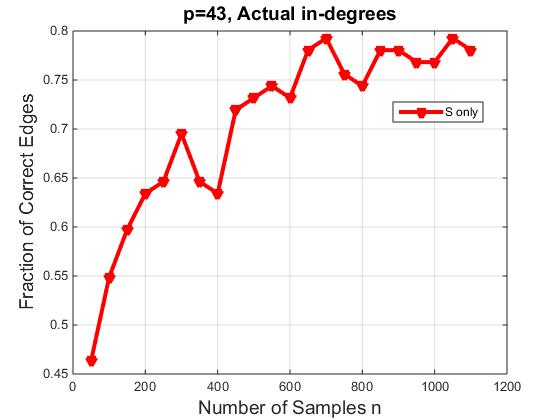}
            \captionsetup{justification=centering}
        \caption{$p=43, d_i\leq 5$.}
        \label{fig:7a}
    \end{subfigure}%
    ~ 
    \begin{subfigure}[t]{0.5\textwidth}
        \centering
        \includegraphics[height=2.2in]{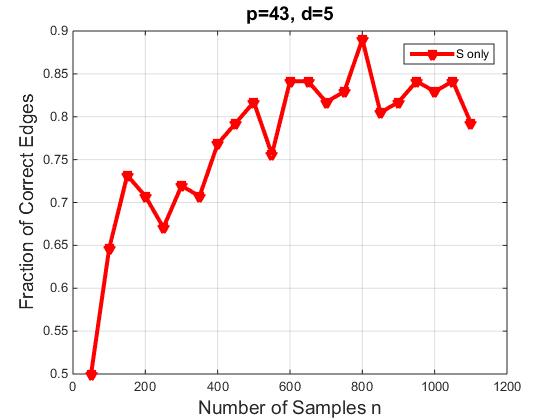}
            \captionsetup{justification=centering}
        \caption{$p=43, d_i\leq 5, d=5$.}
        \label{fig:7b}
    \end{subfigure}%
    \caption{(a)  The scheme in the second remark after Alg. \ref{alg:Obs1}: Biological signaling network. \emph{A priori known} $d_i$'s. (b)  The scheme in the first remark after Alg. \ref{alg:Obs1}: Biological signaling network. Overestimate $d_i\leq d=5$.}
   
\end{figure}

\begin{figure}[t!]
    \centering
    \begin{subfigure}[t]{0.5\textwidth}
        \centering
        \includegraphics[height=2.2in]{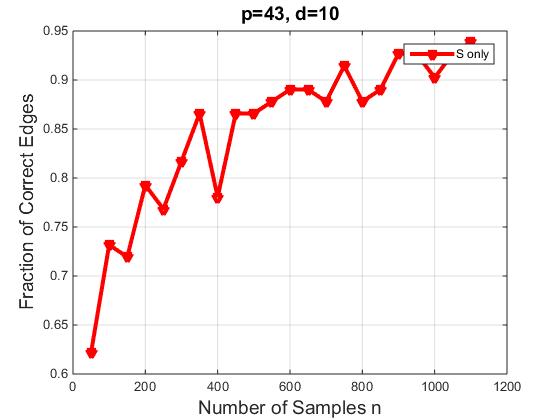}
            \captionsetup{justification=centering}
        \caption{$p=43, d_i\leq 5, d=10$.}
        \label{fig:8a}
    \end{subfigure}%
    ~ 
    \begin{subfigure}[t]{0.5\textwidth}
        \centering
        \includegraphics[height=2.2in]{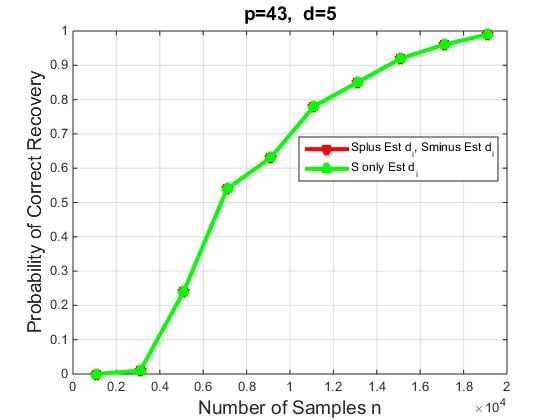}
            \captionsetup{justification=centering}
        \caption{$p=43, d_i=2, d=5$.}
        \label{fig:8b}
    \end{subfigure}%
    \caption{(a)  The scheme in the first remark after Alg. \ref{alg:Obs1}: Biological signaling network. Overestimate $d_i\leq d=10$. (b) AND/OR biological signaling network with $d_i=2$ for all $i\in [p]$. Application of BARObs with $d=5$.}
   
\end{figure}

\section{Conclusions}
\label{sec:concl}

In this paper, a novel model, called BAR, was introduced as an alternative to the description of binary vector random processes with autoregressive dynamics. 
BAR processes can be used to model opinion dynamics over social networks, voter processes, epidemics, interactions among stocks in financial markets and among genes or chemical elements in biological networks. We proved that the general BAR model mixes rapidly. We also showed that two random walk versions of this model mix 
rapidly under some conditions. Furthermore, we provided a low-complexity algorithm that can be used to identify the structure of the BAR network based on time-series data. This structure estimator was shown to be nearly order-optimal in sample complexity, adding new attractive features to the proposed BAR model.

\newpage

\appendix

\section{}
\label{app:2}
\subsection*{Proof of Lemma \ref{lem:nuij}}

$H_{ij,1}(\g{p})$ and $H_{ij,0}(\g{p}')$ depend on $\log P(X_i^{+1}=x_i|X_j=1)$ and $\log P(X_i^{+1}=x_i|X_j=0)$, respectively. To ease the notation, 
we will write $p_{x_i}$ for $P(X_i^{+1}=x_i|X_j=1)$ and $p'_{x_i}$ for $P(X_i^{+1}=x_i|X_j=0)$ in the remaining of this proof. We will further assume that 
any entropy is expressed using the natural logarithm, hence it is measured in nats.

For $|x|<1$, we have the following well-known Taylor series expansion:
\[
\log(1-x)=-\sum_{n=1}^{\infty}\frac{x^n}{n}.
\]
Therefore, 
\[
\log p_{x_i}=\log(1-(1-p_{x_i}))=-\sum_{n=1}^{\infty}\frac{(1-p_{x_i})^n}{n}=-\sum_{n=1}^{\infty}\frac{1}{n}\sum_{k=0}^{n}{n\choose k}(-1)^{n-k}p_{x_i}^{n-k}
\]
and a similar expression holds for $\log p'_{x_i}$. We can now see that 
\begin{align*}
H_{ij,1}(\g{p})-H_{ij,0}(\g{p}')&=\sum_{x_i\in \{0,1\}}\sum_{n=1}^{\infty}\frac{1}{n}\sum_{k=0}^{n}{n\choose k}(-1)^{n-k}(p_{x_i}^{n-k+1}-p_{x_i}^{'n-k+1})\\&=\sum_{x_i\in \{0,1\}}(p_{x_i}-p'_{x_i})\sum_{n=1}^{\infty}\frac{1}{n}\sum_{k=0}^{n}{n\choose k}(-1)^{n-k}\sum_{l=0}^{n-k}p_{x_i}^{n-k-l}p_{x_i}^{'l}\\&=\nu_{i|j}\sum_{x_i\in \{0,1\}}(-1)^{x_i-1}\sum_{n=1}^{\infty}\frac{1}{n}\sum_{k=0}^{n}{n\choose k}(-1)^{n-k}\sum_{l=0}^{n-k}p_{x_i}^{n-k-l}p_{x_i}^{'l},
\end{align*}
where in the last equality we have used the observation that $\nu_{i|j}=p_{x_i=1}-p'_{x_i=1}$ and $p_{x_i=0}-p'_{x_i=0}=1-p_{x_i=1}-1+p'_{x_i=1}=-\nu_{i|j}$. Therefore, $H_{ij,1}(\g{p})-H_{ij,0}(\g{p}')$ depends critically on $\nu_{i|j}$.

\section{}
\label{app:1}
\subsection*{Proof of Lemma \ref{lem:Marg1}}

Given any vector norm $\|\cdot\|$, the induced operator norm for an $m\times n$ matrix $\g{K}$ is defined as \cite{hj85}:
\[
\|\g{K}\|=\sup\left\{\|\g{Kx}\|: \g{x}\in \mathcal{K}^{n}\: \text{with}\: \|\g{x}\|=1\right\},
\]
where $\mathcal{K}$ can be either $\mathbb{R}$ or $\mathbb{C}$.
On the other hand, if $m=n=p$, the spectral radius of $\g{K}$  is defined as:
\[
\rho(\g{K})=\max\{|\lambda_1|,|\lambda_2|,\ldots,|\lambda_p|\},
\]
where $\lambda_1,\lambda_2,\ldots, \lambda_p$ are the eigenvalues of $\g{K}$ and $|\cdot|$ is the complex modulus.
For any $r\in \mathbb{Z}_{>0}$ it holds $\|\g{K}^r\|^{1/r}\geq \rho(\g{K})$, while Gelfand's formula states that for any matrix norm \cite{hj85}:
\[
\rho(\g{K})=\lim_{r\rightarrow \infty}\|\g{K}^r\|^{1/r}.
\]
When $\g{K}$ is symmetric or Hermitian, $\rho(\g{K})=\|\g{K}\|_2$. In general, $\|\g{K}\|\geq \rho(\g{K})$ for any matrix norm $\|\cdot\|$. 

For easiness, we will work with the Euclidean vector norm and the associated induced matrix norm, which is the \emph{spectral norm}:
\begin{equation}\label{eq:spectralNorm}
\|\g{K}\|_2=\sigma_{max}(\g{K})=\sqrt{\lambda_{max}(\gss{K}{}{H}\g{K})}.
\end{equation}
Here, $\sigma_{max}(\cdot)$ and $\lambda_{max}(\cdot)$ correspond to the maximum singular value and the maximum eigenvalue of a matrix, respectively. Moreover, $\cdot^H$ stands for Hermitian transposition, which is $\cdot^T$ for real $\g{K}$'s. Working in the following with $\|\g{\hat{A}}^r\|_2$ and
$\left\|\g{\bar{A}}^r\right\|_2$ for any $r\in \mathbb{Z}_{>0}$, (\ref{eq:spectralNorm}) implies that the corresponding maximizing vectors for $\|\g{\hat{A}}^r\|_2$ and $\left\|\g{\bar{A}}^r\right\|_2$ correspond to the top eigenvectors of $\gss{\hat{A}}{}{rT}\g{\hat{A}}^r$ and $\gss{\bar{A}}{}{rT}\g{\bar{A}}^r$, which are real vectors by basic facts in linear algebra. 

In order to prove our claim, it is sufficient to show that $\|\g{\hat{A}}^r\|_2\leq \|\g{\bar{A}}^r\|_2$ for any $r\in \mathbb{Z}_{>0}$. Consider first the case $r=1$. Since 
$\g{\bar{A}}$ has non-negative entries, it is clear that $\|\g{\bar{A}}\|_2$ is maximized for $\g{x}$ with $\|\g{x}\|_2=1$ such that either $x_i\geq 0$ for all $i\in [p]$ or $x_i\leq 0$ for all $i\in [p]$. Clearly, for any such $\g{x}$, we have: $\|\g{\hat{A}}\g{x}\|_2\leq \|\g{\bar{A}}\g{x}\|_2$. Consider now a $\g{\hat{x}}$ with $\|\g{\hat{x}}\|_2=1$ such that 
$\|\g{\hat{A}}\g{\hat{x}}\|_2$ is maximized. If this $\g{\hat{x}}$ has only non-negative or only non-positive entries, then $\|\g{\hat{A}}\|_2\leq \|\g{\bar{A}}\|_2$. On the other hand, if  $\g{\hat{x}}$ contains some positive and some negative entries, then by inverting the polarities to either the positive or the negative entries, we end up with a $\g{x}$ such that $\|\g{\hat{x}}\|_2=\|\g{x}\|_2$ and $\|\g{\hat{A}\hat{x}}\|_2\leq \|\g{\bar{A}}\g{x}\|_2$; thus again $\|\g{\hat{A}}\|_2\leq \|\g{\bar{A}}\|_2$. Note here, that the existence of non-positive and non-negative entries in $\g{\hat{x}}$ leads 
to possible lower entries in magnitude in $\g{\hat{A}\hat{x}}$ for rows containing only non-negative or only non-positive entries. The corresponding entries in $\g{\bar{A}}\g{x}$ will be larger in magnitude, which verifies our claim.

We can also extend the above argument in the complex field (although this extension is not required for the proof, but only for reasons of mathematical completeness). Consider $\g{\hat{A}\hat{x}}$ for any $\g{\hat{x}}\in \mathbb{C}^p$ such that $\|\g{\hat{x}}\|_2=1$. Then, we can form $\g{x}$ based on $\g{\hat{x}}$ such that all real parts are co-signed and all imaginary parts are also co-signed (the signs in the real and imaginary parts can be different, e.g., all real parts can be positive and all imaginary parts can be negative). Then, $\|\g{\hat{x}}\|_2=\|\g{x}\|_2$ and $\|\g{\hat{A}\hat{x}}\|_2\leq \|\g{\bar{A}}\g{x}\|_2$, which demonstrates the extension of the above argument in the complex field.  

We now note that $\g{\bar{A}}^r$ contains non-negative  entries for any $r\in \mathbb{Z}_{>0}$. On the other hand any entry in $\g{\hat{A}}^r$ that does not coincide with the corresponding entry in $\g{\bar{A}}^r$, will be either negative  and of the same magnitude, or negative and of smaller magnitude or positive and of smaller magnitude (i.e., the existence of negative elements in $\g{\hat{A}}$ tend to have a contracting effect on the entries of $\gss{\hat{A}}{}{r}$ as $r$ increases). Thus, by the same reasoning, $\|\g{\hat{A}}^r\|_2\leq \|\g{\bar{A}}^r\|_2$ for any $r\in \mathbb{Z}_{>0}$. Using now Gelfand's formula and taking the limit in both sides, we obtain:
\[
\rho\left(\g{\hat{A}}\right)=\lim_{r\rightarrow \infty}\|\g{\hat{A}}^r\|_2^{1/r}\leq \lim_{r\rightarrow \infty}\|\g{\bar{A}}^r\|_2^{1/r}=\rho(\g{\bar{A}}).
\]

Finally, we note that $\g{\bar{A}}$ is a substochastic matrix and a straightforward application of Gershgorin's Theorem yields that  $\rho(\g{\bar{A}})<1$.

\subsection*{Proof of Lemma \ref{lem:Marg2}}
By our assumptions:
\begin{align*}
\g{p}=\g{\bar{A}}\g{p}+\rho_w\g{B}\g{1}.
\end{align*}
Moreover,
\begin{align*}
\g{\bar{A}}\g{1}+\g{B}\g{1}=\g{1}.
\end{align*}
Using the stability of $\g{\bar{A}}$, we have:
\begin{align*}
&\g{p}=\rho_w\left(\g{I}-\g{\bar{A}}\right)^{-1}\g{B}\g{1}=\rho_w\left(\g{I}-\g{\bar{A}}\right)^{-1}\g{1}-\\&
\rho_w\left(\g{I}-\g{\bar{A}}\right)^{-1}\g{\bar{A}}\g{1}=\rho_w\sum_{l=0}^{\infty}\g{\bar{A}}^l\g{1}-\rho_w\sum_{l=1}^{\infty}\g{\bar{A}}^l\g{1}=\rho_w\g{1},
\end{align*}
where we have used von Neumann's series formula \cite{hj91}.
\section{}
\label{app:3}
\subsection*{Proof of Corollary \ref{cor:Corr1}}

Let $j\in \mathcal{S}^{+}(m)$. By an elementary argument,  
\begin{align*}
&E[f_m(X_j)|X_{i}=1]-E[f_m(X_j)|X_{i}=0]=P(X_j=1|X_i=1)-\\
&P(X_j=1|X_i=0)=(1-P(X_j=0|X_i=1))-\\
&(1-P(X_j=0|X_i=0))=P(X_j=0|X_i=0)-\\&P(X_j=0|X_i=1).
\end{align*}

To see that 
\begin{align*}
&E[X_j|X_{i}=1]-E[X_j|X_{i}=0]\neq\\
&E[X_i|X_{j}=1]-E[X_i|X_{j}=0],
\end{align*}
we give the following counterexample. Consider the joint measure: 
\begin{align*}
&P(X_i=1,X_j=1)=0.4,\\
&P(X_i=1,X_j=0)=0.2,\\
&P(X_i=0,X_j=1)=0.1,\\
&P(X_i=0,X_j=0)=0.3.
\end{align*}
A straightforward calculation shows that 
\begin{align*}
&\frac{5}{12}=E[X_j|X_{i}=1]-E[X_j|X_{i}=0]\neq\\
&E[X_i|X_{j}=1]-E[X_i|X_{j}=0]=\frac{2}{5}.
\end{align*}

Finally, for $j\in \mathcal{S}^{-}(m)$: 
\begin{align*}
&E[f_m(X_j)|X_{i}=1]-E[f_m(X_j)|X_{i}=0]=P(X_j=0|X_i=1)-\\
&P(X_j=0|X_i=0)=(1-P(X_j=1|X_i=1))-\\
&(1-P(X_j=1|X_i=0))=P(X_j=1|X_i=0)-\\&P(X_j=1|X_i=1).
\end{align*}

\section{}
\label{app:8}
\subsection*{Proof of Lemma \ref{lem:FanoBAR}}

We would like to lower bound the number of required samples such that the worst case $\mathcal{R}_{*}(\psi)\leq \epsilon$. Clearly, the number of required samples should be at least the number of required samples  for $\mathcal{R}_{\sbm{a},\sbm{b}}(\psi)\leq \epsilon$ for any valid choice of $(\g{a},\g{b})$ and $\mathcal{S}$. To this end, fixing $(\g{a},\g{b})$ and assuming that each $\mathcal{S}(i)$ with cardinality $d_i$ is drawn independently and uniformly at random, Fano's inequality implies that 
\[
\mathcal{R}_{\sbm{a},\sbm{b}}\left(\psi\right)\geq \frac{H\left(\mathcal{S}|\g{X}_{0:n-1}\right)-1}{\log {p\choose d_1}{p\choose d_2}\cdots{p\choose d_p}}.
\]
Consider the mutual information:
\begin{align*}
I\left(\mathcal{S};\g{X}_{0:n-1}\right)=H(\mathcal{S})-H\left(\mathcal{S}|\g{X}_{0:n-1}\right).
\end{align*}
Replacing the LHS with 
\begin{align*}
&H\left(\g{X}_{0:n-1}\right)-H\left(\g{X}_{0:n-1}|\mathcal{S}\right)
\end{align*}
we obtain:
\begin{align*}
H\left(\mathcal{S}|\g{X}_{0:n-1}\right)&=\log {p\choose d_1}{p\choose d_2}\cdots{p\choose d_p}-H\left(\g{X}_{0:n-1}\right)+H\left(\g{X}_{0:n-1}|\mathcal{S}\right)\\ &\geq \log {p\choose d_1}{p\choose d_2}\cdots{p\choose d_p}-pn,
\end{align*}
where for the $pn$ term we have used the bound $H\left(\g{X}_{0:n-1}\right)\leq \sum_{i=0}^{n-1}H\left(\g{X}_{i}\right)\leq n\log2^{p}=pn$.
Thus,
\[
\mathcal{R}_{\sbm{a},\sbm{b}}\left(\psi\right)\geq \frac{\log {p\choose d_1}{p\choose d_2}\cdots{p\choose d_p}-np-1}{\log {p\choose d_1}{p\choose d_2}\cdots{p\choose d_p}}.
\]
Using now the fact that $\mathcal{R}_{*}(\psi)\geq \mathcal{R}_{\sbm{a},\sbm{b}}\left(\psi\right)$ and requiring that $\mathcal{R}_{*}(\psi)\leq \epsilon$, we obtain:
\[
n\geq \frac{(1-\epsilon)}{p}\sum_{i=1}^p\log {p\choose d_i},
\]
where the $-1/p$ term has been neglected.

\section{}
\label{app:5}

\subsection*{Proof of Lemma \ref{lem:BARrwMxing2}}

Assume that we populate randomly each row of $\g{\bar{A}}$ with $d_i, i\in [p]$ entries. Let's describe the procedure: 
\begin{enumerate}
\item For each row, we pick independently one of the ${p\choose d_i}$ binary vectors with $d_i$ ones. This is the support of the respective row.
\item At the locations of the $1$'s, we place the nonzero entries of this row of $\g{\bar{A}}$. 
\end{enumerate}
Clearly, the support of the $i$th row is selected with probability 
\[
p_r(i)=\frac{1}{{p\choose d_i}}.
\]

It is now easy to see, that each location in the $i$th $1\times p$ row vector has probability 
\[
p_e(i)=\frac{{p-1\choose d_i-1}}{{p\choose d_i}}=\frac{d_i}{p}
\]
of being selected. The numerator occurs if we fix a $1$ to the location of interest and we distribute the rest of the $d_i-1$ ones to the remaining 
$p-1$ positions at random. 

Consider now the matrix $\mathrm{sgn}(\g{\bar{A}})$, where $\mathrm{sgn}(x)=x/|x|$ for $x\neq 0$ and $0$ otherwise. Thus, $\mathrm{sgn}(\g{\bar{A}})$ is the support matrix of $\g{\bar{A}}$ with $1$'s at the nonzero locations and zeros elsewhere. Let $S_1,S_2,\ldots, S_p$ be the column sums of this matrix. Each sum contains independent $\mathrm{Ber}(d_i/p)$ random variables as summands.  By Hoeffding's inequality:
\begin{align*}
P\left(S_l\geq t+E[S_l]\right)\leq \exp\left(-2t^2/p\right),
\end{align*}
or
\begin{align*}
P\left(S_l\geq t+\sum_{i=1}^p\frac{d_i}{p}\right)\leq \exp\left(-2t^2/p\right),
\end{align*}
Choosing $t=\sqrt{cp\log p}$ for some positive constant $c$, we obtain:

\begin{align*}
P\left(S_l\geq \sqrt{cp\log p}+\sum_{i=1}^p\frac{d_i}{p}\right)\leq \frac{1}{p^{2c}}.
\end{align*}
Using now the union bound, we have:
\begin{align*}
P\left(\max_{1\leq l\leq p}S_l\geq \sqrt{cp\log p}+\sum_{i=1}^p\frac{d_i}{p}\right)\leq \sum_{l=1}^{p}P\left(S_l\geq \sqrt{cp\log p}+\sum_{i=1}^p\frac{d_i}{p}\right)\leq p\frac{1}{p^{2c}}=\frac{1}{p^{2c-1}},
\end{align*}
which goes to $0$ as $p\rightarrow \infty$ for any $c>1/2$. Thus, with probability at least $1-1/p^{2c-1}$, 
\[
\max_{1\leq l\leq p}S_l< \sqrt{cp\log p}+\sum_{i=1}^p\frac{d_i}{p}.
\]

\subsection*{Proof of Lemma \ref{lem:BARrwMxing3}}

Assume that we allow each non-zero entry of $\g{\bar{A}}$ to be in the interval $[a_{min},\bar{a}]$. Then, with probability at least $1-1/p^{2c-1}$
\[
\max_{1\leq j\leq p}\sum_{i=1}^{p}a_{ij}\leq \left(\sum_{i=1}^p\frac{d_i}{p}+\sqrt{cp\log p}\right)\bar{a}.
\]
Requiring the rightmost handside to be $<1$, we obtain that 
\[
\bar{a}<\frac{1}{\left(\sum_{i=1}^p\frac{d_i}{p}+\sqrt{cp\log p}\right)}.
\]
The rest of the claims follow from Theorem \ref{thm:BARrwMxing1}.

\vskip 0.2in
\bibliography{cristian}

\end{document}